\documentclass[11pt]{article}
\usepackage[a4paper, left=3cm, right=3cm, top=3cm, bottom=3cm]{geometry}
\usepackage[utf8]{inputenc}
\usepackage{graphics}
\usepackage{times}
\usepackage{amsmath, amsfonts, amssymb, amsthm}
\usepackage{stmaryrd}
\usepackage{appendix}
\usepackage[hidelinks]{hyperref}

\usepackage{rotating}

\usepackage{todonotes}
\usepackage{caption}
\usepackage{subcaption}
\usepackage{algorithmic}
\usepackage{algorithm}
\usepackage{diagbox}
\usepackage{tablefootnote}
\usepackage{array}
\usepackage{makecell}
\usepackage{circledsteps}
\usepackage{booktabs}
\usepackage{enumitem}
\usepackage{esvect}
\usepackage{graphicx, wrapfig}
\allowdisplaybreaks

\long\def\acks#1{\vskip 0.3in\noindent{\large\bf Acknowledgments}\vskip 0.2in
\noindent #1}

\newcommand{\vertiii}[1]{{\left\vert\kern-0.25ex\left\vert\kern-0.25ex\left\vert #1 
    \right\vert\kern-0.25ex\right\vert\kern-0.25ex\right\vert}}

\newtheorem{definition}{Definition}
\newtheorem{assumption}{Assumption}
\newtheorem{proposition}{Proposition}
\newtheorem{theorem}{Theorem}
\newtheorem{lemma}{Lemma}
\newtheorem{corollary}{Corollary}

\newtheoremstyle{myrem}%
{2pt}
{2pt}
{}
{}
{\bfseries}
{.}
{.5em}
{}%
\theoremstyle{myrem}

\newcommand{\N}{\mathbb{N}}
\newcommand{\R}{\mathbb{R}}
\newcommand{\E}{\mathbb{E}}

\newcommand{\cO}{\mathcal O}

\newcommand{\cL}{\mathcal{L}}
\newcommand{\cW}{\mathcal{W}}

\newcommand{\Proba}{\mathbb{P}}

\newcommand{\ind}[1]{\mathbf 1_{#1}}

\DeclareMathOperator*{\argmin}{argmin}

\DeclareMathOperator{\tv}{TV}
\DeclareMathOperator{\Var}{Var}
\DeclareMathOperator{\Lip}{Lip}

\newcommand{\cX}{\mathcal X}

\DeclareFontFamily{U}{mathx}{\hyphenchar\font45}
\DeclareFontShape{U}{mathx}{m}{n}{
      <5> <6> <7> <8> <9> <10>
      <10.95> <12> <14.4> <17.28> <20.74> <24.88>
      mathx10
      }{}
\DeclareSymbolFont{mathx}{U}{mathx}{m}{n}
\DeclareFontSubstitution{U}{mathx}{m}{n}
\DeclareMathAccent{\widecheck}{0}{mathx}{"71}

\title{Convergence and concentration properties of constant step-size SGD through Markov chains}

\author{Ibrahim Merad%
\thanks{LPSM, UMR 8001, Universit\'e Paris Cité, Paris, France}\\
\and
St\'ephane Ga\"iffas%
\thanks{LPSM, UMR 8001, Universit\'e Paris Cité, Paris, France and DMA, École normale supérieure}
}
\begin{document}

\maketitle

\begin{abstract}
    We consider the optimization of a smooth and strongly convex objective using constant step-size stochastic gradient descent (SGD) and study its properties through the prism of Markov chains. We show that, for unbiased gradient estimates with mildly controlled variance, the iteration converges to an invariant distribution in total variation distance. We also establish this convergence in Wasserstein-2 distance under a relaxed assumption on the gradient noise distribution compared to previous work. Our analysis shows that the SGD iterates and their invariant limit distribution \emph{inherit} sub-Gaussian or sub-exponential concentration properties when these hold true for the gradient. This allows the derivation of high-confidence bounds for the final estimate. Finally, under such conditions in the linear case, we obtain a dimension-free deviation bound for the Polyak-Ruppert average of a tail sequence. All our results are non-asymptotic and their consequences are discussed through a few applications.
    
\end{abstract}


\section{Introduction}

We consider the following stochastic optimization problem
\begin{equation}\label{eq:problem}
    \min_{\theta \in \R^d} \cL(\theta) := \E_{\zeta} [\ell(\theta, \zeta)],
\end{equation}
where $\cL$ is a smooth strongly convex objective only accessible through unbiased random gradient samples $G(\theta, \zeta) = \nabla \ell (\theta, \zeta)$ which may be queried at any parameter value $\theta \in \R^d$. Given an initial point $\theta_0$ and a step-size $\gamma,$ problem~\eqref{eq:problem} is commonly solved using the well-known stochastic gradient descent (SGD) algorithm defined by the iteration
\begin{equation}\label{eq:sgd_iteration1}
    \theta_{t+1} = \theta_t -\gamma G(\theta_t, \zeta_t), \quad \text{ for }\quad t\geq 0.
\end{equation}
We study the convergence properties of the Markov chain $(\theta_t)_{t\geq 0}$ generated by the above iteration as well as the concentration properties satisfied by a derived estimator $\widehat{\theta}$ of the global optimum $\theta^\star = \argmin_{\theta} \cL(\theta)$ based on the concentration of the gradient samples $G(\theta_t, \zeta_t).$

Problem~\eqref{eq:problem} is the common formulation for a large fraction of statistical learning problems where the objective $\cL(\theta)$ is defined as the expectation of a loss function $\ell$ over a random variable $\zeta$ following an unknown distribution of samples. In a practical setting, the random gradients $G(\theta_t, \zeta_t)$ are computed using a dataset of independent and identically distributed samples $(\zeta_i)_{i=1}^n.$ The SGD algorithm is employed to solve~\eqref{eq:problem} in two situations. Either the samples $(\zeta_i)_{i=1}^n$ are available offline but in such a great amount that using the whole dataset at each gradient step incurs an excessive computational load, therefore, individual samples or small batches are used at each iteration instead. Or, the samples $\zeta_i$ are received individually in an online fashion and optimization must be run using one instance at a time. Our framework covers both cases provided that each iteration uses new data which is independent from the past. Note that we exclude the optimization of empirical objectives of the form $\widehat{\cL}(\theta)=\frac{1}{n}\sum_{i=1}^n\ell_i(\theta)$ and focus on generalization errors w.r.t.~an unknown distribution of $\zeta$ in~\eqref{eq:problem}. 

Thanks to its simplicity and efficiency, the SGD algorithm is widely adopted as the go-to approach for stochastic optimization problems in general. Since its first appearance in the seminal work of~\cite{robbins1951stochastic} the theoretical properties of SGD have been investigated in a series of pioneering works~\cite{chung1954stochastic, sacks1958asymptotic, fabian1968asymptotic}. A notable milestone in these theoretical developments was the discovery of Polyak-Ruppert averaging~\cite{ruppert1988efficient, polyak1992acceleration} which allows to reduce the impact of noise and improve the convergence rate for certain cases of interest. The subject benefited from a growing attention with the advent of complex machine learning models such as neural networks and a rich literature has appeared to address the surfacing questions about SGD and its numerous variants and use cases~\cite{shamir2013stochastic, bach2014adaptivity, moulines2011non, bottou2007tradeoffs, needell2014stochastic}.

Although the basic definition of the SGD iteration~\eqref{eq:sgd_iteration1} is quite simple, a great number of variations are possible by playing on various aspects among which the choice of step-size is critical. Early work~\cite{robbins1951stochastic} suggested a decaying step-size of order $t^{-1}$ but this leads to poor dependence on problem conditioning~\cite{bach2014adaptivity, fabian1968asymptotic} while other step-size schedules with slower decay of order $t^{-\alpha}$ with $\alpha \in(1/2, 1]$ combined with averaging achieve better practical and theoretical performance~\cite{ruppert1988efficient,polyak1992acceleration, moulines2011non}. In this work, we consider constant step-size SGD which is also a commonly adopted choice due to its usually fast convergence~\cite{schmidt2013fast, pmlr-v89-vaswani19a, Ma2017ThePO}. 

\subsection{Main contributions}
This paper studies constant step-size SGD as a Markov chain and makes the following contributions.

\begin{itemize}
    \item We state two convergence results of the Markov chain to an invariant distribution. The first ergodicity theorem states convergence in total variation distance and the second one in terms of the Wasserstein-$2$ distance. While similar results exist in the literature~{\cite{yu2021analysis, dieuleveut2020bridging}, our version for the Wasserstein convergence mode improves upon previous work~\cite{dieuleveut2020bridging} since it holds in a more general setting.}
    \item {In a novel result}, we show that sub-Gaussian and sub-exponential concentration of the gradient samples implies the same property for the SGD iterates and their invariant limit distribution. Moreover, the associated constant is proportional to the step-size. {We believe this to be the first such characterization of the invariant distribution of constant step-size SGD.} Thanks to this property, we obtain high-confidence deviation bounds on the final SGD iterate.
    \item Provided a slightly stronger concentration assumption on the gradient samples, we show similar but dimension-free high-confidence bounds on the last SGD iterate. {Our statement is also non-asymptotic and does not require the gradient norm to be almost surely bounded so that it improves on previous results~\cite{jain2019making, harvey2019tight, yu2021analysis} which lacked some of these properties.}
    \item Finally, for the special case of a linear gradient, we obtain a high-confidence dimension-free bound for the Polyak-Ruppert average of a tail sequence of the SGD iterates. This is achieved, in part, thanks to a more generic concentration result which holds for any Lipschitz function applied to a stationary sequence. {Our concentration result improves upon existing literature~\cite{mou2020linear, lou2022beyond} with similar settings thanks to its dimension-free upper bound.}
\end{itemize}
All our results are non-asymptotic.

\subsection{Related works}
\paragraph{SGD as a Markov chain.} A fairly limited portion of the SGD literature adopts the Markov chain approach. Among the earliest,~\cite{pflug1986stochastic} studied the iteration in question for vanishing step-size, while~\cite{bach2013non} considers constant step-size averaged SGD for non-strongly-convex smooth objectives and shows $L_p$ convergence of the excess risk for all $p\geq 1.$ Although their analysis does not use Markov chain theory, they discuss properties of the invariant distribution which the iteration converges to, including a few properties we state in this paper. However, they do not derive high-confidence estimation bounds as we do. More recently, convergence in Wasserstein distance was established by~\cite{dieuleveut2020bridging} for constant step-size SGD applied to a strongly convex and smooth objective, albeit under a co-coercivity condition which is hard to establish in the nonlinear case. Further, an expansion of the asymptotic moments of averaged SGD is provided in~\cite{dieuleveut2020bridging} and the Richardson-Romberg extrapolation strategy is studied which allows to reduce the estimation error on the global optimum. Most recently,~\cite{yu2021analysis} studied SGD run on a non-convex, non-smooth but quadratically growing objective. Under such weakened conditions, they show that the generated Markov chain is geometrically ergodic (see~\cite{meyn2012markov}) and proceed to establish a CLT for the generated Markov iterates. They also state results controlling the bias of the limit distribution under additional assumptions such as convexity, $L_4$ control of the gradient noise and a generalized Łojasiewicz condition~\cite{karimi2016linear}.

\paragraph{High probability bounds.} In addition to establishing the convergence of SGD in expectation, the works of~\cite{rakhlin2011making, bach2013non, bach2014adaptivity} go further to state high-confidence bounds on the final error. Still, sub-Gaussian concentration only holds under strong bounded gradient assumptions. High-confidence deviation results are also stated in~\cite{ghadimi2013optimal} where an accelerated stochastic optimization method for strongly convex composite objectives is studied. However, the bounds are sub-exponential while the gradient is assumed to be sub-Gaussian. 

In~\cite{kakade2008generalization}, high probability bounds are proved for the PEGASOS algorithm using Freedman's inequality for martingales~\cite{freedman1975tail}. A generalization of the said inequality was used by~\cite{harvey2019simple, harvey2019tight} to prove such bounds for SGD in the non-smooth strongly convex case. Most recently, for a careful choice of step-size,~\cite{jain2019making} obtained high-confidence results on the last SGD iterate. Unfortunately, both previous works require a deterministic bound to hold over the gradient or its noise which strongly constrains their applicability. In~\cite{pillaud2018exponential}, the authors derived high probability convergence bounds for averaged and non-averaged SGD applied to classification and regression problems. Finally, a high probability analysis of Delayed AdaGrad with momentum was presented by~\cite{li2020high} in the smooth non-convex setting.

Note that certain recent works design \emph{robust} variants of SGD achieving sub-Gaussian deviation bounds on the last iterate with only a second moment assumption on the gradient~\cite{tsai2022heavy, gorbunov2020stochastic}. Similar results were later obtained under even weaker gradient moment assumptions~\cite{sadiev2023high, nguyen2023high}. However, in this work, we focus on the \emph{classical} SGD algorithm and the properties inherited by its iterates from the gradient samples.

\paragraph{Polyak-Ruppert averaging.} The averaging procedure introduced by~\cite{polyak1992acceleration, ruppert1988efficient} was also studied by~\cite{gyorfi1996averaged, defossez2015averaged} who proved asymptotic convergence properties. Non-asymptotic results and additional developments appeared in the works of~\cite{moulines2011non, dieuleveut2016nonparametric, dieuleveut2017harder, jain2016parallelizing, jain2017markov, lakshminarayanan2018linear} with particular attention to least-squares, logistic regression and kernel-based methods. In particular, non-asymptotic results of convergence in expectation were obtained for averaged SGD in~\cite{dieuleveut2017harder, lakshminarayanan2018linear, neu2018iterate, rakhlin2011making}. Among such results, some demonstrate the advantages of special averaging schemes~\cite{shamir2013stochastic, lacoste2012simpler}. The authors of~\cite{gadat2017optimal} prove a tight non-asymptotic $L^2$ convergence result for averaged iterates with decreasing step-size. Finally, some relatively recent works obtained high probability concentration bounds for Polyak Ruppert averaging with and without sub-Gaussian assumptions on the data~\cite{mou2020linear, lou2022beyond}.

\subsection{Paper organization}
Section~\ref{sec:setting} lays out the basic setting and assumptions necessary for SGD convergence. Section~\ref{sec:ergodicity} states our first SGD ergodicity result. In Section~\ref{sec:invariant_properties}, we first state a basic result on the invariant measure's expectation, bias and variance and proceed to derive concentration properties based on analogous assumptions on the gradient. Section~\ref{sec:wasserstein_convergence} presents an additional convergence result in Wasserstein distance. In Section~\ref{sec:deviation_bounds}, we give deviation bounds on the final SGD iterate which follow from preceding results. We also formulate our high-confidence bound on a tail Polyak-Ruppert average for the linear case. Finally, we discuss a few applications in Section~\ref{sec:applis} and conclude.

\section{Setting and notations}\label{sec:setting}

Let $\Theta$ denote either a convex subset of $\R^d$ or $\R^d$ itself depending on context. 
We refer to the Borel $\sigma$-algebra of $\R^d$ as $\mathcal{B}(\R^d).$ 
For any random variable $X,$ we denote $\mathcal{D}(X)$ its distribution. We refer to the space of square-integrable measures on $\R^d$ as $\mathcal{P}_2(\R^d).$ We denote $\mathcal{M}_1(\R^d)$ the set of probability measures over $\R^d.$ For real numbers $a$ and $b,$ we denote $\min(a, b) = a \wedge b$ and $\max(a, b) = a \vee b.$ We denote $\Lip(\cX)$ the set of $1$-Lipschitz functions $h:\cX \to \R.$ For $p \in \N^*,$ we denote $\|X\|_{L_p} = (\E|X|^p)^{1/p}$ the $L_p$ norm of a random variable~$X.$

In the entirety of this work, we assume that $\cL$ satisfies
\begin{assumption}\label{asm:smooth_strongconvex}
    There exist positive constants $0 < \mu \leq L < +\infty$ such that
    \begin{equation*}
        \frac{\mu}{2}\|\theta - \theta'\|^2 \leq \cL(\theta) - \cL(\theta') - \langle \nabla\cL(\theta'), \theta - \theta' \rangle \leq \frac{L}{2}\|\theta - \theta'\|^2
    \end{equation*}
    for all $\theta, \theta' \in \R^d,$ i.e. $\cL$ is $L$ gradient-Lipschitz and $\mu$-strongly convex.
\end{assumption}
As an immediate consequence, $\cL$ admits a unique minimum $\theta^\star$ which is a critical point:
\begin{equation}
    \theta^\star = \argmin_{\theta \in \R^d} \cL(\theta) \quad \text{such that}\quad \nabla \cL(\theta^\star) = 0.
\end{equation}
For an initial $\theta_0 \in \Theta,$ step-size $\gamma > 0$ and all $t\geq 0,$ we recall the basic SGD iteration
\begin{equation}\label{eq:sgd_iteration}
    \theta_{t+1} = \theta_t -\gamma G(\theta_t, \zeta_t).
\end{equation}
In this work, we consider constant step-size SGD so that $\gamma$ is fixed along the iteration. We require some basic assumptions on the samples $G(\theta_t, \zeta_t)$ in order to prove the convergence of SGD. Namely, $G(\theta_t, \zeta_t)$ needs to be an unbiased estimator of the true gradient $\nabla \cL(\theta_t)$ with controlled variance as we formally state in
\begin{assumption}\label{asm:gradient}
    Iteration~\eqref{eq:sgd_iteration} is run using a sequence of i.i.d samples $(\zeta_i)_{i\geq0}$. Further, given a fixed parameter $\theta \in \Theta$, the random gradient sample $G(\theta, \zeta)$ can be written as
    \begin{equation}
        G(\theta, \zeta) = \nabla\cL(\theta) + \varepsilon_{\zeta}(\theta),
    \end{equation}
    where the noise $\varepsilon_{\zeta}(\theta)$ satisfies the following properties$:$
    \begin{enumerate}[label=(\roman*)]
        \item \label{asm:gradient_centered}(Centered) We have $\E\big[ \varepsilon_{\zeta}(\theta) \big] =0.$
        \item \label{asm:gradient_dens_minor}(Density component \& minorization) The distribution of $\varepsilon_{\zeta}(\theta)$ can be written as $\mathcal{D}(\varepsilon_{\zeta}(\theta)) = \delta \nu_{\theta,1} + (1-\delta)\nu_{\theta, 2}$ with $\delta > 0$ and $\nu_{\theta, 1}, \nu_{\theta, 2}$ two probability distributions over $\R^d$ such that $\nu_{\theta, 1}$ admits a density $h(\theta, \cdot)$ w.r.t. Lebesgue's measure satisfying\textup:
        \begin{equation*}
            \inf_{\omega \in S}h(\theta, \omega) > 0 \quad  \text{for all $\theta$ and compact $S\subset \R^d$}.
        \end{equation*}
        \item \label{asm:gradient_regular}(Regularity) There are positive constants $L_\sigma$ and $\sigma^2$ such that for all $\theta$ we have$:$
        \begin{equation}
              \E\big[ \|\varepsilon_{\zeta}(\theta) \|^2  \big] = \E\big[ \|G(\theta, \zeta) - \nabla \cL (\theta)\|^2 \big] \leq L^2_{\sigma}\|\theta - \theta^\star\|^2 + \sigma^2.
        \end{equation}    
    \end{enumerate}
\end{assumption}
The additional assumptions on the distribution of the noise $\varepsilon_{\zeta}(\theta)$ are needed in order to establish the ergodicity of the resulting Markov chain $(\theta_t)_{t\geq 0}$ (Theorem~\ref{thm:ergodicity} below){. For instance, Assumption~\ref{asm:gradient}~\ref{asm:gradient_dens_minor} ensures that the noise density does not vanish unless taken near infinity. This entails} that the associated transition kernel satisfies a \emph{minorization} property implying that the chain will sufficiently explore the state space, see~\cite{meyn2012markov} for more details. Note also that these requirements are fairly mild since they only require the noise distribution to admit a diffuse component.

\section{Markov Chain and Geometric Ergodicity}\label{sec:ergodicity}

Before stating the convergence result for the SGD Markov chain, we introduce some further useful notation. For a given step-size $\gamma>0,$ we will denote $P_\gamma$ the Markov transition kernel governing the Markov chain $(\theta_t)_{t\geq 0}$ generated by iteration~\eqref{eq:sgd_iteration} so that for any $t\geq 0$ and $A\in \mathcal{B}(\R^d)$ we have:
\begin{equation*}
    \Proba(\theta_{t+1} \in A \:\vert\: \theta_t) = P_\gamma(\theta_t, A).
\end{equation*}
The transition kernel $P_\gamma$ acts on probability distributions $\nu \in \mathcal{M}_1(\R^d)$ through the mapping $\nu \to \nu P_\gamma$ which is defined, for all $A\in\mathcal{B}(\R^d)$, by $\nu P_\gamma(A) = \int P_\gamma(\theta, A)d\nu(\theta)$ which corresponds to $\Proba(\theta_{t+1}\in A \vert \theta_t)$ for $\theta_t \sim \nu.$ For $n \geq 1,$ we similarly define the multi-step transition kernel $P_\gamma^n$ which is such that $P_\gamma^n(\theta_t, A) = \Proba(\theta_{t+n} \in A \:\vert\: \theta_t)$ and acts on probability distributions $\nu \in \mathcal{M}_1(\R^d)$ through $\nu P_\gamma^n = (\nu P_\gamma)P_\gamma^{n-1}.$ Finally, we define the total-variation norm of a signed measure $\nu$ as 
\begin{equation*}
    2\|\nu\|_{\tv} = \sup_{f:|f|\leq 1} \int f(\theta)\nu(d\theta) = \sup_{A \in \mathcal{B}(\R^d)} \nu (A) - \inf_{A \in \mathcal{B}(\R^d)} \nu (A).
\end{equation*}
In particular, we recover the total-variation \emph{distance} between two probability distributions $\nu_1, \nu_2 \in \mathcal{M}_1(\R^d)$ as $d_{\tv}(\nu_1, \nu_2) = \|\nu_1 - \nu_2\|_{\tv}.$ We are now ready to state the geometric ergodicity result for the SGD Markov chain. A similar result to Theorem~\ref{thm:ergodicity} below can be found in~\cite{yu2021analysis}. However, we make a formal convergence statement in total-variation distance for the sake of completeness since it will be useful in the sequel.
\begin{theorem}\label{thm:ergodicity}
    Under Assumptions~\ref{asm:smooth_strongconvex} and~\ref{asm:gradient}\textup, the Markov chain $(\theta_t)_{t\geq 0}$ defined by iteration~\eqref{eq:sgd_iteration} with step-size 
    \begin{equation}\label{eq:thm1_stepsize}
        \gamma< \frac{2\mu}{\mu^2 + (\mu L\vee L^2_{\sigma})}
    \end{equation}
    admits a unique invariant measure $\pi_{\gamma}$ and converges geometrically to it.
    Namely, for any initial $\theta_0 \in \R^d,$ there exist $\rho < 1$ and $M < +\infty$ such that
    \begin{equation}
        \big\|\delta_{\theta_0} P_{\gamma}^n - \pi_{\gamma}\big\|_{\tv} \leq M \rho^n \big(1+\|\theta_0 - \theta^\star\|^2\big),\label{eq:tv_geo_convergence}
    \end{equation}    
    where $\delta_{\theta_0}$ is the Dirac measure located at $\theta_0.$
\end{theorem}
The proof of Theorem~\ref{thm:ergodicity} is given in Section~\ref{sec:proof_thm_ergodicity} and is based on~\cite[Theorem 15.0.1]{meyn2012markov} and a \emph{drift} condition in terms of a Lyapunov function. Assumptions~\ref{asm:smooth_strongconvex},~\ref{asm:gradient}~\ref{asm:gradient_centered} and~\ref{asm:gradient}~\ref{asm:gradient_regular} are standard convergence conditions for convex stochastic optimization ~\cite{dieuleveut2020bridging, bach2014adaptivity, rakhlin2011making}. Assumption~\ref{asm:gradient}~\ref{asm:gradient_dens_minor} is additionally needed to guarantee a Markov chain exploration property necessary for the convergence proof. A similar method was used in~\cite{yu2021analysis} to establish the convergence of SGD for non-convex, non-smooth objectives with quadratic growth. However, the focus in~\cite{yu2021analysis} is on proving a central limit theorem for the Markov sequence $(\theta_t)_{t\geq 0}$ and bounding the invariant distribution's bias under certain regularity conditions. In contrast, we aim to prove concentration properties for the SGD iterates and the invariant distribution $\pi_{\gamma}.$ This will allow us to obtain non-asymptotic deviation bounds on the estimation of the optimum $\theta^\star.$ In particular, convergence in TV distance~\eqref{eq:tv_geo_convergence} will ease this task for a Polyak-Ruppert average (see Section~\ref{sec:polyak_ruppert} below).

Note that condition~\eqref{eq:thm1_stepsize} imposes a conservative bound in $O(\mu/L_{\sigma}^2)$ on the step-size $\gamma.$ This condition may be restrictive compared to related works on stochastic optimization~\cite{dieuleveut2020bridging, bach2013non, needell2014stochastic, pillaud2018exponential}, especially in high-dimensional settings where $\mu$ is typically close to zero. However, by considering a linear regression example, one can show that there are situations where such scaling is actually necessary. Moreover, in such cases, the step-size assumptions used in the previously cited works also align with condition~\eqref{eq:thm1_stepsize}. See the discussion following Proposition~\ref{prop:wasserstein_convergence} below and Section~\ref{sec:stepsize_compare} for further details.

Note also that the focus of Theorem~\ref{thm:ergodicity} is to provide a convergence result although little can be said about the speed of this convergence for now. This is controlled by the contraction factor $\rho$ and the constant $M$ which mainly depend on the distribution of the noise samples $\varepsilon_{\zeta}(\theta),$ the step-size $\gamma$ and the initial state $\theta_0.$
Although the statement claims that $\rho < 1,$ the latter may be arbitrarily close to $1,$ especially for some degenerate noise distributions. This can happen, for instance, if the noise distribution is mostly concentrated on a few atoms causing the Markov chain to have poor mixing properties. A similar effect can be expected if there are no atoms but the distribution is highly concentrated around a few isolated points. 

The issue of providing a precise quantification of $\rho$ in Theorem~\ref{thm:ergodicity} is closely related to renewal theory and can be traced back to Kendall's theorem~\cite{kendall1959unitary} and more generally concerns Markov chains satisfying a drift property. A rich literature investigates the convergence speed of Markov chains with geometric drift~\cite{rosenthal1995convergence,roberts1999bounds,rosenthal2002quantitative,scott1996explicit,douc2004quantitative,meyn1994computable}. Near optimal results were obtained for stochastically ordered Markov Processes~\cite{lund1996computable, roberts2000rates, lund1996geometric, gaudio2019exponential}. Other examples especially amenable to such analysis include reversible Markov chains~\cite{diaconis1991geometric, diaconis1993comparison, jerison2019quantitative, roberts2001geometric} and chains satisfying special assumptions on their renewal distribution~\cite{berenhaut2001geometric, lund2006monotonicity, kijima1997markov}. However, the SGD Markov chain does not satisfy such criteria. For instance, reversibility does not hold since the iterates are driven towards the optimum and stochastic order fails because of the noise. An estimation of $\rho$ may be obtained using results based on renewal theory and Kendall's theorem~\cite{kendall1959unitary, baxendale2005renewal, bednorz2013kendall}. 
However, the resulting estimations are notoriously pessimistic~\cite{qin2021limitations, qin2022wasserstein}. Indeed, let $\alpha := 1-\gamma\mu$ be the contraction factor in the absence of gradient noise (i.e. simple gradient descent) so that we have
\begin{equation*}
    \|\theta_t - \gamma \nabla \cL(\theta_t) - \theta^\star\| \leq \alpha \|\theta_t - \theta^\star\|\quad \text{for all }\quad t\geq 0,
\end{equation*}
then the worst-case bound on $\rho$ obtained thanks to~\cite{baxendale2005renewal, bednorz2013kendall} is such that $1 - \rho\:\: \:{\lesssim}\:\:\: (\gamma \mu)^3$ which is far worse than the intuitive expectation that $\rho \approx \alpha$ i.e.~that TV convergence and optimization convergence would go hand in hand. Moreover, such an estimation would suffer from exponentially small minorization constants depending on the space dimension (see for instance~\cite{qin2021limitations, qin2022wasserstein, rajaratnam2015mcmc} for more detailed discussions of this phenomenon). It is unclear whether the previous estimation can be improved through a careful study of the renewal properties of the SGD Markov chain or if a different approach such as the study of the spectral properties of the transition kernel $P_\gamma$ is more appropriate. Nevertheless, we will see in Section~\ref{sec:wasserstein_convergence} below that $\rho$ can be estimated close to $\alpha$ under additional conditions by leveraging Wasserstein convergence.

\section{Iteration and Invariant Distribution Properties}
\label{sec:invariant_properties}
This section states that concentration properties of the random gradient samples used in~\eqref{eq:sgd_iteration} \emph{transfer} to the SGD iterates and the invariant distribution $\pi_\gamma$ they converge to as stated in Theorem~\ref{thm:ergodicity}. We begin with a basic statement which holds without additional assumptions and bounds the variance of $\pi_\gamma$ and its bias w.r.t.~the true optimum $\theta^\star.$
\begin{proposition}
\label{prop:invariant_properties}
    In the setting of Theorem~\ref{thm:ergodicity}\textup, let $\pi_\gamma$ be the invariant measure and $\bar{\theta}_\gamma := \E_{\theta\sim\pi_\gamma}[\theta]$ be its expectation. We have the following properties\textup:
    \begin{enumerate}[label=\textup(\alph*\textup)]
        \item \label{prop:invariant_properties_a} $\E_{\theta\sim\pi_\gamma} \big[\nabla\cL(\theta)\big] = 0.$ In particular\textup, if the gradient $\nabla\cL$ is linear \textup(see Assumption~\ref{asm:linear_grad} below\textup) then we have $\bar{\theta}_\gamma = \theta^\star$.
        \item \label{prop:invariant_properties_bprime} Denoting $\alpha_\sigma := (1-\gamma\mu)^2 + \gamma^2L_\sigma^2,$ the SGD iterates satisfy for all $t\geq 0,$
        \begin{equation*}
            \Var(\theta_t) \leq \E\|\theta_t - \theta^\star\|^2 \leq \alpha_\sigma^t\E\|\theta_{0} - \theta^\star\|^2 + \gamma^2\sigma^2\frac{1 - \alpha_\sigma^{t}}{1 - \alpha_\sigma}\nonumber,
        \end{equation*}
        \item \label{prop:invariant_properties_b} The variance and the bias of $\pi_\gamma$ are bounded as follows\textup: \begin{equation*}
            \Var_{\pi_\gamma}(\theta) \vee \big\|\bar{\theta}_\gamma - \theta^\star \big\|^2 \leq \E_{\theta\sim\pi_\gamma}\|\theta - \theta^\star\|^2\leq \frac{\gamma \sigma^2}{2\mu - \gamma(\mu^2+ L^2_{\sigma})}.
        \end{equation*} 
    \end{enumerate}
\end{proposition}
Proposition~\ref{prop:invariant_properties} is proven in Section~\ref{sec:proof_invariant_properties} and expresses well-known properties of the SGD iterates and the invariant distribution which we state here for completeness. A parallel to property~\ref{prop:invariant_properties_a} can be found in~\cite{bach2013non} and property~\ref{prop:invariant_properties_bprime} corresponds to~\cite[Lemma 10]{dieuleveut2020bridging}. Finally, property~\ref{prop:invariant_properties_b} reflects the known-fact that the iterates have an asymptotic magnitude of $\sqrt{\gamma}$~\cite{pflug1986stochastic, nedic2001convergence}. Beyond this result, a characterization of the \textit{covariance} of $\pi_\gamma$ in the linear case is given in~\cite[Proposition 3]{dieuleveut2020bridging}. Proofs of such results crucially rely on the unbiasedness of the gradient samples, the invariance of $\pi_\gamma,$ and the contraction property of the optimization iteration.

Before stating further results, we need to define sub-Gaussian and sub-exponential concentration properties for real random variables. Among the many known equivalent characterizations, we only introduce those required for the proofs of our results, see~\cite[Chapter 2]{vershynin2018high} for other characterizations.

\begin{definition}
    \label{def:subgauss}
    Let $X$ be a real random variable. We say that $X$ is $K$-sub-Gaussian for some $K > 0$ whenever
    \begin{enumerate}[label=\textup(\roman*\textup)]
        \item     we have 
        \begin{equation}
        \E\exp(\lambda^2 X^2) \leq \exp(\lambda^2K^2)\quad \text{for} \quad 0\leq \lambda \leq 1/K,
    \end{equation}
    which we will denote $X \in \widetilde{\Psi}_2(K),$ 
    \item or we have
    \begin{equation}
        \E\exp(\lambda X) \leq \exp(\lambda^2 K^2) \quad \text{ for all} \quad\lambda\in \R,
    \end{equation}
    which we will denote $X \in \Psi_2(K)$. 
    \end{enumerate}
\end{definition}
\begin{definition}
    \label{def:subexp}
    Let $X$ be a real random variable. We say that $X$ is sub-exponential if one of the two following conditions holds.
    \begin{enumerate}[label=\textup(\roman*\textup)]
    \item \label{def:subexp_through_pmoments} There exists $K_1 > 0$ such that
        \begin{equation}
            \|X\|_{L_p} \leq K_1p\quad \text{for all} \quad p\geq 1,
        \end{equation}
        in which case we write $X \in \widetilde{\Psi}_1(K_1).$
    \item \label{def:subexp_through_mgf} There exists $K_2$ such that
    \begin{equation}
        \E\exp(\lambda X) \leq \exp(\lambda^2 K_2^2) \quad \text{ for all}\quad |\lambda| \leq 1/K_2,
    \end{equation}
        in which case we write $X \in \Psi_1(K_2).$
    \end{enumerate}
\end{definition}
Note that, for a centered variable $X,$ the first characterization of Definition~\ref{def:subgauss} implies the second one with the same constant $K$ (see the proof of~\cite[Proposition 2.5.2]{vershynin2018high}). 
Analogously, for centered $X,$ we have that $X \in \Psi_1(K)$ entails $X \in \widetilde{\Psi}_1(2eK)$ and $X \in \widetilde{\Psi}_1(K)$ entails $X \in \Psi_1(2eK).$
Namely, the two characterizations of Definition~\ref{def:subexp} imply each other but with worse constants (see~\cite[Proposition 2.7.1]{vershynin2018high}). Since the constants in Definition~\ref{def:subexp} degrade by switching between the two properties, we will specify which property is meant in each subsequent statement in order to minimize these degradations.

We first formulate a sub-Gaussian/sub-exponential concentration assumption on the norms of the gradient errors.
\begin{assumption}\label{asm:grad_concentration}
    There exists $\overline{K} < +\infty$ such that one of the following holds\textup:
    \begin{enumerate}[label=\textup(\alph*\textup)]
        \item \label{asm:grad_subgauss} For all $\theta\in \Theta,$ the gradient error satisfies $\|\varepsilon_{\zeta}(\theta)\| \in \widetilde{\Psi}_2(\overline{K}).$
        \item \label{asm:grad_subexp} For all $\theta\in \Theta,$ the gradient error satisfies $\|\varepsilon_{\zeta}(\theta)\| \in \widetilde{\Psi}_1(\overline{K}).$
    \end{enumerate}
\end{assumption}
The sub-Gaussian concentration condition~\ref{asm:grad_subgauss} is verified, for instance, for logistic regression with Gaussian data. The sub-exponential condition~\ref{asm:grad_subexp} is more lenient and holds in the previous case for sub-exponential data or in linear regression with sub-Gaussian data by restricting the parameter $\theta$ to a bounded set.

In combination with Assumptions~\ref{asm:smooth_strongconvex} and~\ref{asm:gradient}, the previous pair of conditions imply the following concentration properties for the SGD iterates $(\theta_t)_{t\geq 0}$ and $\pi_\gamma.$
\begin{proposition}
    \label{prop:invariant_concentration}
    In the setting of Theorem~\ref{thm:ergodicity}\textup, let the SGD iteration~\eqref{eq:sgd_iteration1} be run starting from a deterministic $\theta_0,$ assume $\gamma\leq (2\mu)^{-1}$ and define for $t\geq 0,$
    \begin{equation*}
        \overline{K}_{\pi}(t) = \Big((1-\gamma\mu)^{t}\|\theta_0 - \theta^\star\|^2 + \frac{4}{3}\big(1-(1-\gamma\mu)^{t}\big)\gamma\overline{K}^2/\mu\Big)^{1/2}.
    \end{equation*}
    We have the following properties\textup:
    \begin{enumerate}[label=\textup(\alph*\textup)]
        \item \label{prop:invariant_properties_c} If Assumption~\ref{asm:grad_concentration}~\ref{asm:grad_subgauss} holds then for all $t,$ we have $\|\theta_t - \theta^\star\| \in\widetilde{\Psi}_2\big(\overline{K}_{\pi}(t)\big).$ Moreover, the invariant distribution satisfies that for $\theta \sim \pi_\gamma,$ we have $\|\theta - \theta^\star\| \in \widetilde{\Psi}_2 \big(2\overline{K}\sqrt{\gamma/\mu}\big).$
        \item \label{prop:invariant_properties_d} If Assumption~\ref{asm:grad_concentration}~\ref{asm:grad_subexp} holds then for all $t,$ we have $\|\theta_t - \theta^\star\| \in\widetilde{\Psi}_1\big(\overline{K}_{\pi}(t)\big).$ Moreover, for $\theta \sim \pi_\gamma,$ we have $\|\theta - \theta^\star\| \in \widetilde{\Psi}_1 \big(2\overline{K}\sqrt{\gamma/\mu}\big).$
    \end{enumerate}
\end{proposition}
The proof of Proposition~\ref{prop:invariant_concentration} is given in Section~\ref{sec:proof_invariant_concentration}. The most important aspect of this statement is that the sub-Gaussian/sub-exponential properties hold with a constant \emph{depending} on the step-size $\gamma.$ Indeed, it is fairly easy to show, for example, that $\theta\sim\pi_\gamma$ satisfies $\|\theta - \theta^\star\|\in \widetilde{\Psi}_2(\overline{K}/\mu)$ under Assumption~\ref{asm:grad_concentration}~\ref{asm:grad_subgauss}. However, this constant is too pessimistic since it fails to take advantage of a small step-size which leads to stronger concentration. The improved constants above are obtained by carefully leveraging the centered property of the gradient error (see Assumption~\ref{asm:gradient}~\ref{asm:gradient_centered}). Previous characterizations of $\pi_\gamma$ obtained bounds on the bias w.r.t.~$\theta^\star$~\cite{yu2021analysis} and moment expansions of $\widehat{\theta} - \theta^\star$ for $\widehat{\theta}\sim \pi_\gamma$ or $\widehat{\theta}$ equal to a Polyak-Ruppert average~\cite{dieuleveut2020bridging}, however, the sub-exponential and sub-Gaussian characterizations of Proposition~\ref{prop:invariant_concentration} appear to be new. Note that the sub-Gaussian property of $\pi_\gamma$ in Proposition~\ref{prop:invariant_concentration} can also be obtained if Assumption~\ref{asm:smooth_strongconvex} is replaced by the dissipativity condition~\cite{yu2021analysis,raginsky2017non,xu2018global} on the objective and a linear gradient growth constraint. This is detailed in Section~\ref{sec:weaker_subgauss} and allows to consider non-convex, non-smooth objectives but with quadratic growth. However, a global minimizer $\theta^\star$ may not exist in this case and the associated sub-Gaussian constant does not vanish for a small step-size.

Note that the constant $\overline{K}$ obtained from Assumption~\ref{asm:grad_concentration} and appearing in Proposition~\ref{prop:invariant_concentration} may hide a dependence on the dimension in $\sqrt{d}$ since it is related to the Euclidean norm $\|\varepsilon_{\zeta}(\theta)\|$ of the gradient noise. In this respect, Proposition~\ref{prop:invariant_concentration} resembles the results of~\cite{maurer2021concentration} where similar hypotheses to Assumption~\ref{asm:grad_concentration} were used entailing the same dimension dependence.~In order to avoid this shortcoming, one needs a stronger assumption which is stated along with the associated results further below. Note also that Assumption~\ref{asm:grad_concentration} considerably strengthens Assumption~\ref{asm:gradient}~\ref{asm:gradient_regular} by requiring that $\|\varepsilon_{\zeta}(\theta)\|$ admits a finite exponential moment. 
In addition, the involved bound is uniform w.r.t. $\theta.$ However, under a non-uniform finite $p$-moment assumption, it is still possible to show the following.
\begin{lemma}
    \label{lem:pfinite_moments_nonunif_subexp}
    Grant Assumptions~\ref{asm:smooth_strongconvex} and~\ref{asm:gradient} and assume that there is $K, \underline{K}> 0$ and $p\in \N^*$ such that, conditionally on any $\theta,$ we have
    \begin{equation}\label{eq:Lp_condition}
        \big\|\|\varepsilon_{\zeta}(\theta)\|\big\|_{L_p} \leq K\|\theta - \theta^\star\| + \underline{K},
    \end{equation}
    then for step-size $\gamma$ as in Theorem~\ref{thm:ergodicity} and satisfying the additional condition $\gamma \leq \frac{\mu}{j(\mu^2 + K^2)}$ with $j\leq p,$ the Markov chain $(\theta_t)_{t\geq 0}$ converges to an invariant distribution $\pi_\gamma$ with at least $j$ finite moments.
\end{lemma}
Lemma~\ref{lem:pfinite_moments_nonunif_subexp} is proved in Section~\ref{sec:proof_lem_pfinite_moments_nonunif_subexp} and shows that $\pi_\gamma$ can have as many finite moments as the gradient, provided that the step-size is small enough. This implies that even weaker concentration properties transfer to the invariant distribution. Note that a non-uniform sub-exponential (resp.~sub-Gaussian) assumption would correspond to condition~\eqref{eq:Lp_condition} with $K$ replaced by $Kp$ (resp.~$K\sqrt{p}$) in which case the condition on $\gamma$ becomes at least $\gamma \leq \cO\big(\mu/(jK^2p) \big)$. This suggests that, for arbitrary $p$, $\pi_\gamma\in L_p$ may only hold in the limit $\gamma \to 0$. The results and proof method of Lemma~\ref{lem:pfinite_moments_nonunif_subexp} and Proposition~\ref{prop:invariant_properties}~\ref{prop:invariant_properties_d} share many aspects with~\cite[Lemma 13]{dieuleveut2020bridging}, we provide a detailed comparison in Section~\ref{sec:compare_with_18}. Lemma~\ref{lem:pfinite_moments_nonunif_subexp} also shares a similarity with~\cite[Theorem 2]{bach2013non} which states $L_p$ convergence of the excess risk $\cL(\theta)-\cL(\theta^\star)$ for a step-size satisfying $\gamma\leq O(1/(p\kappa R^2))$ where $R^2$ and $\kappa$ respectively correspond to a uniform bound on the data samples and the distribution's kurtosis after projection in any direction in space. Although~\cite[Theorem 2]{bach2013non} provides an explicit bound, the assumption of almost surely bounded data is strong and is combined with a uniform $L_p$ condition on the gradient error whereas we allow the upperbound to depend on $\|\theta - \theta^\star\|$ in Inequality~\eqref{eq:Lp_condition}.

We now introduce a stronger analog to Assumption~\ref{asm:grad_concentration}, which will enable later the proof of dimension-free deviation bounds.
\begin{assumption}\label{asm:grad_special_concentration}
    There is $ K < +\infty$ such that one of the following holds\textup:
    \begin{enumerate}[label=\textup(\alph*\textup)]
        \item \label{asm:grad_special_subgauss} For all $\theta \in \Theta$ and all $f \in \Lip(\R^d),$ we have $f(G(\theta, \zeta)) - \E f(G(\theta, \zeta)) \in \Psi_2(K).$
        \item \label{asm:grad_special_subexp} For all $\theta \in \Theta$ and all $f \in \Lip(\R^d),$ we have $f(G(\theta, \zeta))-\E f(G(\theta, \zeta)) \in \Psi_1(K).$
    \end{enumerate}
\end{assumption}
As announced, the subtle difference with Assumption~\ref{asm:grad_concentration} is that the involved constants are, a priori, independent from the dimension. The so-called Bobkov-Götze theorem~\cite{bobkov1999exponential} states that Assumption~\ref{asm:grad_special_concentration}~\ref{asm:grad_special_subgauss} is equivalent to the fact that $\nu_{\theta} :=\mathcal{D}(G(\theta, \zeta))$ satisfies the following \emph{Transportation-Information} inequality
\begin{equation}\label{eq:transport}
    \cW_1\big(\nu, \nu_{\theta}\big) \leq \sqrt{2K^2 D(\nu \| \nu_{\theta})} \quad \text{ for all } \quad \nu \in \mathcal{M}_1(\R^d),
\end{equation}
where $\cW_1$ and $D(\cdot\|\cdot)$ are the Wasserstein-$1$ distance~\cite{villani2009optimal} (see definition below) and the Kullback-Leibler divergence~\cite{kullback1951information} between probability measures respectively. An analogous equivalence may be established for the sub-exponential case of Assumption~\ref{asm:grad_special_concentration}~\ref{asm:grad_special_subexp} (for instance, by adapting the proof given in~\cite[Theorem 4.8]{van2014probability}).

By restricting the functions $f$ in Assumption~\ref{asm:grad_special_concentration} to be linear, we recover the assumption that the vector $G(\theta, \zeta)$ is sub-Gaussian/sub-exponential. An interesting question is then whether this weaker property implies Assumption~\ref{asm:grad_special_concentration} with a dimension independent constant. To the best of our knowledge of the current literature, this is only known to hold for Gaussian vectors (see for instance~\cite[Theorem 3.25]{van2014probability}). In fact, Talagrand's well-known transport inequality states that Gaussian vectors satisfy Inequality~\eqref{eq:transport} for the $\cW_2$ distance rather than $\cW_1$, which is an even stronger property. Since Inequality~\eqref{eq:transport} involves two very different forms of distance between probability measures, a direct intuitive understanding of its meaning is elusive. However, the above inequality is related to a host of properties used to describe the concentration of measure phenomenon including Poincaré inequalities~\cite{bobkov1997poincare, gozlan2009characterization}, logarithmic Sobolev inequalities~\cite{bobkov1999exponential, ledoux1997talagrand} and modified logarithmic Sobolev inequalities~\cite{gentil2005modified,barthe2008modified} to mention only a few. A broad and comprehensive survey on transport inequalities and their consequences on concentration and deviation inequalities is available in~\cite{gozlan2010transport}.

Using the previous assumption, we can show that the iterates $(\theta_t)_{t\geq 0}$ and the invariant distribution inherit similar properties.
\begin{proposition}
    \label{prop:invariant_special_concentration}
    In the setting of Theorem~\ref{thm:ergodicity} {with step-size condition~(\ref{eq:thm1_stepsize})}\textup, let the SGD iteration~\eqref{eq:sgd_iteration1} be run starting from a deterministic $\theta_0$ and let $\pi_\gamma$ be the invariant limit distribution. Define for $t\geq 0,$
    \begin{equation*}
        K_{\pi}(t) = \gamma K \sqrt{\frac{ 1 - (1-\gamma\mu)^{2t}}{1 - (1-\gamma\mu)^2}}.
    \end{equation*}
    We have the following properties\textup:
    \begin{enumerate}[label=\textup(\alph*\textup)]
        \item \label{prop:invariant_properties_e} If Assumption~\ref{asm:grad_special_concentration}~\ref{asm:grad_special_subgauss} holds then $f(\theta_t ) - \E f(\theta_t ) \in \Psi_2\big(K_{\pi}(t)\big)$ for all $f\in \Lip(\R^d).$
        Moreover, for $\theta\sim\pi_\gamma$ we have $f(\theta ) - \E f(\theta ) \in \Psi_2\big(K\sqrt{\gamma/\mu}\big)$ for all $f\in \Lip(\R^d).$
        \item \label{prop:invariant_properties_f} If Assumption~\ref{asm:grad_special_concentration}~\ref{asm:grad_special_subexp} holds then $f(\theta_t ) - \E f(\theta_t ) \in \Psi_1\big(K_{\pi}(t)\big)$ for all $f\in \Lip(\R^d).$
        Moreover $\theta\sim\pi_\gamma$ satisfies $f(\theta ) - \E f(\theta ) \in \Psi_1\big(K\sqrt{\gamma/\mu}\big)$ for all $f\in \Lip(\R^d).$
    \end{enumerate}
\end{proposition}
Proposition~\ref{prop:invariant_special_concentration} is proven in Section~\ref{sec:proof_invariant_special_concentration} and will be used in Section~\ref{sec:deviation_bounds} to derive dimension-free deviation bounds. Note that the $\Psi_1$/$\Psi_2$ constants of $\pi_\gamma$ in Proposition~\ref{prop:invariant_special_concentration} also display the crucial $\sqrt{\gamma/\mu}$ dependence as in Proposition~\ref{prop:invariant_concentration} and without further degradation. Before proceeding to the statement of high confidence bounds for SGD estimators, we explore another convergence mode of the SGD Markov chain.

\section{Wasserstein Convergence}\label{sec:wasserstein_convergence}
This section complements Theorem~\ref{thm:ergodicity} with an additional convergence result w.r.t. the Wasserstein metric. We recall that, for $p \geq 1$ and two distributions $\varpi, \nu  \in \mathcal{M}_1(\R^d),$ the Wasserstein-$p$ distance is defined by
\begin{equation*}
    \cW_p^p(\varpi, \nu) = \inf_{\xi \in \Pi(\varpi, \nu)} \E_{X, Y \sim \xi} \|X - Y\|^p,
\end{equation*}
where $\Pi(\varpi, \nu)$ is the set of all couplings of $\varpi$ and $\nu$ i.e. distributions over $\R^d \times \R^d$ with first and second marginals equal to $\varpi$ and $\nu$ respectively.

In order to show that the SGD iteration converges w.r.t.~the Wasserstein-$2$ distance, we require the following assumption.
\begin{assumption}
    \label{asm:gradient_wasserstein}
    There is $L_{\cW}< +\infty$ such that for all $\theta, \theta',$ the gradient noise distributions $\mathcal{D}(\varepsilon_{\zeta}(\theta))$ and $\mathcal{D}(\varepsilon_{\zeta}(\theta'))$ at $\theta$ and $\theta'$ satisfy
    \begin{equation*}
        \cW_2\big(\mathcal{D}(\varepsilon_{\zeta}(\theta)), \mathcal{D}(\varepsilon_{\zeta}(\theta'))\big) \leq L_{\cW} \|\theta - \theta' \|.
    \end{equation*}
\end{assumption}
In words, we assume that the change in the gradient noise distribution measured with the $\cW_2$ metric is controlled by the change in the parameter $\theta.$ This assumption is discussed below and allows to obtain the following result.
\begin{proposition}\label{prop:wasserstein_convergence}
Grant Assumptions~\ref{asm:smooth_strongconvex},~\ref{asm:gradient}~\ref{asm:gradient_centered} and~\ref{asm:gradient_wasserstein}. Let $\nu_1, \nu_2 \in \mathcal{P}_2(\R^d)$ be two initial distributions and let $\gamma$ be a step-size such that 
\begin{equation*}
    \gamma < \frac{2\mu}{\mu^2 + (\mu L\vee L^2_{\cW}) },
\end{equation*}
then we have the contraction
\begin{equation*}
    \cW_2^2(\nu_1 P_\gamma, \nu_2 P_\gamma) \leq \big((1-\gamma\mu)^2 + \gamma^2 L^2_{\cW} \big)\cW_2^2(\nu_1, \nu_2).
\end{equation*}
Consequently\textup, for such a $\gamma$ and any initial $\theta_0 \sim \nu \in \mathcal{P}_2(\R^d),$ the Markov chain generated by iteration~\eqref{eq:sgd_iteration} converges to a unique stationary measure $\pi_{\gamma}$ in $\cW_2$ distance.
\end{proposition}
The proof of Proposition~\ref{prop:wasserstein_convergence} is given in Section~\ref{sec:proof_wasserstein_convergence}. The intuition behind it is that, if the Markov chain evolves according to a locally similar dynamic when started from different points then, for small enough step-size, the contraction phenomenon coming from the optimization will prevail so that trajectories associated to different initializations join even before convergence. A similar result was previously stated in~\cite[Proposition 2]{dieuleveut2020bridging} for smooth and strongly convex functions as well. 
In~\cite{dieuleveut2020bridging}, Wasserstein convergence is shown under the assumption that every random gradient $G(\theta, \zeta)$ be almost surely co-coercive with fixed constant. Denoting $L' > 0$ the said constant\footnote{We refer to the $L$ constant defined in~\cite{dieuleveut2020bridging} as $L'$ in order to avoid confusion with our own definition of $L.$}, this corresponds to assuming that for all $\theta, \theta'$ and $\zeta$ we have the inequality
\begin{equation*}
    L'\langle G(\theta, \zeta) - G(\theta', \zeta), \theta - \theta'\rangle \geq  \|G(\theta, \zeta) - G(\theta', \zeta)\|^2.
\end{equation*}
Nonetheless, they mention that the proof also works when this property holds only in expectation (see~\cite[Assumption~A7]{dieuleveut2020bridging}), which translates to the following inequality for all $\theta, \theta':$
\begin{equation}\label{eq:asm_A7}
    L'\langle \nabla\cL(\theta) - \nabla\cL(\theta'), \theta - \theta'\rangle \geq \E\big[\|G(\theta, \zeta) - G(\theta', \zeta)\|^2\big].
\end{equation}
For the sake of illustration, we consider the simple example of least-squares linear regression in which, given a sample $\zeta = (X, { Y}) \in \R^d \times \R$, a random gradient is computed as
\begin{equation*}
    G(\theta, \zeta) = XX^\top \theta -XY \quad \text{with}\quad Y = X^\top\theta^\star + \xi,
\end{equation*}
where $\xi$ is an independent centered noise and the lowest and highest eigenvalues of $\Sigma := \E [XX^\top]$ are $(\lambda_{\min}(\Sigma), \lambda_{\max}(\Sigma)) = (\mu, L)$ and we assume $\mu>0.$ In this particular case, Inequality~\eqref{eq:asm_A7} can be verified as long as $X$ has a bounded fourth moment. Indeed, we have $\nabla\cL(\theta) = \Sigma(\theta - \theta^\star)$ and~\eqref{eq:asm_A7} can be verified by finding $L'$ such that, for all $\theta, \theta'$
\begin{align*}
    \E\big[\|G(\theta, \zeta) - G(\theta', \zeta)\|^2\big] &= \E\big[\|XX^\top(\theta - \theta')\|^2\big] \\
    &= \|\theta - \theta'\|^2_{\E[\|X\|^2 XX^\top]} \\
    &\leq L'\langle \nabla\cL(\theta) - \nabla\cL(\theta'), \theta - \theta'\rangle = L'\|\theta - \theta'\|^2_{\Sigma},
\end{align*}
where we used the notation $\|v\|^2_{A} = v^\top A v$ for a vector $v\in\R^n$ and a symmetric positive definite matrix $A \in R^{n\times n}.$

Regarding Assumption~\ref{asm:gradient_wasserstein}, we have 
\begin{equation*}
    \varepsilon_{\zeta}(\theta) = G(\theta, \zeta) - \nabla \cL(\theta) = (XX^\top - \Sigma)(\theta - \theta^\star) - X\xi,
\end{equation*}
and it is easy to couple the distributions of $\varepsilon_{\zeta}(\theta)$ and $\varepsilon_{\zeta'}(\theta')$ by defining them with the same variables $\zeta = \zeta' = (X, { Y})$ so that we find\begin{equation*}
    \cW_2^2\big(\mathcal{D}(\varepsilon_{\zeta}(\theta)), \mathcal{D}(\varepsilon_{\zeta'}(\theta'))\big) \leq \E\|\varepsilon_{\zeta}(\theta) - \varepsilon_{\zeta'}(\theta')\|^2 {= \|\theta - \theta'\|^2_{\E (XX^\top - \Sigma)^2}}.
\end{equation*}
Assumption~\ref{asm:gradient_wasserstein} is then verified with {$L^2_\cW = \big\|\E (XX^\top - \Sigma)^2\big\|_2$ where $\|\cdot\|_2$ is the operator norm} and we recover the bounded fourth moment condition on $X.$ In this particular setting, one can also check that Assumption~\ref{asm:gradient}~\ref{asm:gradient_regular} holds with $L_\sigma$ equal to this choice of $L_\cW.$ 

It is important to note that the constant $L'$ used by~\cite{dieuleveut2020bridging} in~\eqref{eq:asm_A7} is a different constant from $L.$ For the case of linear regression, considering $\theta - \theta'$ aligned with the top eigenvector of $\Sigma$ in~\eqref{eq:asm_A7} implies $L'\geq L.$ It is unclear how $L'$ may depend on $L_{\sigma}^2$ or $\mu$ in the general case. However, one can show that, for some particular distributions of $(X, Y),$ one can choose $\theta - \theta'$ simultaneously aligned with the least eigenvector of $\Sigma$ and the top eigenvector of $\E(XX^\top - \Sigma)^2$ (see Section~\ref{sec:stepsize_compare} for a simple example where this happens). In this case and for such a choice of $\theta - \theta',$~\eqref{eq:asm_A7} implies
\begin{align*}
    L'\mu\|\theta - \theta'\|^2 &\geq \E\big[\|XX^\top(\theta - \theta')\|^2\big] \\
    &= \E\big[\|(XX^\top - \Sigma)(\theta - \theta')\|^2\big] + \|\Sigma(\theta - \theta')\|^2 \\
    &= \|\theta-\theta'\|^2_{\E(XX^\top - \Sigma)^2} + \|\theta-\theta'\|^2_{\Sigma^2}\\
    &= L_{\cW}^2\|\theta-\theta'\|^2 + \mu^2\|\theta-\theta'\|^2,    
\end{align*}
hence, it follows that $L'\geq L^2_\cW/\mu + \mu.$ Consequently, in this case, the step-size condition of Proposition~\ref{prop:wasserstein_convergence} (and Theorem~\ref{thm:ergodicity}) has the same scale as the condition $\gamma\leq 2/L'$ used in~\cite{dieuleveut2020bridging}.

Assumption~\ref{asm:gradient_wasserstein} is more general as it allows to consider an objective $\cL$ defined by a linear learning task {on random samples and labels $(X, Y) = \zeta$} such that $\cL(\theta) = \E_{\zeta}[\ell(X^\top \theta, Y)]$ for a convex {smooth} loss $\ell$ so that the gradient samples are $G(\theta, {\zeta}) = X \ell'(X^\top \theta, Y)$ with $\ell'$ the derivative in the first argument. {One can} verify Assumption~\ref{asm:gradient_wasserstein} as soon as $\ell$ is smooth in its first argument and $X$ has a finite fourth moment. {Indeed, let $\chi > 0$ be the smoothness constant such that for all $x, y, z \in \R$ it holds that
\begin{equation*}
    \big|\ell'(x, z) - \ell'(y, z)\big| \leq \chi |x - y|,
\end{equation*}
for $\theta, \theta' \in \R^d$ and $\zeta = \zeta',$ using Jensen's inequality, we have
\begin{align*}
    \cW_2^2\big(\mathcal{D}(\varepsilon_{\zeta}(\theta)), &\mathcal{D}(\varepsilon_{\zeta'}(\theta'))\big) \leq \E\|\varepsilon_{\zeta}(\theta) - \varepsilon_{\zeta'}(\theta')\|^2 \\
    &= \E\big\|G(\theta, \zeta) - G(\theta', \zeta) - (\nabla\cL(\theta) - \nabla\cL(\theta'))\big\|^2 \\
    &\leq 2\E\big\|G(\theta, \zeta) - G(\theta', \zeta)\big\|^2 + 2\big\|\nabla\cL(\theta) - \nabla\cL(\theta')\big\|^2 \\
    &\leq 4\E\big\|G(\theta, \zeta) - G(\theta', \zeta)\big\|^2 = 4\E\big\|X(\ell'(X^\top\theta, Y) - \ell'(X^\top\theta', Y))\big\|^2 \\
    &\leq 4\chi^2\E\big[\|X\|^2\cdot \big|X^\top(\theta - \theta')\big|^2\big] \leq 4\chi^2\E\|X\|^4\cdot \|\theta - \theta'\|^2, 
    \end{align*}
which shows that Assumption~\ref{asm:gradient_wasserstein} holds under the previous conditions with $L_{\cW}^2 = 4\chi^2\E\|X\|^4.$} On the other hand, the fact that it is unclear how to establish~\eqref{eq:asm_A7} in this setting makes Assumption~\ref{asm:gradient_wasserstein} more generic.

In the same vein as Assumption~\ref{asm:gradient_wasserstein}, it is possible to introduce a regularity condition on the transition kernel $P_\gamma$ in terms of the $\tv$ distance which allows to obtain the following result.
\begin{proposition}\label{prop:wasserstein_to_tv}
    Let the assumptions of Proposition~\ref{prop:wasserstein_convergence} hold and further assume that$:$
    \begin{itemize}
        \item For all $\theta\in \R^d$ the probability measure $P_\gamma(\theta, \cdot)$ admits a density $p_\gamma(\theta, \omega)$ w.r.t. Lebesgue's measure.
        \item There exists $A < \infty$ such that for all $\theta, \theta' \in \R^d$ 
        \begin{equation}\label{eq:wasserstein_to_tv}
            \|P_\gamma(\theta, \cdot) - P_\gamma(\theta', \cdot)\|_{\tv} = \frac{1}{2}\int_{\R^d}|p_\gamma(\theta, \omega) - p_\gamma(\theta', \omega)|d\omega \leq A\|\theta - \theta'\|.
        \end{equation}
    \end{itemize}
    Then, for all $\theta_0 \in \R^d,$ we have$:$
    \begin{equation*}
        \big\|\delta_{\theta_0} P_{\gamma}^n - \pi_{\gamma}\big\|_{\tv} \leq A \rho^{n-1} \Big(\int_{\R^d} \|\theta - \theta_0\|^2 d\pi_{\gamma}\Big)^{1/2},
    \end{equation*}
    where $\rho \leq \sqrt{(1-\gamma\mu)^2 + \gamma^2 L^2_{\cW} }.$
\end{proposition}
\begin{proof}
    Using~\cite[Theorem 12]{madras2010quantitative} (see also~\cite[Lemma 13]{madras2010quantitative}), the assumptions imply that for all $\varpi, \nu \in \mathcal{M}_1(\R^d)$ we have:
    \begin{equation*}
        \|\varpi P_\gamma - \nu P_\gamma\|_{\tv} \leq A \mathcal{W}_1(\varpi, \nu).
    \end{equation*}
    It then only remains to use Proposition~\ref{prop:wasserstein_convergence} with $\varpi = \delta_{\theta_0}P_\gamma^{n-1}$ and $\nu = \pi_\gamma = \pi_\gamma P_\gamma$ along with the inequality:
    \begin{equation*}
        \mathcal{W}_1(\varpi, \nu) \leq \sqrt{\mathcal{W}_2^2(\varpi, \nu)}
    \end{equation*}
    valid for all $\varpi, \nu \in \mathcal{M}_1(\R^d),$ and the identity $\mathcal{W}_2^2(\delta_{\theta_0}, \pi_{\gamma}) = \int_{\R^d} \|\theta - \theta_0\|^2 d\pi_{\gamma}.$
\end{proof}
Proposition~\ref{prop:wasserstein_to_tv} uses the ``Wasserstein-to-TV'' method~\cite{qin2022wasserstein, madras2010quantitative} in order to derive convergence in $\tv$ distance from Proposition~\ref{prop:wasserstein_convergence} which leads to an explicit estimate of the convergence speed in Theorem~\ref{thm:ergodicity}. While the latter relies on Assumption~\ref{asm:gradient}~\ref{asm:gradient_dens_minor}, Proposition~\ref{prop:wasserstein_to_tv} replaces it with the stronger density requirement over $P_\gamma(\theta, \cdot)$ together with condition~\eqref{eq:wasserstein_to_tv}. Keep in mind that this excludes mini-batch SGD or full gradient descent on an empirical objective $\widehat{\cL}(\theta)=\frac{1}{n}\sum_{i=1}^n\ell_i(\theta)$ since the transition distribution would be a combination of Diracs. However, Proposition~\ref{prop:wasserstein_convergence} may still apply in this case. Similarly, Assumption~\ref{asm:gradient}~\ref{asm:gradient_regular} in Theorem~\ref{thm:ergodicity} is replaced by Assumption~\ref{asm:gradient_wasserstein} in Proposition~\ref{prop:wasserstein_to_tv}. Finally, note that although the necessary condition~\eqref{eq:wasserstein_to_tv} is quite intuitive, its verification is not straightforward even for a toy example.

\section{Confidence bounds}\label{sec:deviation_bounds}
Using the convergence and concentration results formulated in the previous sections for the iteration and invariant distribution of the SGD Markov chain, we are ready to state confidence bounds on the estimation of the optimal $\theta^\star.$ Recall that by Proposition~\ref{prop:invariant_properties}, the invariant distribution $\pi_\gamma$ may not be centered around $\theta^\star$ unless the gradient is linear, which is a particular case. In general, the expectation of $\pi_\gamma$ may not be equal to $\theta^\star$ but the bias is controlled by the step-size $\gamma.$ Therefore, two possibilities are available for the final estimator:
\begin{itemize}
    \item The last iterate $\theta_T$: with $T$ the optimization horizon. In which case a small step-size is appropriate.
    \item A tail average $\frac{1}{n}\sum_{j=n_0+1}^{n_0+n} \theta_j$: in which case the step-size may be chosen reasonably large within the convergence conditions.
\end{itemize}

\subsection{Final iterate concentration bounds}

When the expectation of the invariant measure $\overline{\theta}_\gamma$ differs from the true optimum $\theta^\star$, one may choose a small step-size $\gamma$ in order to obtain a precise estimator of $\theta^\star$ through the final iterate of the SGD sequence~\eqref{eq:sgd_iteration}. When the conditions of Assumption~\ref{asm:grad_concentration} are fulfilled, the consequences of Proposition~\ref{prop:invariant_concentration} lead to the following first deviation bounds.
\begin{corollary}
    \label{cor:concentration}
    In the setting of Proposition~\ref{prop:invariant_concentration}\textup, let $\delta \in(0, 1/2)$ be a confidence level and assume the horizon $T$ large enough to allow a step-size
    \begin{equation}\label{eq:cor_stepsize}
        \gamma = \frac{\log\big(\overline{A}_{\theta_0}T\big)}{\mu T} \leq \frac{\mu}{\mu^2 + (\mu L\vee L^2_{\sigma})},
    \end{equation}
    where $\overline{A}_{\theta_0}:=\mu^2\|\theta_0 - \theta^\star\|^2/\overline{K}^2.$ Then, we have the following high-confidence bounds\textup:
    \begin{enumerate}[label=\textup(\alph*\textup)]
    \item Under Assumption~\ref{asm:grad_concentration}~\ref{asm:grad_subgauss}\textup, with probability at least $ 1-\delta,$
    \begin{equation*}
        \big\|\theta_T - \theta^\star\big\| \leq \frac{\overline{K}}{\mu\sqrt{T}}\sqrt{1 + \log\big(\overline{A}_{\theta_0}T\big)\big(1+4\log(1/\delta)\big) },
    \end{equation*}

    \item Under Assumption~\ref{asm:grad_concentration}~\ref{asm:grad_subexp}\textup, with probability at least $1-\delta,$
    \begin{equation*}
        \big\|\theta_T - \theta^\star\big\| \leq \frac{2e\overline{K}\log(1/\delta)}{\mu\sqrt{T}}\sqrt{1 + 4\log\big(\overline{A}_{\theta_0}T\big)\big) },
    \end{equation*}
    \end{enumerate}
\end{corollary}
The proof of Corollary~\ref{cor:concentration} is given in Section~\ref{sec:proof_cor_concentration}. The step-size choice~\eqref{eq:cor_stepsize} corresponds to $\gamma=O(\log(T)/T)$ and allows to recover the nearly optimal statistical rate of $\sqrt{\log(T)/T}$ in the $\Psi_2/\Psi_1$ constants given by Proposition~\ref{prop:invariant_concentration}. An alternative way to obtain such confidence bounds is to use the concentration properties of $\pi_\gamma$ directly to bound $\Proba_{\theta\sim\pi_\gamma}\big(\mathcal{E}(\theta)\big)$ with $\mathcal{E}(\theta)=\{\|\theta - \theta^\star\| > \epsilon\}$ and combine this with TV convergence (Theorem~\ref{thm:ergodicity}) in order to bound the difference in probabilities 
\begin{equation*}
    \Proba_{\theta_T\sim\delta_{\theta_0}P^T_\gamma}\big(\mathcal{E}(\theta_T)\big) - \Proba_{\theta\sim\pi_\gamma}\big(\mathcal{E}(\theta)\big) \leq \|\delta_{\theta_0}P^T_\gamma - \pi_\gamma\|_{\tv}.
\end{equation*}
A first obstacle to this method is that Theorem~\ref{thm:ergodicity} lacks quantification of the contraction factor $\rho$ in terms of $\gamma,$ which is of particular concern when the latter is in $\widetilde{O}(1/T).$ This difficulty can be sidestepped by granting the assumptions of Proposition~\ref{prop:wasserstein_to_tv} providing an explicit bound on the TV distance. However, this requires $\gamma = O(\log(1/\delta)/T)$ in order to ensure $\|\delta_{\theta_0}P^T_\gamma - \pi_\gamma\|_{\tv}\leq \delta$ and replaces the $O(\sqrt{\log(T)})$ sub-optimality in Corollary~\ref{cor:concentration} by another one in $O(\sqrt{\log(1/\delta)}).$ This turns out to be much worse since the confidence level $\delta$ scales as $\exp(-T)$ or $\exp(-\sqrt{T})$ in the sub-Gaussian and sub-exponential cases respectively, hence the preference for the result above.

Although combining the properties of $\pi_\gamma$ with TV convergence proves to be inappropriate for $\gamma=\widetilde{O}(1/T),$ the associated issue resolves for constant order step-sizes. This will be explored in the next section and allow for simpler proofs.

As discussed earlier, the constants $\overline{K}$ drawn from Assumption~\ref{asm:grad_concentration} may have a poor dependence on the dimension in $\sqrt{d}$ which leaves room for improvement in the above bounds. This can be achieved when the requirements of Assumption~\ref{asm:grad_special_concentration} are met leading to the following \emph{dimension-free} deviation bounds.
\begin{corollary}
    \label{cor:special_concentration}
    In the setting of Proposition~\ref{prop:invariant_special_concentration}\textup, let $\delta \in(0, 1/2)$ be a confidence level and assume the horizon $T$ large enough to allow the step-size
    \begin{equation*}
        \gamma = \frac{\log(A_{\theta_0}T)}{\mu T} \leq \frac{\mu}{\mu^2 + (\mu L \vee L_\sigma^2)}
    \end{equation*}
    where $A_{\theta_0}=\mu^2\|\theta_0 - \theta^\star\|^2/\sigma^2.$ Then, we have the following high-confidence bounds\textup:
    \begin{enumerate}[label=\textup(\alph*\textup)]
    \item Under Assumption~\ref{asm:grad_special_concentration}~\ref{asm:grad_special_subgauss}\textup, with probability at least $ 1-\delta,$
    \begin{equation}\label{eq:special_concentration_subgauss}
        \big\| \theta_{T} - \theta^\star\big\| \leq \frac{\sigma}{\mu \sqrt{T}} + \frac{\sqrt{\log(A_{\theta_0} T)}}{\mu\sqrt{T}}\big(\sigma + 2K\sqrt{\log(1/\delta)}\big).
    \end{equation}

    \item Under Assumption~\ref{asm:grad_special_concentration}~\ref{asm:grad_special_subexp}\textup, with probability at least $1-\delta,$ 
    \begin{align}
        \big\| \theta_{T} - \theta^\star\big\| \leq & \frac{\sigma}{\mu\sqrt{T}}+\frac{\sqrt{\log(A_{\theta_0} T)}}{\mu\sqrt{T}}\bigg(\sigma +\nonumber\\
        &\quad 2K\sqrt{\log(1/\delta)} \bigg(1 \vee \sqrt{\frac{\log(A_{\theta_0}T) \log(1/\delta)}{T}} \bigg)\bigg).\label{eq:special_concentration_subexp}
    \end{align}
    \end{enumerate}

\end{corollary}
Corollary~\ref{cor:special_concentration} is proven in Section~\ref{sec:proof_cor_special_concentration} and uses Proposition~\ref{prop:invariant_special_concentration} as opposed to Proposition~\ref{prop:invariant_concentration} in Corollary~\ref{cor:concentration}. As announced, this new set of inequalities improves upon the previous ones by removing the uncertainty terms' potential dependency in the dimension thanks to Assumption~\ref{asm:grad_special_concentration}. This can be assessed by checking that the terms with $\log(1/\delta)$ have the factor $K$ which is dimension-free as opposed to $\sigma.$ In this respect, Inequality~\eqref{eq:special_concentration_subgauss} is an example of a sub-Gaussian deviation bound~\cite{lugosi2019mean}.

\subsection{Polyak-Ruppert averaging}\label{sec:polyak_ruppert}
In this part, we consider the case where the step-size $\gamma$ is chosen as a constant order value satisfying the convergence criteria required in our previous results. Our goal is to obtain a high-confidence bound for the Polyak-Ruppert average $\frac{1}{n}\sum_{j=n_0+1}^{n_0+n} \theta_j$ computed after a burn-in period of $n_0$ iterations. This raises two challenges, the first of which is that, even for a very long burn-in period $n_0,$ the stationary regime is never reached in theory so that one cannot immediately use the concentration properties of $\pi_\gamma.$ The second challenge comes from the lack of independence of the Markov chain iterates. This prevents the adoption of certain approaches such as the entropy method as done in~\cite{maurer2021concentration} for example. 

Notice that, unless the gradient is linear, there is little hope to estimate $\theta^\star$ using the Polyak-Ruppert average since it is bound to approach $\E_{\theta\sim\pi_\gamma}[\theta] = \overline{\theta}_\gamma$ which may differ from $\theta^\star$ by up to $\sigma\sqrt{\gamma/\mu}$ in the non linear case. Nevertheless, the following initial statement holds without this assumption.
\begin{theorem}
    \label{thm:average_concentration}
    Grant Assumptions~\ref{asm:smooth_strongconvex}\textup,~\ref{asm:gradient}\textup,~\ref{asm:grad_special_concentration}~\ref{asm:grad_special_subgauss} and~\ref{asm:gradient_wasserstein}. Let $f : \Theta^n \to \R$ be a $1$-Lipschitz function in each of its parameters and $\vec{\theta} := (\theta_{{0}}, \dots, \theta_{{n-1}})$ be a sequence of SGD iterates with step-size $\gamma < \frac{2\mu}{\mu^2 + (\mu L \vee L^2_{\cW})}$ started from stationarity i.e. such that $\theta_{{0}}\sim\pi_\gamma$. 
    Then we have
    \begin{equation*}
        f(\vec{\theta}\:) - \E f(\vec{\theta})\in \Psi_2\big(KC_{\cW}\sqrt{\gamma/\mu + (n-1) \gamma^2} \big),
    \end{equation*}
    where $C_{\cW} = \big(1 - \sqrt{(1-\gamma\mu)^2 + \gamma^2L^2_{\cW}}\big)^{-1}$.
    If Assumption~\ref{asm:grad_special_concentration}~\ref{asm:grad_special_subgauss} is replaced by Assumption~\ref{asm:grad_special_concentration}~\ref{asm:grad_special_subexp} then 
    \begin{equation*}
     f(\vec{\theta}\:) - \E f(\vec{\theta}) \in \Psi_1\big(KC_{\cW}\sqrt{\gamma/\mu + (n - 1) \gamma^2} \big).
    \end{equation*}
\end{theorem}
The proof of Theorem~\ref{thm:average_concentration} is given in Section~\ref{sec:proof_thm_average_concentration} and employs a hybrid martingale transportation method (see~\cite{boucheron2013concentration, mcdiarmid1998concentration,dubhashi2009concentration, chung2006concentration} for a reference) leveraging the $\cW_2$ convergence established in Proposition~\ref{prop:wasserstein_convergence} in combination with~\cite[Theorem 4.3]{kontorovich2017concentration}.

Theorem~\ref{thm:average_concentration} may be used in a variety of ways by plugging different choices of the function $f.$ For instance, one may choose $f(\vec{\theta}) = \sum_i g(\theta_i)$ for any $g\in\Lip(\R^d).$ Instead, in what follows, we set $\vec{\theta} = (\theta_{n_0+1}, \dots, \theta_{n_0+n})$ and focus on the choice 
\begin{equation*}
    f(\vec{\theta}) = \Big\|\sum_{i=1}^{n} \theta_{n_0+i} - n\theta^\star\Big\|.
\end{equation*}
Before we proceed, we formalize the gradient linearity assumption.
\begin{assumption}\label{asm:linear_grad}
    The gradient $\nabla \cL$ is linear i.e. for all $\theta\in \Theta$ it is equal to $\nabla \cL(\theta) = \Sigma (\theta - \theta^\star)$ for some symmetric positive definite matrix $\Sigma \in \R^{d\times d}.$
\end{assumption}
Note that the positive definiteness of $\Sigma$ in Assumption~\ref{asm:linear_grad} is a consequence of strong convexity while its symmetry is a result of the Hessian $\nabla^2\cL$ being constant in this case and therefore continuous. 
We are now ready to state our non-asymptotic deviation bound for Polyak-Ruppert averaging.
\begin{proposition}
    \label{prop:average_concentration}
    Grant Assumptions~\ref{asm:smooth_strongconvex}\textup,~\ref{asm:gradient}\textup,~\ref{asm:grad_special_concentration}~\ref{asm:grad_special_subgauss}\textup,~\ref{asm:gradient_wasserstein} and~\ref{asm:linear_grad}. Let $(\theta_t)_{t\geq0}$ be the Markov sequence obtained by running SGD with step-size 
    \begin{equation*}
        \gamma < \frac{2\mu}{\mu^2 + (\mu L \vee L^2_{\cW})} \wedge \frac{\mu}{\mu^2 + L^2_{\sigma}}   
    \end{equation*}
    and initial distribution $\theta_0 \sim \nu.$ Then there exist $\rho <1$ and $M< \infty$ such that
    \begin{equation}
    \label{eq:polyak_ruppert1}
    \begin{split}
        \Big\|\frac{1}{n} \sum_{t=1}^{n} \theta_{n_0+t} - \theta^\star\Big\| 
        & \leq \sqrt{\frac{2}{n}\frac{1+\alpha}{1-\alpha}\Big( \alpha_{\cW}^{n_0} \cW_2^2(\nu, \pi_\gamma) + \frac{\gamma \sigma^2}{\mu}\Big)} \\
        &\quad + \frac{2K\sqrt{\gamma/\mu}}{1-\alpha_{\cW}}\sqrt{\gamma\mu+\frac{1}{n}}\sqrt{\frac{\log(1/\delta)}{n}}        
    \end{split}
    \end{equation}
    for $\delta >0$ and $n , n_0 > 0$ with probability at least $1-\Upsilon(\nu, n_0)\delta,$ where 
    \begin{align*}
        \alpha = 1-\gamma\mu,\quad \alpha_{\cW} = \sqrt{\alpha^2 + \gamma^2L^2_{\cW}}\quad \text{and} \quad \Upsilon(\nu, n_0) = 1 + M\rho^{n_0}\Big\| \frac{d\nu}{d\pi_\gamma} \Big\|_{\infty}.
    \end{align*}
    If Assumption~\ref{asm:grad_special_concentration}~\ref{asm:grad_special_subgauss} is replaced by Assumption~\ref{asm:grad_special_concentration}~\ref{asm:grad_special_subexp} then
    \begin{equation}
    \label{eq:polyak_ruppert2}
        \begin{split}
        \Big\|\frac{1}{n}\sum_{t=1}^{n} \theta_{n_0+t} - \theta^\star\Big\| &\leq \sqrt{\frac{2}{n}\frac{1+\alpha}{1-\alpha}\Big( \alpha_{\cW}^{n_0} \cW_2^2(\nu, \pi_\gamma) + \frac{\gamma \sigma^2}{\mu}\Big)} \\
        &\quad + \frac{2K\sqrt{\gamma/\mu}}{1-\alpha_{\cW}} \bigg(\sqrt{\gamma\mu+\frac{1}{n}}\sqrt{\frac{\log(1/\delta)}{n}} \vee  \frac{\log(1/\delta)}{n} \bigg)
        \end{split}
    \end{equation}
    holds with the same probability.
\end{proposition}
The proof of Proposition~\ref{prop:average_concentration} is given in Section~\ref{sec:proof_prop_average_concentration} and takes advantage of the convergence both in total-variation distance and in the $\cW_2$ metric. Note that the given bounds are also dimension-free thanks to Assumption~\ref{asm:grad_special_concentration}. It is possible to derive a weaker result using only Assumption~\ref{asm:grad_concentration} but we omit it to avoid repetition. The variance terms in the upperbounds of~\eqref{eq:polyak_ruppert1} and~\eqref{eq:polyak_ruppert2} (those independent of $\delta$) are controlled thanks to a geometric decorrelation phenomenon which can be shown for the Markov chain iterates under Assumption~\ref{asm:linear_grad} (see Lemma~\ref{lem:geometric_covariances} in the Appendix). This phenomenon becomes weaker for smaller step-size $\gamma,$ therefore, it only makes sense to apply Proposition~\ref{prop:average_concentration} with $\gamma$ of constant order to avoid excessive correlation between the averaged samples. For such $\gamma$ and granted the assumptions of Proposition~\ref{prop:wasserstein_to_tv}, one can also control $\rho$ and show that $\Upsilon(\nu, n_0)$ reaches constant order after a logarithmic number of burn-in steps $n_0.$ Finally, the lack of stationarity of the involved Markov samples is tackled by taking advantage of a spectral gap property satisfied by the transition kernel $P_\gamma$ under the conditions of Theorem~\ref{thm:ergodicity} (see~\cite{kontoyiannis2012geometric}).

Proposition~\ref{prop:average_concentration} may be compared to the works of~\cite{mou2020linear} and~\cite{lou2022beyond}. The former derives a similar high probability bound for linear stochastic approximation under a generalized sub-Gaussianity assumption and uncorrelated noise. The latter considers a weaker finite $L_p$ moment assumption on the SGD data and uses mini-batching to obtain Nagaev type concentration bounds with provably optimal dependence in the confidence level. However, the results of~\cite{mou2020linear,lou2022beyond} both lack the dimension-free property of Proposition~\ref{prop:average_concentration}. 

\section{Applications}\label{sec:applis}
We discuss the consequences of our results for two common use-cases of SGD.
\subsection{Linear regression}\label{sec:lin_reg}

Linear regression is one of the most popular and most used standard models. The aim is to predict a real variable $Y$ based on a random vector $X \in \R^d$ according to the linear model
\begin{equation*}
    Y = X^\top \theta^\star + \epsilon
\end{equation*}
where $\theta^\star$ is an unknown parameter and $\epsilon$ a centered noise. The estimation of $\theta^\star$ may be carried out by minimizing the least-squares objective $\cL(\theta) := \frac{1}{2}\E \big( X^\top \theta - Y\big)^2$ with respect to $\theta\in \R^d.$ This may be done by running SGD with the random gradient $G(\theta, (X, Y)) = X( X^\top \theta - Y).$

Provided the previous gradient admits a finite second moment, Theorem~\ref{thm:ergodicity} and Proposition~\ref{prop:wasserstein_convergence} apply and guarantee the convergence of the SGD Markov chain in total-variation and $\cW_2$ distance. If the covariates $X$ and the noise $\epsilon$ are both Gaussian then the gradient $G(\theta, (X, Y))$ is sub-exponential. However, note that Assumption~\ref{asm:grad_concentration}~\ref{asm:grad_subexp} or~\ref{asm:grad_special_concentration}~\ref{asm:grad_special_subexp} are not immediately satisfied since the associated $\Psi_1$ constant may be unbounded for arbitrarily high values of $\|\theta - \theta^\star\|.$ This problem can be remedied thanks to the following lemma.
\begin{lemma}\label{lem:dimfree_boundedness}
    Let Assumption~\ref{asm:smooth_strongconvex} hold and assume that gradient errors write $\varepsilon_{\zeta_t}(\theta_t) = \Xi_t (\theta_t - \theta^\star) + \xi_t$ where the pairs $(\Xi_t, \xi_t)_{t\geq 0}$ are i.i.d in $\R^{d\times d} \times \R^d$ with $\Xi_t$ symmetric and such that for all $u\in \R^d, \|u\|=1$ we have $\langle u, \Xi_t u \rangle \in \Psi_1(K_\Xi)$ and $\langle u, \xi_t\rangle \in \Psi_1(K_{\xi})$ for $K_\Xi, K_\xi > 0.$ Assume the following minibatch SGD iteration is run starting from $\theta_{0}$ such that $\|\theta_{0} - \theta^\star\| \leq C $ for some $C >0$ for a finite horizon $T$ 
    \begin{equation*}
        \theta_{t+1} = \theta_t - \gamma \overline{G}_N(\theta_t)\quad  \text{with}\quad \overline{G}_N(\theta_t) = \frac{1}{N}\sum_{i=1}^N G(\theta_t, \zeta_{tN + i})
    \end{equation*} with $N$ the minibatch size. For a confidence level $\delta > 0,$ assume that $N$ and $\gamma$ satisfy
    \begin{align*}
        \frac{N}{\log(4T/\delta) + 3d} \geq 1 &\vee \Big(\frac{6}{\mu}\big(3K_\Xi \vee 4K_\xi/C\big)\Big)^2  \\ &\text{and} \quad \gamma \leq \frac{\mu N}{54 K_\Xi^2 (\log(4T/\delta) + 3d)} \wedge \frac{2}{\mu + L}.
    \end{align*}
    Then, with probability at least $1-\delta,$ we have $\max_{0 \leq s\leq T} \|\theta_s - \theta^\star\| \leq C$.
\end{lemma}
Lemma~\ref{lem:dimfree_boundedness} is proven in Section~\ref{sec:prf_dimfree_boundedness} and guarantees that, using a small step-size and minibatching to reduce the gradient variance, with high probability, the iteration does not stray from the vicinity of the optimum during a finite horizon. This shows that the uniform aspect of Assumptions~\ref{asm:grad_concentration} and~\ref{asm:grad_special_concentration} does not prevent the application of the results given in the previous sections. Note that although Lemma~\ref{lem:dimfree_boundedness} requires that $N = \Omega(d),$ the constant $C$ is arbitrary and may be taken dimension-free, for instance, by starting the iteration from a preliminary estimator $\theta_0 = \widehat{\theta}.$

For the example of linear regression with sub-Gaussian samples $(X_t, Y_t)_t,$ Lemma~\ref{lem:dimfree_boundedness} applies with $\Xi_t = X_tX_t^\top - \E[ X_tX_t^\top]$ and $\xi_t = -\epsilon_t X_t.$ Thus, for finite horizon, one may consider the event where the bound of Lemma~\ref{lem:dimfree_boundedness} holds to apply results from Sections~\ref{sec:invariant_properties} and~\ref{sec:deviation_bounds}.

Alternatively, it is also possible to restrict the optimization to a convex and bounded subset $\Theta \subset \R^d$ such that $\theta^\star \in \Theta.$ By letting $\Pi_{\Theta}(\cdot)$ be the projection onto $\Theta$ and replacing iteration~\eqref{eq:sgd_iteration} with
\begin{equation}\label{eq:sgd_projected_iteration}
    \theta_{t+1} = \Pi_{\Theta}\big(\theta_t -\gamma G(\theta_t, \zeta_t)\big),
\end{equation}
we obtain a Markov chain to which Proposition~\ref{prop:invariant_concentration}~\ref{prop:invariant_properties_d} applies and leads to the deviation bound~\eqref{eq:special_concentration_subexp}. Indeed, it is easy to verify that these results still hold for iteration~\eqref{eq:sgd_projected_iteration} thanks to the inequality
\begin{equation*}
    \big\|\Pi_{\Theta}\big(\theta -\gamma G(\theta, \zeta)\big) - \theta^\star \big\| \leq \big\|\theta -\gamma G(\theta, \zeta) - \theta^\star \big\|,
\end{equation*}
valid for all $\theta\in \R^d$ since $\theta^\star \in \Theta$ which is convex. However, by considering the projected iteration~\eqref{eq:sgd_projected_iteration}, Proposition~\ref{prop:invariant_properties}~\ref{prop:invariant_properties_a} may no longer hold so that $\overline{\theta}_\gamma \neq \theta^\star$ making Proposition~\ref{prop:average_concentration} no longer applicable.

\subsection{Logistic regression}
Logistic regression corresponds to the model
\begin{equation*}
    1 - \Proba(Y = -1 \vert X) = \Proba(Y = +1 \vert X) = \sigma(X^\top \theta^\star),
\end{equation*}
where $\sigma$ is the sigmoid function $\sigma(x) = 1/(1+e^{-x}).$ For a parameter $\theta$ and a sample $X,$ the predicted probability is $\Proba(Y = +1 \vert X) = \sigma(X^\top \theta)$ and the model is trained using the log-loss $\ell(\theta, (X, Y)) = -\log(\sigma(Y X^\top \theta))$ which yields the objective $\cL(\theta) = \E_{(X, Y)} \ell(\theta, (X, Y)).$

In order to ensure the objective is strongly-convex, it is necessary to restrict the parameter $\theta$ to a bounded convex set $\Theta.$ This is commonly done by setting $\Theta = \{\theta \in \R^d \:,\: \|\theta\| \leq R\}$ for some radius $R > 0$~\cite{hazan2014logistic, bach2014adaptivity, mourtada2022improper}.

In this case, the projected iteration~\eqref{eq:sgd_projected_iteration} may be used. In this setting, one may easily check that the gradient is sub-Gaussian/sub-exponential as soon as the covariates $X$ satisfy one or the other of these properties. Therefore, the results of Propositions~\ref{prop:invariant_concentration} and~\ref{prop:invariant_special_concentration} apply in this context as well.

\section{Conclusion and Discussion}\label{sec:concl}

The Markov chain point of view for SGD is very useful since it allows to draw conclusions and establish a number of characterizations for the invariant limit distribution. Convergence of the SGD Markov chain holds under fairly weak conditions~\cite{meyn2012markov}. As evidenced by our results, this opens doors for a better characterization of the limit distribution when the associated optimization iteration progresses at \emph{geometric} speed, for instance, when strong convexity holds. 
The precise determination of the speed of convergence in distribution constitutes a particular difficulty which more generally concerns Markov chains with a geometric drift property. However, this difficulty may be circumvented for SGD by leveraging Wasserstein convergence provided a regularity condition on the noise distribution and transition kernel. Obtaining such properties from generic assumptions on the gradient distribution represents an interesting perspective. 

Finally, despite being quite productive, the Markov chain study of SGD remains limited to the constant step-size setting. This excludes the combination of a decreasing step-size with averaging which is known for its better dependence on problem conditioning~\cite{bach2014adaptivity, bach2013non}.%

\acks{This research is supported by the Agence Nationale de la Recherche as part of the ``Investissements d'avenir'' program (reference ANR-19-P3IA-0001; PRAIRIE 3IA Institute).
}

\bibliographystyle{plain}
\bibliography{ref}

\appendix
\newpage

\section{Preliminary lemmas}

\begin{lemma}\label{lem:subgauss}
    Let $X$ be a real random variable such that $X\in \widetilde{\Psi}_2(K)$ then, for $\delta > 0,$ with probability at least $1-\delta,$ we have
    \begin{equation*}
        |X| \leq K\sqrt{\log(e/\delta)}.
    \end{equation*}
\end{lemma}
\begin{proof}
    Using Chernoff's method, we find for $t > 0$ and $\lambda > 0$
    \begin{align*}
        \Proba\big(|X| > t\big) &= \Proba\big(\lambda^2 X^2 > \lambda^2 t^2\big) = \Proba\big(\exp(\lambda^2 X^2) > \exp(\lambda^2 t^2)\big) \\
        &\leq \E \exp(\lambda^2 X^2) e^{-\lambda^2 t^2} \leq \exp\big(\lambda^2(K^2 - t^2)\big).
    \end{align*}
    Choosing $\lambda \!=\! 1/K,$ we have $\exp\big(1 \!-\! (t/K)^2\big) \!\leq\! \delta \!\!\iff\!\! t\!\geq\! K\sqrt{\log(e/\delta)}$ and the result follows.
\end{proof}

\begin{lemma}\label{lem:subexp}
    Let $X$ be a real random variable such that $X\in \widetilde{\Psi}_1(K)$ then, for $\delta > 0,$ with probability at least $1-\delta,$ we have
    \begin{equation*}
        |X| \leq 2eK\log(2/\delta).
    \end{equation*}
\end{lemma}
\begin{proof}
    Using Stirling's approximation, we find for $|\lambda| < (eK)^{-1}:$
    \begin{align*}
        \E \exp(\lambda |X|) &= \sum_{p\geq 0}\frac{\lambda^p \E |X|^p}{p!} \leq 1 + \sum_{p\geq 1} \frac{(\lambda K p)^p}{p!}\\
        &\leq 1 + \sum_{p\geq 1}\frac{(\lambda e K)^p}{\sqrt{2\pi p}} \leq 1 + \frac{1}{\sqrt{2\pi}}\frac{\lambda e K}{1 - \lambda e K} \leq \exp\Big( \frac{1}{\sqrt{2\pi}}\frac{\lambda e K}{1 - \lambda e K} \Big),
    \end{align*}
    where we used the inequality $1+x \leq e^x$ in the last step. For $t > 0,$ using Chernoff's method and choosing $\lambda = (2eK)^{-1}$, we find\textup:
    \begin{align*}
        \Proba\big(|X| > t\big) &= \Proba\big(\lambda |X| > \lambda t\big) = \Proba\big(\exp(\lambda |X|) > \exp(\lambda t)\big) \\
        &\leq \E \exp(\lambda|X|) e^{-\lambda t} \leq \exp\Big(\frac{1}{\sqrt{2\pi}} - \frac{t}{2eK}\Big).
    \end{align*}
    It only remains to choose $t = 2eK \log(2/\delta)$ to obtain the desired bound.
\end{proof}

The following fundamental lemma will be often used in our proofs.
\begin{lemma}\label{lem:contraction}
    Grant Assumption~\ref{asm:smooth_strongconvex}. For any $\theta, \theta' \in \R^d$ and $\gamma \leq \frac{2}{\mu+L}$ we have
    \begin{equation}
        \|\theta - \gamma \nabla\cL(\theta) - (\theta' - \gamma \nabla\cL(\theta'))\|^2 \leq (1 - \gamma\mu)^2 \|\theta - \theta'\|^2.
    \end{equation}        
\end{lemma}
\begin{proof}
    For $\gamma \leq \frac{2}{\mu+L},$ we have
    \begin{align*}
        \|\theta &- \gamma \nabla\cL(\theta) - (\theta' - \gamma \nabla\cL(\theta'))\|^2 \\
        &= \|\theta - \theta'\|^2 - 2\gamma \langle \theta - \theta', \nabla\cL(\theta) -\nabla\cL(\theta') \rangle + \gamma^2\|\nabla\cL(\theta) - \nabla\cL(\theta')\|^2\\
        &\leq(1 - \gamma^2\mu L)\|\theta - \theta'\|^2 - \gamma(2 - \gamma(\mu+L))\langle \nabla\cL(\theta) - \nabla\cL(\theta'), \theta - \theta' \rangle\\
        &\leq (1 - \gamma^2\mu L)\|\theta - \theta'\|^2 - \gamma(2 - \gamma(\mu+L) )\mu \|\theta - \theta'\|^2\\
        &= (1 - \gamma^2\mu L - 2\gamma \mu + \gamma^2\mu(\mu+L))\|\theta - \theta'\|^2\\
        &= (1 - \gamma \mu)^2\|\theta - \theta'\|^2,
    \end{align*}
    where we used the inequalities
    \begin{align}
        \|\nabla\cL(\theta) - \nabla\cL(\theta')\|^2 &\leq (\mu+L)\langle \nabla\cL(\theta) - \nabla\cL(\theta'), \theta - \theta' \rangle  - \mu L\|\theta - \theta'\|^2 \label{eq:gradient_coercivity}\\
        \mu \|\theta - \theta'\|^2 &\leq \langle \nabla\cL(\theta) - \nabla\cL(\theta'), \theta - \theta' \rangle \label{eq:strong_convexity},
    \end{align}
    valid for all $\theta, \theta'$. Equation~\eqref{eq:gradient_coercivity} is stated, for example, in~\cite[Theoerem 2.1.12]{Nesterov} (see also \cite[Lemma 3.11]{bubeck2015convex} and~\eqref{eq:strong_convexity} is just a characterization of strong convexity (see for instance~\cite[Theorem 2.1.9]{Nesterov}).
\end{proof}

\section{Proof of Geometric Ergodicity}\label{sec:proof_geometric_ergodicity}

In the remaining part of this document, we make the dependencies on $\zeta$ in the gradient samples and errors implicit and write $G(\theta)$ and $\varepsilon(\theta)$ instead of $G(\theta, \zeta)$ and $\varepsilon_{\zeta}(\theta)$ respectively.

We show the geometric ergodicity of the SGD Markov chain $(\theta_t)_{t\geq 0}$ by relying on~\cite[Theorem 15.0.1]{meyn2012markov}. We will show that the following function:
\begin{equation*}
    V(\theta) := 1 + \|\theta - \theta^\star\|^2,
\end{equation*}
is a \emph{drift} function for this Markov chain. We define the action of the transition kernel $P$ on integrable functions $f$ through
\begin{align*}
    P_{\gamma}f(\theta) = \E f(\theta - \gamma G(\theta)).
\end{align*}
We also define the variation operator
\begin{equation*}
    \Delta f(\theta) := P_{\gamma} f(\theta) - f(\theta).
\end{equation*}

\subsection{Proof of Theorem~\ref{thm:ergodicity}}\label{sec:proof_thm_ergodicity}
First, we establish that the Markov chain is aperiodic. Indeed, by Assumption~\ref{asm:gradient}~\ref{asm:gradient_dens_minor}, for all $\theta,$ the gradient is distributed according to an everywhere positive density, therefore, for all $\theta \in S \subset \R^d$ with $S$ a set with non zero Lebesgue measure we have $P_{\gamma}(\theta, S) > 0.$ This implies that the greatest possible period for the chain is $1$ which makes it aperiodic.

We also show that the Markov chain is $\psi$-irreducible (see~\cite[Chapter 4]{meyn2012markov}). For any initial $\theta_0,$ its successor reads:
\begin{equation*}
    \theta_{1} = \theta_0 - \gamma (\nabla \cL(\theta_0) + \varepsilon(\theta_0))
\end{equation*}
Given Assumption~\ref{asm:gradient}~\ref{asm:gradient_dens_minor}, the distribution of $\varepsilon(\theta_0)$ is minorized by $\delta \nu_{\theta_0, 1}$ where $\nu_{\theta_0, 1}$ is a probability distribution which admits an everywhere positive density $h(\theta_0, \cdot).$ Consequently, for all $A\in \mathcal{B}(\R^d)$ with non zero Lebesgue measure, we have the following minorization:
\begin{align*}
    \Proba(\theta_{1} \in A \vert \theta_0) = P_{\gamma}(\theta_0, A) &\geq \delta \int_{\R^d} h(\theta_0, \omega) \ind{\theta_0 - \gamma(\nabla\cL(\theta_0) + \omega) \in A}d\omega \\    
    &= \frac{\delta}{\gamma^d} \int_A h\Big(\theta_0, \frac{\theta - \theta_0}{\gamma} - \nabla\cL(\theta_0)\Big)d\theta > 0,
\end{align*}
where we applied the change of variables $\omega \mapsto \theta = \theta_0 - \gamma(\nabla\cL(\theta_0) + \omega)$ whose Jacobian is $-\gamma I_d$ with $I_d$ the $d$-dimensional identity matrix. It follows that the Markov chain is irreducible w.r.t. Lebesgue's measure and is thus $\psi$-irreducible.

For fixed $\theta,$ and step-size $\gamma <\frac{2}{\mu+L},$ using Lemma~\ref{lem:contraction} we find:
\begin{align}
    P_{\gamma}\|\theta - \theta^\star\|^2 &= \E \|\theta - \gamma G(\theta) - \theta^\star\|^2 \nonumber\\
    &= \E \big[\|\theta - \gamma \nabla \cL(\theta) - \theta^\star\|^2 - 2\gamma\langle \theta - \gamma \nabla \cL(\theta) - \theta^\star, \varepsilon(\theta) \rangle + \gamma^2 \|\varepsilon(\theta)\|^2 \big] \nonumber\\
    &\leq (1-\gamma\mu)^2 \|\theta - \theta^\star\|^2 + \gamma^2 \E \|\varepsilon(\theta)\|^2 \nonumber\\
    &\leq (1-\gamma\mu)^2 \|\theta - \theta^\star\|^2 + \gamma^2 \big(L^2_{\sigma}\|\theta - \theta^\star\|^2 + \sigma^2\big) \nonumber \\
    &= \big((1-\gamma\mu)^2 + \gamma^2L^2_{\sigma}\big) \|\theta - \theta^\star\|^2 + \gamma^2 \sigma^2 \label{eq:contract}
\end{align}
The previous inequality yields a contraction for step-size satisfying $0 < \gamma < \frac{2\mu}{\mu^2 + L^2_{\sigma}}$ and, as a consequence, we have:
\begin{equation*}
    P_{\gamma}V(\theta) \leq \underbrace{\big((1-\gamma\mu)^2 + \gamma^2L^2_{\sigma}\big)}_{=:\widetilde{\lambda}} V(\theta) + \underbrace{\gamma^2\sigma^2 + \big(1-\big((1-\gamma\mu)^2 + \gamma^2L^2_{\sigma}\big)\big)}_{=:\widetilde{b}}
\end{equation*}
We now define the set $\mathcal{C} = \big\{\theta \in \R^d \:, \: V(\theta) \leq 2\widetilde{b}/(1 - \widetilde{\lambda})\big\}$ which satisfies:
\begin{equation}
    \Delta V(\theta) \leq -\frac{1 - \widetilde{\lambda}}{2}V(\theta) + \widetilde{b}\ind{\theta \in \mathcal{C}}.\label{eq:lyapunov_cond}
\end{equation}
For such $\mathcal{C},$ let $\underline{h}(\theta) = \inf_{\theta_0 \in \mathcal{C}}h\Big(\theta_0, \frac{\theta - \theta_0}{\gamma} - \nabla\cL(\theta_0)\Big)$ and define the probability measure $\nu_{\mathcal{C}}$ by
\begin{equation*}
    \nu_{\mathcal{C}}(A) = \frac{\int_{A \cap \mathcal{C}} \underline{h}(\theta) d\theta }{\int_{\mathcal{C}} \underline{h}(\theta')d\theta'} \quad \text{ for all }\quad A\in \mathcal{B}(\R^d).
\end{equation*}
It follows that for all $\theta_0 \in \mathcal{C},$ we have the following minorization property:
\begin{equation*}
    P_{\gamma}(\theta_0, A) \geq \xi \nu_{\mathcal{C}}(A)  \quad \text{ for all }\quad A\in \mathcal{B}(\R^d),
\end{equation*}
where $\xi = \delta \int_{\mathcal{C}} \underline{h}(\theta)d\theta > 0.$ In words, the set $\mathcal{C}$ is a \emph{small} set and, thanks to~\cite[Proposition 5.5.3]{meyn2012markov}, also a \emph{petite} set (see definitions in~\cite[Chapter 5]{meyn2012markov}).

We now define the hitting time $\tau_{\mathcal{C}} = \inf\{n > 0 : \theta_n \in \mathcal{C}\}.$ Thanks to the drift property~\eqref{eq:lyapunov_cond}, we can apply~\cite[Corollary A.4]{mattingly2002ergodicity} which implies that, for any $\theta_0 \in \R^d,$ we have 
\begin{equation*}
    \Proba(\tau_{\mathcal{C}} < \infty) = 1,
\end{equation*} meaning that $\mathcal{C}$ is Harris recurrent (see~\cite[Chapter 9]{meyn2012markov}). Moreover, since $\mathcal{C}$ is a petite set, using~\cite[Proposition 10.2.4]{douc2018markov}, we get that the Markov chain $(\theta_t)_t$ itself is Harris recurrent.

In addition, notice that since $V(\theta) \geq 1,$ we can scale~\eqref{eq:lyapunov_cond} by a factor $2/(1-\widetilde{\lambda})$ to obtain the following drift property
\begin{equation*}
    \Delta \widehat{V}(\theta) \leq -1 + \widehat{b}\ind{\theta \in \mathcal{C}},
\end{equation*}
where $\widehat{V}$ and $\widehat{b}$ are the scaled versions of $V$ and $\widetilde{b}$ respectively. Thus, the Markov chain $(\theta_t)_t$ verifies condition (iv) of~\cite[Theorem 13.0.1]{meyn2012markov}. Consequently, it admits a unique and finite invariant measure $\pi_{\gamma}.$

Inequality~\eqref{eq:lyapunov_cond} and the properties of the set $\mathcal{C}$ show that the Markov chain $(\theta_t)$ fulfills condition (iii) of~\cite[Theorem 15.0.1]{meyn2012markov}. By the latter result, it follows that there exist $r > 1$ and $M < \infty$ such that:
\begin{equation}\label{eq:mt_bound}
    \sum_{t\geq 0} r^t \|P^t_{\gamma}(\theta_0, \cdot) - \pi_{\gamma}\|_{\tv} \leq M V(\theta_0).
\end{equation}
In particular, taking $\rho = r^{-1},$ we find for all $n\geq 0$:
\begin{equation}\label{eq:mt_bound1}
    \rho^{-n} \|P^n_{\gamma}(\theta_0, \cdot) - \pi_{\gamma}\|_{\tv} \leq \sum_{t\geq 0} r^t \|P^t_{\gamma}(\theta_0, \cdot) - \pi_{\gamma}\|_{\tv} \leq M V(\theta_0),
\end{equation}
which concludes the proof.

\subsubsection{Aligned step-size scaling with related works}\label{sec:stepsize_compare}

In this section, we showcase a setting where our a priori restrictive step-size condition~\eqref{eq:thm1_stepsize} scaling in $O(\mu/L_{\sigma}^2)$ is on par with related works on stochastic optimization.

\paragraph{Setting}
We consider linear regression similarly to our discussion following Proposition~\ref{prop:wasserstein_convergence} with random covariates $X \in \R^2$ such that $\E X = 0$ and $X_1, X_2$ are independent.

The main purpose is to expose a distribution such that the least eigenvector of the covariance $\Sigma = \E XX^\top$ is aligned with the top eigenvector of the noise covariance $\E(\Sigma - XX^\top)^2.$ This can be achieved using a scalar distribution with a wide gap between its second and fourth moments.

\subparagraph{\underline{Covariate distribution and moments}}
We let $M > 1, \epsilon \in (0, 1)$ and define $X_1$ as a uniform variable with a random offset as follows
\begin{equation*}
    X_1 = U + B \quad \text{with} \quad U\sim \mathcal{U}_{[-1/M, 1/M]} \quad \text{and}\quad B = \begin{cases}
        0 & \text{w.p.}\quad 1-\epsilon \\
        +M & \text{w.p.}\quad \epsilon/2 \\
        -M & \text{w.p.}\quad \epsilon/2.
    \end{cases}
\end{equation*}
We let $X_2$ be uniform over $[-1, 1]$ i.e.~$X_2 \sim \mathcal{U}_{[-1, 1]}.$ The random variables $U, B$ and $X_2$ are mutually independent implying the same for $X_1$ and $X_2.$

For $i\in \{1, 2\}$ and $j \geq 1,$ we denote the signed moments $m_{i,j}^j := \E X_i^j.$ Since both distributions are symmetric, we have $m_{1,1} = m_{2,1} = m_{1,3} = m_{2,3} = 0.$ Moreover, simple computations yield $m_{2,2}^2 = 1/3$ and $m_{2,4}^4 = 1/5.$ Finally, for $m_{1,2}^2$ and $m_{1,4}^4,$ we have

\begin{align*}
    m_{1,2}^2 &= (1-\epsilon)\int_{-1/M}^{1/M}\frac{x^2}{2/M}dx + \frac{\epsilon}{2}\bigg(\int_{-1/M}^{1/M}\frac{(x+M)^2}{2/M}dx + \int_{-1/M}^{1/M}\frac{(x-M)^2}{2/M}dx\bigg) \\
    &= \frac{(1 - \epsilon)}{3M^2} + \frac{M\epsilon}{2}\int_{-1/M}^{1/M}(M+x)^2dx \\
    &= \frac{(1 - \epsilon)}{3M^2} + \frac{M\epsilon}{6}\big((M+1/M)^3 - (M-1/M)^3\big) \\
    &= \frac{(1 - \epsilon)}{3M^2} + \epsilon\Big(M^2 + \frac{1}{3M^2}\Big) = \frac{1}{3M^2} + \epsilon M^2,
\end{align*}
as well as 
\begin{align*}
    m_{1,4}^4 &= (1-\epsilon)\int_{-1/M}^{1/M}\frac{x^4}{2/M}dx + \epsilon\int_{-1/M}^{1/M}\frac{(M+x)^4}{2/M}dx \\
    &= \frac{(1 - \epsilon)}{5M^4} + \frac{M\epsilon}{10}\big((M+1/M)^5 - (M-1/M)^5\big) \\
    &= \frac{(1 - \epsilon)}{5M^4} + \frac{M\epsilon}{5}\bigg(\binom{5}{0}M^{-5} + \binom{5}{2}M^{-1} + \binom{5}{4}M^{3}\bigg)\\
    &= \frac{(1 - \epsilon)}{5M^4} + \epsilon\big(M^{-4}/5 + 2 + M^{4}\big)= \frac{1}{5M^4} + \epsilon\big(M^{4} + 2\big),
\end{align*}
where the third equality uses that the odd terms in the expansions of the two fifth powers cancel out while the even ones are duplicated.

\subparagraph{\underline{Covariance}}
Computing the covariance matrix leads to 
\begin{equation*}
    \Sigma = \E XX^\top = \begin{pmatrix}
    \E X_1^2 & \E X_1 X_2 \\ \E X_1X_2 & \E X_2^2
\end{pmatrix} = \begin{pmatrix}
    m_{1,2}^2 & 0 \\ 0 & m_{2,2}^2
\end{pmatrix}.
\end{equation*}
Note that, in this case, we have $\mu = \min(m_{1,2}^2, m_{2,2}^2)$ and $L = \max(m_{1,2}^2, m_{2,2}^2).$
\subparagraph{\underline{Noise covariance}}
We write $\E (\Sigma - XX^\top)^2 = \begin{pmatrix}
    a_{1,1} & a_{1,2} \\ a_{2,1} & a_{2,2}
\end{pmatrix}$ and compute the coefficients $a_{i,j}$ for $1\leq i,j\leq 2.$ We find 
\begin{align*}
    a_{1, 2} = a_{2, 1} &= \E\big[(\E X_1 X_2 - X_1 X_2)(\E X_1^2 - X_1^2 + \E X_2^2 - X_2^2)\big] \\
    &= - \E\big[(\E X_1 X_2 - X_1 X_2)(X_1^2 + X_2^2)\big] \\
    &= \E\big[(X_1 X_2)(X_1^2 + X_2^2)\big] = m_{1,3}^3 m_{2,1} + m_{2,3}^3 m_{1,1} = 0,
\end{align*}
where we used that $\E X_1 X_2 - X_1 X_2$ has zero expectation and $\E X_1^2, \E X_2^2$ are constants then the fact that, by independence, $\E X_1 X_2 = \E X_1 \E X_2 = 0.$

As for the diagonal coefficients, we have
\begin{align*}
    a_{i,i} &= \E \big[(\E X_i^2 - X_i^2)^2 + (\E X_1 X_2 - X_1 X_2)^2\big]\\
    &= m_{i, 4}^4 - m_{i, 2}^4 + m_{1, 2}^2 m_{2, 2}^2.
\end{align*}
\subparagraph{\underline{Eigenvector alignment}}
By setting $\epsilon = 1/M^3$ and choosing $M \geq 4$, we get that $m_{1,2}^2 < m_{2,2}^2$ in the covariance matrix $\Sigma$ since $m_{2,2}^2 = 1/3 > \big(3M^{-1} + M^{-2}\big)/3 = m_{1,2}^2.$ This entails that the first canonical basis vector $e_1$ is the eigenvector of $\Sigma$ with the smallest eigenvalue $\mu = m_{1,2}^2.$

At the same time, the noise covariance matrix being diagonal and our choice of $\epsilon$ and $M$ lead to $L_{\sigma}^2 = \|\E (\Sigma - XX^\top)^2\|_2 = a_{1,1} > a_{2,2}.$ Indeed, we have
\begin{alignat*}{2}
    &&a_{2,2} &< a_{1,1}  \\
    \iff &&\quad m_{2,4}^4 - m_{2,2}^4 &< m_{1,4}^4 - m_{1,2}^4 \\
    \iff &&\quad 1/5 - 1/9 &< \frac{1}{5M^4} + \epsilon\big(M^{4}+2\big) - \Big(\frac{1}{3M^2} + \epsilon M^2\Big)^2  \\
    \iff &&\quad 1/5 - 1/9 &< \epsilon (1-\epsilon) M^4 + \frac{4\epsilon}{3} + M^{-4}\Big(\frac{1}{5} - \frac{1}{9}\Big),
\end{alignat*}
where all terms in the RHS are positive. Plugging $\epsilon=1/M^3$ into the term $\epsilon (1-\epsilon) M^4$ shows that the inequality is satisfied for $M \geq 4.$

As a result, the first canonical basis vector $e_1$ is the eigenvector of the noise covariance matrix $\E (\Sigma - XX^\top)^2$ with the top eigenvalue $L_{\sigma}^2 = a_{1,1}.$

Thus, the least eigenvalue $\mu$ of $\Sigma$ and the top eigenvalue $L_{\sigma}^2$ of $\E (\Sigma - XX^\top)^2$ are associated to the same eigenvector which is $e_1.$ As a result, for $\theta = e_1,$ the following pair of equalities hold at the same time
\begin{equation*}
    \Sigma \theta = \mu \theta \quad \text{and}\quad \E\big\|\big(\Sigma - XX^\top\big)\theta\big\|^2 = L_{\sigma}^2\|\theta\|^2.
\end{equation*}
Note also that for $M\to +\infty,$ we simultaneously have $\mu \to 0$ and $L_{\sigma}^2 \to +\infty.$
\paragraph{Comparison with related works}
We now show that the step-size scaling of Theorem~\ref{thm:ergodicity} and Proposition~\ref{prop:wasserstein_convergence} is equivalent to some related works on stochastic optimization in the above setting where $L^2_\sigma = \big\|\E[ (XX^\top - \Sigma)^2]\big\|_2.$ 

The main arguments for the comparison with~\cite{dieuleveut2020bridging} have already been laid out in the discussion following the statement of Proposition~\ref{prop:wasserstein_convergence}. These arguments are completed by the setting above where the choice of $\theta - \theta'$ simultaneously aligned with the least eigenvector of $\Sigma$ and the top eigenvector of $\E (XX^\top - \Sigma)^2$ is justified.

In~\cite{needell2014stochastic}, the authors consider an objective $F(x) = \E_{i\sim\mathcal{D}}f_i(x)$ which is $\mu$-strongly convex and assume the $f_i$'s are convex and $L_i$-smooth. Note  that, since the $f_i$'s are convex, their $L_i$-smoothness is equivalent to $L_i$-co-coercivity. The step-size condition is then $\gamma \leq 1/\sup L$ with $L_i \leq \sup L$ almost surely. This comparison is therefore similar to the one with~\cite{dieuleveut2020bridging} with $\sup L$ replacing $L'.$

In~\cite{bach2013non}, the matrix $H$ in assumption (A3) corresponds to $H = \E XX^\top = \Sigma.$ Considering the case $\lambda_{\min}(\Sigma) = \mu,$ the condition $\E[\|X\|^2 X \otimes X] \preceq R^2 H$ of (A6) can be rewritten as
\begin{align*}
        \E&\big[(XX^\top)^2\big] \preceq R^2 \Sigma \\
        \iff \E&\big[(XX^\top - \Sigma + \Sigma)^2\big] \preceq R^2 \Sigma \\
        \iff \E&\big[(XX^\top - \Sigma)^2\big] + \Sigma^2 \preceq R^2 \Sigma
\end{align*}
We consider the setting given above where a vector exists which is aligned with the top eigenvector of $\E[(XX^\top - \Sigma)^2],$ whose eigenvalue would be $L_{\sigma}^2,$ and the least eigenvector of $\Sigma,$ whose eigenvalue is $\mu,$ at the same time. This leads to $L_{\sigma}^2 +\mu^2 \:\leq\: R^2 \mu\implies R^2\:\geq\: L_{\sigma}^2/\mu +\mu $ so that the step-size condition $\gamma \leq 1/R^2$ has a similar scale to~\eqref{eq:thm1_stepsize}.

Similarly, in~\cite{pillaud2018exponential}, linear regression is considered in Equation (3) by setting $\eta_n = \theta_n - \theta^\star$ with $H_n = XX^\top$ and $\varepsilon_n = -\xi X.$ According to (H2) we have again $H = \Sigma$ with strong-convexity constant $\lambda = \mu.$ By defining $\E\xi^2 = \sigma^2$ in the linear regression setting, the assumption $\E[\varepsilon_n\otimes \varepsilon_n]\preceq C$ in~\cite{pillaud2018exponential} holds with $C = \sigma^2\Sigma$ for independent noise $\xi.$ Assumption (H4) is then equivalent to $\E[H_n C H^{-1}H_n] = \sigma^2 \E[\|X\|^2 XX^\top] \preceq \gamma_0^{-1}\sigma^2 \Sigma.$ Since the strong-convexity constant is $\lambda = \mu = \lambda_{\min}(\Sigma),$ we can use the same argument as above to conclude that the step-size condition $\gamma\leq \gamma_0$ has the same scale as~\eqref{eq:thm1_stepsize} in the setting laid out earlier in this section.

\subsection{Proof of Proposition~\ref{prop:invariant_properties}}\label{sec:proof_invariant_properties}

To prove~\ref{prop:invariant_properties_a}, let $\theta \sim \pi_\gamma$ and simply compute
\begin{align*}
    \E [\theta] = \E\big[\theta-\gamma G(\theta)\big] = \E\big[\theta -\gamma\nabla\cL(\theta)\big] = \E [\theta]  -\gamma\E[\nabla\cL(\theta)]
\end{align*}
since we know that $\E [\theta] < \infty $ (this follows from~\eqref{eq:mt_bound} in the proof of Theorem~\ref{thm:ergodicity}), this implies the first part of the claim. If we further assume the gradient to be linear, we have in addition
\begin{equation*}
    \E \nabla\cL(\theta) = \nabla\cL\big(\E\theta \big) = \nabla\cL(\bar{\theta}_{\gamma}) = 0,
\end{equation*}
and the conclusion follows since $\theta^\star$ is the unique critical point. We now consider $t\geq 1$ and compute
\begin{align}
    \E\|\theta_t - \theta^\star\|^2 &= \E \|\theta_{t-1} - \gamma G(\theta_{t-1}) - \theta^\star\|^2 \nonumber\\
    &= \E\big[ \| \theta_{t-1} - \gamma \nabla\cL(\theta_{t-1}) - \theta^\star\|^2  + \gamma^2 \| \varepsilon(\theta_{t-1})\|^2\nonumber\\ &\quad- 2\gamma \langle \theta_{t-1} - \gamma \nabla \cL(\theta_{t-1}) - \theta^\star, \varepsilon(\theta_{t-1}) \rangle \big] \nonumber \\ 
    &\leq (1-\gamma\mu)^2\E\|\theta_{t-1} - \theta^\star\|^2 + \gamma^2\E\|\varepsilon(\theta_{t-1})\|^2\nonumber\\
    &\leq \big((1-\gamma\mu)^2 + \gamma^2L_\sigma^2\big)\E\|\theta_{t-1} - \theta^\star\|^2 + \gamma^2\sigma^2 \nonumber\\ 
    &= \alpha_\sigma \E\|\theta_{t-1} - \theta^\star\|^2 + \gamma^2\sigma^2 \label{eq:prf_prop1_checkpt} 
\end{align}
where we used Lemma~\ref{lem:contraction} and Assumption~\ref{asm:gradient}~\ref{asm:gradient_centered} and~\ref{asm:gradient_regular}. 
We then iterate this relationship to find
\begin{align}
    \E\|\theta_t - \theta^\star\|^2 &\leq \alpha_\sigma^t\E\|\theta_{0} - \theta^\star\|^2 + \gamma^2\sigma^2\sum_{i=0}^{t-1} \alpha_\sigma^i\nonumber\\
    &= \alpha_\sigma^t\E\|\theta_{0} - \theta^\star\|^2 + \gamma^2\sigma^2\frac{1 - \alpha_\sigma^{t}}{1 - \alpha_\sigma}\label{eq:prf_prop1_checkpt2},
\end{align}
which proves~\ref{prop:invariant_properties_bprime}. To prove~\ref{prop:invariant_properties_b}, we consider a stationary chain such that $\theta_{t-1}, \theta_{t}\sim\pi_\gamma$ and therefore $\E\|\theta_{t}-\theta^\star\|^2 = \E\|\theta_{t-1}-\theta^\star\|^2.$ Resuming from~\eqref{eq:prf_prop1_checkpt}, we find that for $\theta\sim\pi_\gamma,$ we have
\begin{equation*}
    \Var_{\pi_\gamma}(\theta) \leq \E\|\theta - \theta^\star\|^2  \leq \frac{\gamma^2 \sigma^2}{1 - (1-\gamma \mu)^2 - \gamma^2 L^2_{\sigma}} = \frac{\gamma \sigma^2}{2\mu - \gamma(\mu^2+ L^2_{\sigma})}.
\end{equation*}
Moreover, using Jensen's inequality followed by Cauchy-Schwarz, we have
\begin{equation*}
    \|\bar{\theta}_\gamma - \theta^\star\| = \| \E\theta - \theta^\star \| \leq \E\| \theta - \theta^\star \|\leq \sqrt{\E \| \theta - \theta^\star \|^2},
\end{equation*}
which concludes the proof.

\subsection{Proof of Proposition~\ref{prop:invariant_concentration}}\label{sec:proof_invariant_concentration}

Let us prove~\ref{prop:invariant_properties_c}, consider an iterate $\theta_{k+1}$ from the SGD iteration for some $k\geq 0,$ we have:
\begin{align*}
    \E\exp(\lambda^2\|\theta_{k+1} - \theta^\star\|^2) &= \E\exp(\lambda^2\|\theta_{k} - \gamma G(\theta_{k}) - \theta^\star\|^2) \\
    &= \E\exp\big(\lambda^2(\|\theta_{k} -\gamma\nabla\cL(\theta_{k})- \theta^\star\|^2 \\
    &\quad - 2\gamma\langle \theta_{k} -\gamma\nabla\cL(\theta_{k}) - \theta^\star, \varepsilon(\theta_{k}) \rangle + \gamma^2 \|\varepsilon(\theta_{k})\|^2)\big)
\end{align*}
Since we assume that $\|\varepsilon(\theta)\| \in \widetilde{\Psi}_2(\overline{K})$ for all $\theta,$ it is easy to check that for all $u\in\R^d$ with unit norm, $\langle u, \varepsilon(\theta)\rangle \in \Psi_2(\overline{K}).$ Indeed, recall our observation following Definition~\ref{def:subexp} and note that $\langle u, \varepsilon(\theta)\rangle$ is centered because $\varepsilon(\theta)$ is centered and $|\langle u, \varepsilon(\theta)\rangle| \leq \|\varepsilon(\theta)\| \in \widetilde{\Psi}_2(\overline{K}).$ Therefore, conditioning on $\theta_{k},$ we have
\begin{align*}
    \E \big[&\exp \big(\lambda^2(- 2\gamma\langle \theta_{k} -\gamma\nabla\cL(\theta_{k}) - \theta^\star, \varepsilon(\theta_{k}) \rangle + \gamma^2 \|\varepsilon(\theta_{k})\|^2)\big)\vert \theta_{k} \big]\\
    &\leq \E \big[\exp \big(-(2\lambda)^2\gamma\langle \theta_{k} -\gamma\nabla\cL(\theta_{k}) - \theta^\star, \varepsilon(\theta_{k}) \rangle\big) \vert \theta_{k} \big] ^{1/2} \E \big[\exp \big(2(\gamma\lambda)^2 \|\varepsilon(\theta_{k})\|^2\big)\vert \theta_{k}\big]^{1/2}\\
    &\leq \exp\big( 8\lambda^4\gamma^2 \|\theta -\gamma \nabla\cL(\theta_{k}) - \theta^\star\|^2 \overline{K}^2 + \lambda^2 \gamma^2 \overline{K}^2\big)\\
    &\leq \exp\big( 8\lambda^4\gamma^2 (1-\gamma\mu)^2\|\theta_{k} - \theta^\star\|^2 \overline{K}^2 + \lambda^2 \gamma^2 \overline{K}^2\big),
\end{align*}
where the last line uses Lemma~\ref{lem:contraction}. The previous inequality holds for $|\lambda| \leq \big(\sqrt{2}\gamma \overline{K}\big)^{-1} $. We now restrict $\lambda$ so that $|\lambda| \leq \big(2\overline{K}\sqrt{\gamma/\mu}\big)^{-1}$ which implies 
\begin{equation*}
    1 + 8\lambda^2\gamma^2\overline{K}^2\leq 1 + 2\gamma\mu\leq 1+\frac{\gamma\mu}{1-\gamma\mu} = \frac{1}{1-\gamma\mu}.
\end{equation*}
We thus obtain
\begin{align}
    \E\exp&\big(\lambda^2\|\theta_{k+1} - \theta^\star\|^2) \nonumber\\
    &\leq \E\exp\big(\lambda^2(1-\gamma\mu)^2\|\theta_{k}  - \theta^\star\|^2(1 + 8\lambda^2\gamma^2 \overline{K}^2)\big) \exp\big( \lambda^2 \gamma^2 \overline{K}^2\big) \nonumber\\
    &\leq \E\exp\big(\lambda^2(1-\gamma\mu)\|\theta_{k}  - \theta^\star\|^2\big) \exp\big( \lambda^2 \gamma^2 \overline{K}^2\big) . \label{eq:subgauss_prf_checkpt}
\end{align}
This relationship can be iterated to find that
\begin{align}
    \E\exp&\big(\lambda^2\|\theta_{k+1} - \theta^\star\|^2\big) \nonumber\\
    &\leq \exp\Big(\lambda^2\Big((1-\gamma\mu)^{k+1}\|\theta_0 - \theta^\star\|^2 + \gamma^2\overline{K}^2\sum_{j=0}^{k}(1-\gamma\mu)^j\Big)\Big) \nonumber\\
    &=\exp\Big(\lambda^2\Big((1-\gamma\mu)^{k+1}\|\theta_0 - \theta^\star\|^2 + \big(1-(1-\gamma\mu)^{k+1}\big)\gamma\overline{K}^2/\mu\Big)\Big),\label{eq:subgauss_prf_checkpt2}
\end{align}
which shows that $\|\theta_{k+1} - \theta^\star\|$ is sub-Gaussian with the desired constant. Regarding the invariant distribution at the limit, we use the fact that for $\theta_k \sim \pi_\gamma$ i.e.~at stationarity, $\theta_{k+1} = \theta_k - \gamma G(\theta_k)$ and $\theta_k$ have the same distribution and $\E\exp\big(\lambda^2\|\theta_{k+1} - \theta^\star\|^2) = \E\exp \big(\lambda^2\|\theta_{k} - \theta^\star\|^2).$ Using~\eqref{eq:subgauss_prf_checkpt} and Jensen's inequality, this allows to conclude that for $\theta \sim \pi_\gamma,$ we have:
\begin{align*}
    \E\exp(\lambda^2\|\theta - \theta^\star\|^2) &\leq \E\exp\big(\lambda^2(1-\gamma\mu)\|\theta  - \theta^\star\|^2\big) \exp\big( \lambda^2 \gamma^2 \overline{K}^2\big) \\    
    &\leq \Big[\E\exp\big(\lambda^2\|\theta  - \theta^\star\|^2\big)\Big]^{(1-\gamma\mu)} \exp\big( \lambda^2 \gamma^2 \overline{K}^2\big) \\    
    \implies\E\exp(\lambda^2\|\theta - \theta^\star\|^2) &\leq\exp\Big(\frac{\lambda^2\gamma^2\overline{K}^2}{\gamma\mu}\Big) = \exp\big(\lambda^2\gamma \overline{K}^2/\mu\big),
\end{align*}
which shows the desired property for $\pi_\gamma.$

The task of proving~\ref{prop:invariant_properties_d} is more delicate. By assumption, we know that for all $\theta$, the gradient error $\varepsilon(\theta)$ satisfies
\begin{equation}\label{eq:subexp_assump}
    \big\| \|\varepsilon(\theta)\| \big\|_{L_p} \leq \overline{K}p \quad \text{for }\quad p \geq 1. 
\end{equation}
We denote $M_p^p(k) = \E\|\theta_k - \theta^\star\|^p.$ We will show by induction over $k$ that for all $p\geq 1,$ we have :
\begin{align}
    M_p(k) &\leq p\Big((1-\gamma\mu)^{k}\|\theta_0 - \theta^\star\|^2 + (1-(1-\gamma\mu)^{k+1})4\gamma\overline{K}^2/\mu\Big)^{1/2}\nonumber\\
    &=p\Big(\alpha^{k}d_0^2 + \frac{1-\alpha^{k+1}}{1-\alpha}\frac{4}{3}\gamma^2\overline{K}^2\Big)^{1/2}=:pK_{\pi}(k),\label{eq:ind_hyp}
\end{align}
where we introduced $\alpha := 1-\gamma\mu$ and $d_0^2 := \|\theta_0 - \theta^\star\|^2$ to lighten the notations.
For $k=0,$ we have $M_p(0) = \|\theta_0 - \theta^\star\|$ for all $p$ so that~\eqref{eq:ind_hyp} holds trivially. We now assume~\eqref{eq:ind_hyp} holds for $k$ and prove it for $k+1.$

We consider $M_{2p}(k+1),$ and compute
\begin{align*}
    M_{2p}^{2p}&(k+1) = \E \|\theta_{k+1} - \theta^\star\|^{2p} = \E \|\theta_{k} - \gamma G(\theta_{k}) - \theta^\star\|^{2p}\\
    =& \E\big( \|\theta_{k} - \gamma \nabla\cL(\theta_{k}) - \theta^\star\|^2 - 2\gamma \langle \theta_{k} - \gamma \nabla\cL(\theta_{k}) - \theta^\star, \varepsilon(\theta_{k}) \rangle +\gamma^2\|\varepsilon(\theta_{k})\|^2\big)^p \\
    \leq& \E \sum_{i=0}^p \binom{p}{i}\|\theta_{k} - \gamma \nabla\cL(\theta_{k}) - \theta^\star\|^{2i} \big( \gamma^2\|\varepsilon(\theta_{k})\|^2 - 2\gamma \langle \theta_{k} - \gamma \nabla\cL(\theta_{k}) - \theta^\star, \varepsilon(\theta_{k}) \rangle\big)^{p-i} \\
    \leq& \alpha^{2p}M_{2p}^{2p}(k) + p \E\|\theta_{k} - \gamma\nabla\cL(\theta_{k})-\theta^\star\|^{2p-2}(\gamma \|\varepsilon(\theta_{k})\|)^2 +\\
    & \E\sum_{i=0}^{p-2} \binom{p}{i}\|\theta_{k} - \gamma \nabla\cL(\theta_{k}) - \theta^\star\|^{2i} \big( \gamma^2\|\varepsilon(\theta_{k})\|^2 - 2\gamma \langle \theta_{k} - \gamma \nabla\cL(\theta_{k}) - \theta^\star, \varepsilon(\theta_{k}) \rangle\big)^{p-i},
\end{align*}
where we isolated the two last terms of the sum in the last step and used that $\varepsilon(\theta_{k})$ is centered conditionally on $\theta_{k}.$ Further, we have
\begin{align*}
    \E&\sum_{i=0}^{p-2} \binom{p}{i}\|\theta_{k} - \gamma \nabla\cL(\theta_{k}) - \theta^\star\|^{2i} \big( \gamma^2\|\varepsilon(\theta_{k})\|^2 - 2\gamma \langle \theta_{k} - \gamma \nabla\cL(\theta_{k}) - \theta^\star, \varepsilon(\theta_{k}) \rangle\big)^{p-i}\\
    &\leq \E\sum_{i=0}^{p-2}\!\binom{p}{i} \|\theta_{k} \!-\! \gamma \nabla\cL(\theta_{k}) \!-\! \theta^\star\|^{2i}\sum_{j=0}^{p-i}\!\binom{p\!-\!i}{j}(\gamma \|\varepsilon(\theta_{k})\|)^{2j}(2\gamma \|\theta_{k} \!-\!\gamma \nabla\cL(\theta_{k}) \!-\! \theta^\star\|\|\varepsilon(\theta_{k})\|)^{p-i-j}\\
    &=\sum_{i=0}^{p-2}\sum_{j=0}^{p-i}\binom{p}{i}\binom{p-i}{j}\|\theta_{k} - \gamma \nabla\cL(\theta_{k}) - \theta^\star\|^{p+i-j}(\gamma \|\varepsilon(\theta_{k})\|)^{p-i+j}2^{p-i-j}.
\end{align*}
Now consider an index $l = p-i+j$, note that $2p - l = p+i-j$ and we have $2\leq l \leq 2p.$ We compute the sum $\sum_{i=0}^{p-2}\sum_{j=0}^{p-i}\binom{p}{i}\binom{p-i}{j}2^{p-i-j}$ for a fixed value of $l:$
\begin{align}
    \sum_{\substack{0\leq i\leq p-2\\ 0\leq j\leq p-i \\ p-i+j = l}}& \binom{p}{i}\binom{p-i}{j}2^{p-i-j} = \sum_{i=0}^{p-2}\binom{p}{i}\binom{p-i}{l-(p-i)}2^{2(p-i)-l}\ind{p-i\leq l \leq 2(p-i)} \nonumber \\
    &=\!\!\! \sum_{i= 0\vee (p-l)}^{(p-2) \wedge (p-\lceil l/2\rceil)}\!\!\binom{p}{i}\binom{p-i}{l\!-\!(p-i)}2^{2(p-i)-l} \nonumber \\
    &= \!\!\!\sum_{i=0}^{(\lfloor l/2\rfloor) \wedge (l-2)}\!\!\binom{p}{i, i\!+\!p\!-\!l, l\!-\!2i} 2^{l-2i} \label{eq:trinomial_step}
\end{align}
where $\binom{p}{i, i+p-l, l-2i} = \frac{p!}{i!(i+p-l)!(l - 2i)!}$ is the \emph{trinomial} coefficient. The first equality above replaces $j$ in terms of $l,p$ and $i$ and adds the indicator function to restrict to valid terms. The second equality translates the constraints on the bounds on index $i$ and the third one applies the change of variable $i \to i+p-l$.

Similarly we find that:
\begin{align*}
    \sum_{\substack{0\leq i\leq p\\ 0\leq j\leq p-i \\ p-i+j = l}} \binom{p}{i}\binom{p-i}{j}2^{p-i-j} = \sum_{i=0}^{\lfloor l/2\rfloor}\binom{p}{i, i\!+\!p\!-\!l, l\!-\!2i} 2^{l-2i}.
\end{align*}
In what follows, we set the convention that $\binom{p}{i, i+p-l, l-2i} = 0$ whenever $i\wedge (i+p-l) \wedge (l-2i) < 0 $ which allows us to sum over all integer values without specifying the limits. For some variable $x,$ we multiply by $x^l,$ sum over $l$ and perform the change of variable $l\to l+2i$ to find
\begin{align*}
    \sum_{l}&\sum_{i=0}^{\lfloor l/2 \rfloor} \binom{p}{i, i+p-l, l-2i}2^{l-2i}x^l = \sum_{l,i} \binom{p}{i, i+p-l, l-2i}2^{l-2i}x^l \\
    &= \sum_{l,i} \binom{p}{i, l,p-l -i}2^{l}x^{l+2i} = \sum_{l,i} \binom{p}{i, l,p-l -i}(2x)^{l}(x^2)^{i}\\
    &= (x^2+2x+1)^p= (x+1)^{2p} = \sum_{l=0}^{2p} \binom{2p}{l}x^l.
\end{align*}
By identification of the terms in the sum over $l,$ this yields the equality
\begin{align*}
    \sum_{i=0}^{\lfloor l/2\rfloor}\binom{p}{i, i\!+\!p\!-\!l, l\!-\!2i} 2^{l-2i} = \binom{2p}{l}
\end{align*}
We plug back into~\eqref{eq:trinomial_step} and pay attention to the missing terms in the original sum. This happens when $l -2 < \lfloor l/2 \rfloor$ i.e.~when $\lceil l/2 \rceil < 2$ and since $2\leq l\leq 2p,$ it only happens for $l=2$ in which case the sum~\eqref{eq:trinomial_step} misses the term for $i=1$ which is equal to $p.$ Therefore, we get:
\begin{align*}
    \sum_{\substack{0\leq i\leq p-2\\ 0\leq j\leq p-i \\ p-i+j = l}} \binom{p}{i}\binom{p-i}{j}2^{p-i-j} = \binom{2p}{l} - p\ind{l=2}.
\end{align*}
Plugging back in the original sum, we find
\begin{align}
    M_{2p}^{2p}&(k+1) \leq \alpha^{2p} M^{2p}_{2p}(k) + p \E\|\theta_{k} - \gamma\nabla\cL(\theta_{k})-\theta^\star\|^{2p-2}(\gamma \|\varepsilon(\theta_{k})\|)^2 \nonumber \\
    &\quad + \E\sum_{l=2}^{2p}\Big(\binom{2p}{l} - p\ind{l=2}\Big)\|\theta_{k} - \gamma \nabla\cL(\theta_{k}) - \theta^\star\|^{2p-l}(\gamma \|\varepsilon(\theta_{k})\|)^{l} \nonumber \\
    &= \alpha^{2p} M^{2p}_{2p}(k) + \E\sum_{l=2}^{2p} \binom{2p}{l} \|\theta_{k} - \gamma \nabla\cL(\theta_{k}) - \theta^\star\|^{2p-l}(\gamma \|\varepsilon(\theta_{k})\|)^{l} \label{eq:moments_inequality} \\
    &\stackrel{\Circled{1}}{\leq}\alpha^{2p} M^{2p}_{2p}(k) + \sum_{l=2}^{2p}\binom{2p}{l}\big(\alpha M_{2p-l}(k)\big)^{2p-l}\big(\gamma l\overline{K}\big)^{l} \nonumber \\
    &\stackrel{\Circled{2}}{\leq} \big(2\alpha pK_{\pi}(k)\big)^{2p} + \sum_{l=2}^{2p}\binom{2p}{l}\big(\alpha (2p-l)K_\pi(k)\big)^{2p-l}\big(\gamma l\overline{K}\big)^{l}  \nonumber \\
    &\stackrel{\Circled{3}}{\leq} \big(2\alpha pK_{\pi}(k)\big)^{2p} + \frac{(2p)^{2p}e^{\frac{1}{24p}}}{\sqrt{2\pi}}\sum_{l=2}^{2p}\sqrt{\frac{2p}{l(2p-l)}}\big(\alpha K_\pi(k)\big)^{2p-l}(\gamma \overline{K})^{l}  \nonumber \\
    &\stackrel{\Circled{4}}{\leq} (2p)^{2p}\Big[\big(\alpha^2 K_{\pi}(k)^2\big)^{p} + p\kappa(p)\big(\alpha^2 K_\pi(k)^2\big)^{p-1}\big(\gamma^2 \overline{K}^2\big)\Big] \nonumber \\
    &\stackrel{\Circled{5}}{\leq} (2p)^{2p}\Big[\alpha^2 K_{\pi}(k)^2 + \frac{2}{3}\big(\gamma^2 \overline{K}^2\big)\Big]^{p} \label{eq:subexp_prf_checkpt},
\end{align}
where $\Circled{1}$ uses~\eqref{eq:subexp_assump} and Lemma~\ref{lem:contraction}, $\Circled{2}$ uses our induction hypothesis~\eqref{eq:ind_hyp}, $\Circled{3}$ uses Stirling's approximation, $\Circled{4}$ uses that $\alpha K_\pi(k) \geq \gamma \overline{K}$ (keep in mind that the condition $\gamma \leq (2\mu)^{-1}$ implies $\alpha \geq 1/2$) to substitute all terms in the sum with the term for $l=2$ introducing $\kappa(p) = \frac{e^{\frac{1}{24p}}}{\sqrt{2\pi}}\big(\frac{2p-1}{p}\big)\sqrt{\frac{p}{2p-2}}$ and $\Circled{5}$ uses the inequalities $\sup_{p\geq 2}\kappa(p)\leq 2/3$ and $a^p + pa^{p-1}b \leq (a+b)^p$ for $a, b\geq 0$ and $p\in\N^*.$ Note that, from~\eqref{eq:ind_hyp}, we have
\begin{align*}
    \alpha K_{\pi}(k)^2 = \alpha^{k+1}d_0^2 + \frac{\alpha-\alpha^{k+2}}{1-\alpha}\frac{4}{3}\gamma^2\overline{K}^2 = K_{\pi}(k+1)^2 - \frac{4}{3}\gamma^2\overline{K}^2,
\end{align*}
which we can plug into~\eqref{eq:subexp_prf_checkpt} to obtain $M_{2p}(k+1) \leq 2pK_{\pi}(k+1).$ Since this implies similar bounds for moments of uneven orders $M_{2p-1}(k+1),$ the induction over $k$ is complete and we have that $\|\theta_k - \theta^\star\|$ is sub-exponential with the desired constant.

Finally, we turn to $\theta~\sim\pi_\gamma$ and denote $M_p^p = \E\|\theta - \theta^\star\|^p$. For $p = 2,$ we have using Proposition~\ref{prop:invariant_properties}~\ref{prop:invariant_properties_b}
\begin{equation*}
    M_2^2 \leq \frac{\gamma(2\overline{K})^2 }{\mu(2 - \gamma \mu)} \leq \frac{\gamma(2\overline{K})^2 }{\mu},
\end{equation*}
which immediately entails $M_1\leq 2\overline{K}\sqrt{\gamma /\mu}.$ We will show by induction that 
\begin{equation}\label{eq:ind_hyp2}
    M_p \leq K_{\pi} p \quad \text{for all } \quad p\geq 1,
\end{equation}
with $K_\pi = C\overline{K}\sqrt{\gamma/\mu}$ for some $C \geq 2.$ For $p\geq 2,$ we assume~\eqref{eq:ind_hyp2} holds up to $2p-2$ and consider $M_{2p}.$ Through similar computations to the above and since the invariance of $\pi_\gamma$ implies $M_{2p}(k)=M_{2p}(k+1),$ starting from~\eqref{eq:moments_inequality}, we find 
\begin{align}
    (1-&\alpha^{2p})M_{2p}^{2p} \leq \E\sum_{l=2}^{2p} \binom{2p}{l} \|\theta - \gamma \nabla\cL(\theta) - \theta^\star\|^{2p-l}(\gamma \|\varepsilon(\theta)\|)^{l}  \label{eq:moments_inequality2} \\
    &\stackrel{\Circled{1}}{\leq}\sum_{l=2}^{2p}\binom{2p}{l}(\alpha M_{2p-l})^{2p-l}(\gamma l\overline{K})^{l} \nonumber \\
    &\stackrel{\Circled{2}}{\leq} (\gamma 2p\overline{K})^{2p} + \sum_{l=2}^{2p-1}\binom{2p}{l}(\alpha(2p-l)K_\pi)^{2p-l}(\gamma l\overline{K})^{l} \nonumber \\
    &\stackrel{\Circled{3}}{\leq} ((2p\!-\!1)K_\pi)^{2p}\Big(\frac{2p}{2p\!-\!1}\Big)^{2p}\Big[(\gamma \overline{K}/K_\pi)^{2p} \nonumber \\
    &\quad + \frac{e^{\frac{1}{24p}}}{\sqrt{2\pi}}\sum_{l=2}^{2p-1}\sqrt{\frac{2p}{l(2p\!-\!l)}} \alpha^{2p-l}(\gamma \overline{K}/K_{\pi})^{l}\Big] \nonumber \\
    &\stackrel{\Circled{4}}{\leq} ((2p\!-\!1)K_\pi)^{2p}\Big(\frac{2p}{2p\!-\!1}\Big)^{2p}\Big[(\gamma \overline{K}/K_\pi)^{2p}  \nonumber \\
    &\quad + \frac{e^{\frac{1}{24p}}}{\sqrt{4\pi}}(2p\!-\!2)\sqrt{\frac{2p}{2p\!-\!2}} \alpha^{2p-2}(\gamma \overline{K}/K_{\pi})^{2}\Big], \nonumber 
\end{align}
where $\Circled{1}$ uses~\eqref{eq:subexp_assump} and Lemma~\ref{lem:contraction}, $\Circled{2}$ uses~\eqref{eq:ind_hyp2}, $\Circled{3}$ uses Stirling's approximation and $\Circled{4}$ uses that $\alpha > \gamma \overline{K}/K_\pi.$ We now use the following inequalities for $p\geq 2:$
\begin{align*}
    1-\alpha^{2p} &=\gamma\mu\sum_{i=0}^{2p-1}\alpha^i \geq 2p\gamma\mu \alpha^{2p-1} \quad \quad \gamma \leq 1/(2\mu)\\
    \Big(\frac{2p}{2p-1}\Big)^{2p-1} &= \Big(1+\frac{1}{2p-1}\Big)^{2p-1} \leq e \quad \text{and}\quad \frac{\sqrt{2p(2p-2)}}{2p-1}\leq 1, 
\end{align*} 
in addition to the choice $K_\pi = C\overline{K}\sqrt{\gamma/\mu}$ with $C=2$ to find
\begin{align*}
    M_{2p}^{2p} &\leq ((2p\!-\!1)K_\pi)^{2p}\Big(\frac{e}{1-\gamma\mu}\Big)\Big[ \Big(\frac{\sqrt{\gamma\mu}}{1-\gamma\mu}\Big)^{2p-2}\frac{C^{-2p}}{2p-1} + \frac{e^{\frac{1}{48}} C^{-2}}{\sqrt{4\pi}}\Big] \\
    &\leq ((2p\!-\!1)K_\pi)^{2p}(2e)\Big[ \frac{1}{6}\Big(\frac{1}{\sqrt{2}}\Big)^{2p} + \frac{e^{\frac{1}{48}}}{8\sqrt{\pi}}\Big] \leq ((2p\!-\!1)K_\pi)^{2p}
\end{align*}
which yields the desired bound~\eqref{eq:ind_hyp2} for $M_{2p}$ as well as $M_{2p-1}$ through $M_{2p-1}\leq M_{2p}.$ This finishes the induction.

\subsection{sub-Gaussianity under weaker conditions}\label{sec:weaker_subgauss}

In this section, we prove a sub-Gaussian property of the invariant distribution similar to Proposition~\ref{prop:invariant_concentration}~\ref{prop:invariant_properties_c} which holds if Assumption~\ref{asm:smooth_strongconvex} is replaced by the following weaker conditions on the objective.
\begin{assumption}\label{asm:lingrad_dissip}
    There exist positive constants $0< \mu \leq L < +\infty$ and $\beta < +\infty$ such that the objective $\cL$ satisfies the following properties:
    \begin{enumerate}[label=(\roman*)]
        \item \label{asm:lingrad} (Gradient linear growth) The gradient $\nabla\cL$ is such that
        \begin{equation*}
            \|\nabla \cL (\theta)\| \leq L\big(1 + \|\theta\|\big).
        \end{equation*}
        \item \label{asm:dissip} (Dissipativity) We have the lower bound
        \begin{equation*}
            \langle \theta, \nabla\cL(\theta)\rangle \geq \mu\|\theta\|^2 - \beta.
        \end{equation*}
    \end{enumerate}
\end{assumption}
Assumption~\ref{asm:lingrad_dissip} allows for non-convex, non-smooth objectives but requires a quadratic growth. Under such conditions, the unique global minimum $\theta^\star$ may not exist. Therefore, we adapt Assumption~\ref{asm:gradient} by setting $\theta^\star = 0$ to prove the following result.
\begin{proposition}\label{prop:weaker_subgauss}
    Under Assumptions~\ref{asm:lingrad_dissip} and~\ref{asm:gradient} with $\theta^\star=0,$ the Markov chain $(\theta_t)_{t\geq 0}$ defined by iteration~\eqref{eq:sgd_iteration} with step-size 
    \begin{equation*}
        \gamma< \frac{\mu}{8L^2 +  L^2_{\sigma}}
    \end{equation*}
    converges geometrically to a unique invariant measure $\pi_{\gamma}$. Moreover\textup, if Assumption~\ref{asm:grad_concentration}~\ref{asm:grad_subgauss} holds, for $\theta \sim \pi_\gamma,$ the invariant distribution $\pi_\gamma$ is such that $\|\theta\| \in \widetilde{\Psi}_2(\Breve{K})$ with $\Breve{K} = 2\sqrt{\big(\beta + 2\gamma(L^2+\overline{K}^2)\big)/\mu}$.
\end{proposition}
\begin{proof}
    The convergence proof is mostly similar to Theorem~\ref{thm:ergodicity} and mainly differs in the way to obtain an equivalent of Inequality~\eqref{eq:contract}. We consider a fixed $\theta\in\R^d$ and compute:
    \begin{align*}
        P_{\gamma}\|\theta\|^2 &= \E \|\theta - \gamma G(\theta)\|^2 = \E \|\theta - \gamma (\nabla\cL(\theta)+\varepsilon(\theta))\|^2 \\
        &\stackrel{\Circled{1}}{=} \E \big[ \|\theta\|^2 -2\gamma\langle \theta, \nabla\cL(\theta) \rangle + \gamma^2 \|\nabla\cL(\theta) + \varepsilon(\theta)\|^2 \big]\\
        &\stackrel{\Circled{2}}{\leq} \|\theta\|^2(1 -2\gamma\mu) +2\gamma\beta  + \gamma^2 \E \big[ \|\nabla\cL(\theta) + \varepsilon(\theta)\|^2 \big]\\
        &\stackrel{\Circled{3}}{=} \|\theta\|^2(1 -2\gamma\mu) +2\gamma\beta  + \gamma^2 \|\nabla\cL(\theta)\|^2 + \gamma^2\E \big[ \|\varepsilon(\theta)\|^2 \big]\\
        &\stackrel{\Circled{4}}{\leq} \|\theta\|^2\big(1 -2\gamma\mu + \gamma^2(2L^2+L^2_{\sigma})\big) +2\gamma\beta  + 2\gamma^2L^2 + \gamma^2\sigma^2,
    \end{align*}
    where $\Circled{1}$ and $\Circled{3}$ use Assumption~\ref{asm:gradient}~\ref{asm:gradient_centered}, $\Circled{2}$ uses Assumption~\ref{asm:lingrad_dissip}~\ref{asm:dissip} and $\Circled{4}$ uses Assumption~\ref{asm:lingrad_dissip}~\ref{asm:lingrad} and Assumption~\ref{asm:gradient}~\ref{asm:gradient_regular}. Our choice of $\gamma$ ensures that the factor in front of $\|\theta\|^2$ is $<1.$ From here, one can easily derive a similar inequality to~\eqref{eq:lyapunov_cond} and unfold the rest of Theorem~\ref{thm:ergodicity}'s proof pattern with $\theta^\star = 0$ leading to geometric ergodicity. We omit the details and focus on proving the sub-Gaussian property of $\pi_{\gamma}.$

    As in the proof of Proposition~\ref{prop:invariant_concentration}, we use the fact that $\theta$ and $\theta - \gamma G(\theta)$ have the same distribution when $\theta \sim \pi_\gamma:$
    \begin{align*}
        \E\exp\big(\lambda^2\|\theta\|^2\big) &= \E\exp\big(\lambda^2\|\theta - \gamma G(\theta)\|^2\big) = \E\exp\big(\lambda^2\|\theta - \gamma \nabla\cL(\theta) -\gamma\varepsilon(\theta)\|^2\big) \\
         &= \E\exp\big(\lambda^2\big[\|\theta\|^2 - 2\gamma\langle \theta, \nabla\cL(\theta) + \varepsilon(\theta)\rangle \\
         &\quad \quad+ \gamma^2 \|\nabla\cL(\theta) + \varepsilon(\theta)\|^2\big]\big) \\
         &\leq \E\exp\big(\lambda^2\big[\|\theta\|^2 - 2\mu\gamma\|\theta\|^2 + 2\beta \gamma -2\gamma \langle \theta, \varepsilon(\theta)\rangle  \\
         &\quad \quad+ 2\gamma^2 \big(2L^2(1+\|\theta\|^2) + \|\varepsilon(\theta)\|^2\big)\big]\big) \\
         &\leq \E\exp\big(\lambda^2\big[(1 - 2\mu\gamma + 4\gamma^2 L^2)\|\theta\|^2 -2\gamma \langle \theta, \varepsilon(\theta)\rangle + 2\beta \gamma  \\
         &\quad \quad+ 2\gamma^2 \big(2L^2 + \|\varepsilon(\theta)\|^2\big)\big]\big)
    \end{align*}
    We now condition on $\theta$ and use similar arguments to the proof of Proposition~\ref{prop:invariant_concentration}~\ref{prop:invariant_properties_c} to find that for $|\lambda| \leq \big(\overline{K}\gamma\sqrt{2}\big)^{-1},$ we have:
    \begin{align*}
        \E\big[\exp\big(\lambda^2&(-2\gamma \langle \theta, \varepsilon(\theta)\rangle + 2\gamma^2 \|\varepsilon(\theta)\|^2)\big)\vert \theta\big]\\
        &\leq \E\big[\exp\big(-\gamma (2\lambda)^2\langle \theta, \varepsilon(\theta)\rangle\big)\vert \theta\big]^{1/2}\E\big[\exp\big((2\lambda\gamma)^2 \|\varepsilon(\theta)\|^2\big)\vert \theta\big]^{1/2}\\
        &\leq \exp\big(8\gamma^2\lambda^4 \|\theta\|^2\overline{K}^2 + 2\lambda^2\gamma^2\overline{K}^2\big).
    \end{align*}
    We now further restrict $\lambda$ to $|\lambda| \leq \big(2\overline{K}\sqrt{2\gamma/\mu}\big)^{-1}$ and plug back above to find
    \begin{align*}
        \E\exp\big(\lambda^2\|\theta\|^2\big) &\leq \E \exp\big(\lambda^2\big[(1-2\mu\gamma + 4\gamma^2(L^2 + 2\lambda^2\overline{K}^2))\|\theta\|^2\\
        &\quad \quad + 2\gamma(\beta + \gamma (2L^2 + \overline{K}^2))\big]\big)\\
        &\leq \E \exp\big(\lambda^2\big[(1-\mu\gamma + 4\gamma^2L^2)\|\theta\|^2 + 2\gamma(\beta + \gamma (2L^2 + \overline{K}^2))\big]\big)\\
        &\leq \E \big[\exp\big(\lambda^2\|\theta\|^2\big)\big]^{1-\mu\gamma + 4\gamma^2L^2} \exp\big(2\lambda^2 \gamma(\beta + \gamma (2L^2 + \overline{K}^2))\big),
    \end{align*}
    where we used Jensen's inequality. Finally, using our choice of $\gamma,$ this leads to
    \begin{align*}
        \E\exp(\lambda^2\|\theta\|^2) &\leq \exp\Big(2\lambda^2\Big(\frac{\beta + \gamma(2L^2+\overline{K}^2)}{\mu - 4\gamma L^2}\Big)\Big)\\
        &\leq \exp\big((2\lambda)^2\big(\beta + \gamma(2L^2+\overline{K}^2)\big)/\mu\big),
    \end{align*}
    which implies the result.
\end{proof}
A contractive optimization inequality such as~\eqref{eq:contract} combined with a centered and uniform concentration condition on the gradient noise appear to be necessary to obtain results such as Propositions~\ref{prop:invariant_concentration} and~\ref{prop:weaker_subgauss}.

\subsection{Proof of Lemma~\ref{lem:pfinite_moments_nonunif_subexp}}\label{sec:proof_lem_pfinite_moments_nonunif_subexp}

Without loss of generality, we consider moments of even order. For $j\geq 1,$ denoting $M_{2j}^{2j} = \E\|\theta - \theta^\star\|^{2j}$ and starting from Equation~\eqref{eq:moments_inequality2} which was obtained in the proof of Proposition~\ref{prop:invariant_concentration} and using Lemma~\ref{lem:contraction} and our assumption on $\|\varepsilon(\theta)\|$ yields
\begin{align*}
    \big(1-(1-&\gamma\mu)^{2j}\big)M_{2j}^{2j} \leq \E\sum_{l=2}^{2j} \binom{2j}{l} \|\theta - \gamma \nabla\cL(\theta) - \theta^\star\|^{2j-l}(\gamma \|\varepsilon(\theta)\|)^{l} \\
    &\leq \E\sum_{l=2}^{2j} \binom{2j}{l} \big((1-\gamma\mu)\|\theta - \theta^\star\|\big)^{2j-l}\gamma^l\big( K\|\theta - \theta^\star\| + \underline{K}\big)^{l}\\
    &\leq \E\sum_{l=2}^{2j} \binom{2j}{l} \big((1-\gamma\mu)\|\theta - \theta^\star\|\big)^{2j-l}\gamma^l\sum_{k=0}^l\binom{l}{k} \big( K\|\theta - \theta^\star\|\big)^{l-k}\underline{K}^k\\
    &\leq \sum_{l=2}^{2j} \binom{2j}{l} (1-\gamma\mu)^{2j-l}\gamma^l \Big(K^l M_{2j}^{2j} + \sum_{k=1}^l\binom{l}{k} K^{l-k}\underline{K}^k M_{2j-k}^{2p-k}\Big).
\end{align*}
By sorting out the factors of $M_{2j}^{2j}$ and rearranging the terms, we find
\begin{align*}
    \Big(1 - (1-\gamma\mu)^{2j}&- \sum_{l=2}^{2j} \binom{2j}{l} (1-\gamma\mu)^{2j-l}(\gamma K)^l\Big)M_{2j}^{2j} \leq \\ &\sum_{l=2}^{2j} \binom{2j}{l} (1-\gamma\mu)^{2j-l}\gamma^l \sum_{k=1}^l\binom{l}{k} K^{l-k}\underline{K}^k M_{2j-k}^{2j-k}.
\end{align*}
Assuming that $M_i < \infty$ for $i < 2j,$ the above inequality would allow us to show that $M_{2j} < \infty$ as well provided that the factor of $M_{2j}^{2j}$ on the LHS is positive. We now use the inequalities 
\begin{equation}
    (1-x)^k \leq (1-kx) + k(k-1)x^2/2, \label{eq:devlim1}
\end{equation}
and
\begin{equation}
    (1-kx)\leq (1-x)^k,\label{eq:devlim2}
\end{equation}
valid for $x\geq 0$ and $k\in\N^*$ to find
\begin{align*}
    1 - (1-\gamma\mu)^{2j}&- \sum_{l=2}^{2j} \binom{2j}{l} (1-\gamma\mu)^{2j-l}(\gamma K)^l \\
    &= 1 - (1-\gamma (\mu-K))^{2j} + 2j\gamma K (1 - \gamma \mu)^{2j-1}\\
    &\geq 2j\gamma (\mu - K) - 2j(2j-1)\gamma^2 (\mu - K)^2/2 + 2j\gamma K(1-\gamma\mu)^{(2j-1)} \\
    &\geq 2j\gamma (\mu - K) - 2j(2j-1)\gamma^2 (\mu - K)^2/2 + 2j\gamma K(1-(2j-1)\gamma\mu) \\
    &= 2j\gamma \mu - 2j(2j-1)\gamma^2 ((\mu - K)^2/2 + \mu K) \\
    & = 2j\gamma (\mu - \gamma (2j-1)(\mu^2+K^2)/2) \\
    &\geq 0,
\end{align*}
where the first inequality uses~\eqref{eq:devlim1} with $x=\gamma(\mu-K), k=2j,$ the second one uses~\eqref{eq:devlim2} with $x=\gamma\mu, k=2j-1$ and the last one follows from the bound we imposed on $\gamma.$

Therefore, we can deduce that $M_{2j} < \infty.$ Since a similar argument works for $M_i < \infty$ with $i < 2p$ with a weaker condition on $\gamma,$ the result follows.

\subsection{Comparison with~\cite{dieuleveut2020bridging}}\label{sec:compare_with_18}

In Section~\ref{sec:proof_invariant_concentration}, we showed that for $p\geq1,$
\begin{align}
    M_{2p}(k) &= \big(\E\|\theta_k - \theta^\star\|^{2p}\big)^{1/(2p)} \nonumber\\
    &\leq 2p\Big((1-\gamma\mu)^{k}\|\theta_0 - \theta^\star\|^2 + (1-(1-\gamma\mu)^{k+1})4\gamma\overline{K}^2/\mu\Big)^{1/2}\nonumber
\end{align}
and that
    \begin{equation*}
        M^{2p}_{2p} := \int_{\R^d}\|\theta - \theta^*\|^{2p}\pi_{\gamma}(d\theta) \leq ((2p-1)K_{\pi})^{2p},
    \end{equation*}
with $K_{\pi} = 2 \overline{K}\sqrt{\gamma/\mu}.$ The previous bounds can be compared to~\cite[Lemma 13]{dieuleveut2020bridging} which states that, for a given $p,$:
\begin{equation}
    M_{2p}^2(k) \leq \big(1-2\gamma\mu(1 - C_p\gamma L/2)\big)^k M_{2p}^2(0) + \frac{D_p \gamma\tau^2_{2p}}{\mu}\label{eq:bridging_M2pk_bound}
\end{equation}
and
\begin{equation}
    M^{2p}_{2p} \leq (D_p \gamma \tau^2_{2p}/\mu)^p,\label{eq:bridging_M2p_bound}
\end{equation}
where $\tau_{2p}$ is an upperbound on $\big\|\|\varepsilon(\theta^\star)\|\big\|_{L_{2p}}$ and $C_p, D_p$ are constants depending only on $p.$ The dependency w.r.t. $\mu, \gamma$ and $p$ is therefore similar. A comparison between the results, assumptions and proof methods of~\cite[Lemma 13]{dieuleveut2020bridging} and Proposition~\ref{prop:invariant_concentration}~\ref{prop:invariant_properties_d} and Lemma~\ref{lem:pfinite_moments_nonunif_subexp} is therefore in order. We detail the differences and similarities below : 
\begin{itemize}
\item \underline{Concentration bounds and step-size condition :}~\cite[Lemma 13]{dieuleveut2020bridging} requires a step-size $\gamma\leq 1/(LC_p).$ The involved constants $C_p, D_p$ are exponential in $p$ which is reflected on~\eqref{eq:bridging_M2pk_bound} and~\eqref{eq:bridging_M2p_bound} and the step-size. In Proposition~\ref{prop:invariant_concentration}~\ref{prop:invariant_properties_d}, we do not add any significant constraint on the step-size beyond the convergence condition of Theorem~\ref{thm:ergodicity} and show that the iterates $\theta_k$ are sub-Gaussian/sub-exponential as well as $\pi_\gamma$ with a limit constant $O(\overline{K}\sqrt{\gamma/\mu}).$ In Lemma~\ref{lem:pfinite_moments_nonunif_subexp}, we assume a step-size in $O(1/p)$ and show finiteness of the $p$-moment of $\pi_\gamma$ without an explicit bound.
\item \underline{Assumptions :} Our Assumption~\ref{asm:grad_concentration}~\ref{asm:grad_subexp} on the gradient noise is uniform in $\theta$ allowing to derive Proposition~\ref{prop:invariant_concentration}~\ref{prop:invariant_properties_d}. In contrast, the upper bound assumed in Lemma~\ref{lem:pfinite_moments_nonunif_subexp} is much weaker taking arbitrarily high values depending on $\theta.$ The latter is more comparable with~\cite[Assumption A4]{dieuleveut2020bridging} which only assumes a moment bound on the noise at the optimum $\varepsilon(\theta^\star)$ and combines it with almost sure co-coercivity to obtain bounds for $\varepsilon(\theta)$ with arbitrary $\theta$ in the proofs.
\item \underline{Method :} The proofs of Proposition~\ref{prop:invariant_concentration}~\ref{prop:invariant_properties_d} and Lemma~\ref{lem:pfinite_moments_nonunif_subexp} and~\cite[Lemma 13]{dieuleveut2020bridging} are similarly based on the development in a trinomial sum of the quantity
\begin{equation*}
\|\theta - \gamma G(\theta) - \theta^\star\|^{2p} = \big(\|\theta - \gamma\nabla\mathcal{L}(\theta) - \theta^\star\|^2 - 2\gamma \langle \theta - \gamma\nabla\mathcal{L}(\theta) - \theta^\star, \varepsilon(\theta)\rangle + \gamma^2\|\varepsilon(\theta)\|^2 \big)^p.
\end{equation*}
This approach appears to combine better with Assumption~\ref{asm:grad_concentration}~\ref{asm:grad_subexp} allowing to make our estimation of $M_{2p}$ and $M_{2p}(k)$ in the proof of Proposition~\ref{prop:invariant_concentration}~\ref{prop:invariant_properties_d} tighter.
\item \underline{Induction index :} The proof of Lemma 13 in~\cite{dieuleveut2020bridging} uses an induction argument over the iteration index $k$ of $\theta_k.$ In the proof of Proposition~\ref{prop:invariant_concentration}~\ref{prop:invariant_properties_d}, we use a similar induction in order to handle $M_{2p}(k)$ but arguing for all $p$ rather than a single one. In contrast, the part handling $M_{2p}$ (for the invariant distribution) is proved by induction over the moment orders via $p.$
\end{itemize}

\subsection{Proof of Proposition~\ref{prop:invariant_special_concentration}}\label{sec:proof_invariant_special_concentration}

We now prove~\ref{prop:invariant_properties_e}. Let $\theta_{k+1}$ be an SGD iterate and define the \emph{gradient step} function $g_\gamma$ as
\begin{equation*}
    g_{\gamma}(\vartheta) = \vartheta - \gamma \nabla\cL(\vartheta)\quad \text{ for }\quad \vartheta\in\R^d.
\end{equation*}
Note that, by Lemma~\ref{lem:contraction}, $g_{\gamma}$ is $(1-\gamma \mu)$-Lipschitz. Under Assumption~\ref{asm:grad_special_concentration}~\ref{asm:grad_special_subgauss}, we have for all $\lambda \in \R:$
\begin{align*}
    \sup_{f\in \Lip(\R^d)} &\E \exp\big(\lambda (f(\theta_{k+1} ) \!-\! \E f(\theta_{k+1} ))\big) \\
    &= \sup_{f\in \Lip(\R^d)} \E \exp\big(\lambda (f(\theta_{k} \!-\!\gamma G(\theta_{k}) ) \!-\! \E f(\theta_{k} \!-\!\gamma G(\theta_{k}) ))\big) \\
    &= \sup_{f\in \Lip(\R^d)} \E \exp\Big(\lambda \big(f((\theta_{k} - \gamma \nabla\cL(\theta_{k})) - \gamma \varepsilon(\theta_{k})) \\
    &\quad - \E f((\theta_{k} - \gamma \nabla\cL(\theta_{k})) - \gamma \varepsilon(\theta_{k}))\big)\Big) \\
    &= \sup_{f\in \Lip(\R^d)} \E \exp\Big(\lambda \big(f( g_{\gamma}(\theta_{k})) - \E f( g_{\gamma}(\theta_{k}))\big) + \lambda \big( f( g_{\gamma}(\theta_{k}) - \gamma \varepsilon(\theta_{k})) \\ 
    &\quad- f( g_{\gamma}(\theta_{k})) - \E [f( g_{\gamma}(\theta_{k}) - \gamma \varepsilon(\theta_{k})) - f( g_{\gamma}(\theta_{k}))]\big)\Big).
\end{align*}
Conditioning on $\theta_{k},$ it is clear that 
\begin{align*}
    \phi\big(G(\theta_{k})\big) &:= f\big( g_{\gamma}(\theta_{k}) - \gamma (G(\theta_{k})-\nabla\cL (\theta_{k}))\big) - f\big( g_{\gamma}(\theta_{k})\big) \\
    &= f\big( g_{\gamma}(\theta_{k}) - \gamma \varepsilon(\theta_{k})\big) - f\big( g_{\gamma}(\theta_{k})\big)
\end{align*}
is a $\gamma$-Lipschitz function of $G(\theta_{k}).$ In addition, $f( g_{\gamma}(\theta_{k}))$ is a $(1-\gamma\mu)$-Lipschitz function of $\theta_{k},$ therefore by reparametrizing the space of Lipschitz functions, we find
\begin{align}
    \sup_{f\in \Lip(\R^d)} &\E \exp\big(\lambda (f(\theta_{k+1} ) \!-\! \E f(\theta_{k+1} ))\big) \nonumber\\
    &\stackrel{\Circled{1}}{\leq} \sup_{f\in \Lip(\R^d)} \E \exp\Big(\lambda \big(f( g_{\gamma}(\theta_{k})) - \E f( g_{\gamma}(\theta_{k}))\big)\Big)\exp(\lambda^2 \gamma^2 K^2)\nonumber\\
    &\stackrel{\Circled{2}}{\leq} \sup_{f\in \Lip(\R^d)} \E \exp\Big(\lambda(1-\gamma\mu) \big(f(\theta_{k}) - \E f(\theta_{k})\big)\Big)\exp(\lambda^2 \gamma^2 K^2)\label{eq:spec_conc_prf_checkpt}
\end{align}
where $\Circled{1}$ uses that $\phi\big(G(\theta_k)\big)$ is $\gamma$-Lipschitz together with Assumption~\ref{asm:grad_special_concentration}~\ref{asm:grad_special_subgauss} and $\Circled{2}$ uses that $f( g_{\gamma}(\cdot))$ is $(1-\gamma\mu)$-Lipschitz replacing it by $(1-\gamma\mu)f(\cdot).$ The previous relationship can be iterated to find
\begin{alignat*}{2}
    &&&\sup_{f\in \Lip(\R^d)} \E \exp\big(\lambda (f(\theta_{k+1} ) \!-\! \E f(\theta_{k+1} ))\big) \\
    \leq &&&\sup_{f\in \Lip(\R^d)} \E \exp\big(\lambda(1-\gamma\mu)^{k+1} \big(f(\theta_{0}) - \E f(\theta_{0})\big)\big)\exp\Big(\lambda^2 \gamma^2 K^2 \sum_{i=0}^k(1-\gamma\mu)^{2i}\Big)\\
    = &&&\exp\Big(\lambda^2 \gamma^2 K^2 \frac{ 1 - (1-\gamma\mu)^{2k+2}}{1 - (1-\gamma\mu)^2}\Big),
\end{alignat*}
which implies the desired property for $\theta_{k+1}.$

In order to obtain the property for the limit distribution, we consider a stationary $\theta_k \sim \pi_{\gamma}$ such that  $\theta_{k+1} \sim \pi_{\gamma}$ as well. Resuming from~\eqref{eq:spec_conc_prf_checkpt} and using Jensen's inequality, we get that for $\theta\sim\pi_\gamma:$
\begin{align*}
    \sup_{f\in \Lip(\R^d)} &\E \exp\big(\lambda (f(\theta ) \!-\! \E f(\theta ))\big) \nonumber\\
    &\leq \Big(\sup_{f\in \Lip(\R^d)} \E \exp\big(\lambda\big(f(\theta) - \E f(\theta)\big)\big)\Big)^{1-\gamma\mu}\exp(\lambda^2 \gamma^2 K^2)\\
    &\implies \sup_{f\in \Lip(\R^d)} \E \exp\big(\lambda (f(\theta ) \!-\! \E f(\theta ))\big) \leq \exp(\lambda^2 K^2 \gamma/\mu).
\end{align*}
The proof of~\ref{prop:invariant_properties_f} is analogous except for the fact that the above inequalities only hold for $|\lambda| \leq (\gamma K)^{-1}$ when $f(G(\theta))$ is $K$-sub-exponential for all $f\in \Lip(\R^d)$. The rest of the proof is unchanged and since $\big(K\sqrt{\gamma/\mu}\big)^{-1} < (\gamma K)^{-1}$, we similarly obtain the sub-exponential properties.

\subsection{Proof of Proposition~\ref{prop:wasserstein_convergence}}
\label{sec:proof_wasserstein_convergence}
Let $\theta_1 \sim \nu_1$ and $\theta_2 \sim \nu_2$ be random variables such that $\cW_2^2(\nu_1, \nu_2) = \E\big[\|\theta_1 - \theta_2\|^2\big].$ Such a pair of variables exists by~\cite[Theorem 4.1]{villani2009optimal}.

In this proof, we will use the notations $G(\theta, \zeta)$ and $\varepsilon_{\zeta}(\theta)$ for the gradient and noise samples due to the particular relevance of the sample $\zeta$ in this context. We consider the set of couplings of the distributions $\nu_1 P_\gamma$ and $\nu_2 P_\gamma$ through the random variables $G(\theta_1, \zeta_1)$ and $G(\theta_2, \zeta_2)$ such that
\begin{align*}
     \theta_1 - \gamma G(\theta_1, \zeta_1) \sim \nu_1 P_\gamma \quad &\text{and}\quad  \theta_2 - \gamma G(\theta_2, \zeta_2) \sim \nu_2 P_\gamma.
\end{align*}
Recall also that by Assumption~\ref{asm:gradient}~\ref{asm:gradient_centered}, for $j=1,2,$ conditionally on $\theta_j,$ we have
\begin{equation}\label{eq:zeroexpect_grad_noise}
    G(\theta_j, \zeta_j) = \nabla\cL(\theta_j) + \varepsilon_{\zeta_j}(\theta_j) \quad \text{with}\quad \E [\varepsilon_{\zeta_j}(\theta_j)\vert \theta_j] = 0.
\end{equation}
Taking the infimum over the variables $\varepsilon_{\zeta_j}(\theta_j),$ we compute
\begin{align*}
    \cW_{2}^2(\nu_1P_{\gamma}, \nu_2P_{\gamma}) &= \inf_{\varepsilon_{\zeta_j}\!(\theta_j)}\E \big\|\theta_1 - \gamma G(\theta_1, \zeta_1) - (\theta_2 - \gamma G(\theta_2, \zeta_2))\big\|^2 \\
    &= \inf_{\varepsilon_{\zeta_j}\!(\theta_j)} \E \big[\big\|\theta_1 - \gamma \nabla \cL(\theta_1) - (\theta_2 - \gamma \nabla \cL(\theta_2))\big\|^2 \\
    &\quad -2\gamma\langle \theta_1 - \gamma\nabla\cL(\theta_1) - (\theta_2 - \gamma \nabla\cL(\theta_2)), \varepsilon_{\zeta_1}(\theta_1) - \varepsilon_{\zeta_2}(\theta_2) \rangle \\
    &\quad +\gamma^2 \|\varepsilon_{\zeta_1}(\theta_1) - \varepsilon_{\zeta_2}(\theta_2)\|^2 \big]\\
    &\stackrel{\Circled{1}}{=} \E \big[\big\|\theta_1 - \gamma \nabla \cL(\theta_1) - (\theta_2 - \gamma \nabla \cL(\theta_2))\big\|^2 \\
    &\quad +\gamma^2 \inf_{\varepsilon_{\zeta_j}\!(\theta_j)}\E\big[\|\varepsilon_{\zeta_1}(\theta_1) - \varepsilon_{\zeta_2}(\theta_2)\|^2 \vert \theta_1, \theta_2\big]\big]\\
    &\stackrel{\Circled{2}}{\leq} \E\big[(1-\gamma\mu)^2\|\theta_1 - \theta_2\|^2 + \gamma^2 \cW_2^2\big(\mathcal{D}(\varepsilon_{\zeta_1}(\theta_1)), \mathcal{D}(\varepsilon_{\zeta_2}(\theta_2))\big)\big] \\
    &\stackrel{\Circled{3}}{\leq} \E\big[\big((1-\gamma\mu)^2 + \gamma^2 L^2_{\cW} \big)\|\theta_1 - \theta_2\|^2 \big] \\
    &= \big((1-\gamma\mu)^2 + \gamma^2L^2_{\cW} \big) \cW_2^2(\nu_1, \nu_2),
\end{align*}
where $\Circled{1}$ is obtained by conditioning on $\theta_1, \theta_2$ and using~\eqref{eq:zeroexpect_grad_noise}, $\Circled{2}$ uses Lemma~\ref{lem:contraction} and $\Circled{3}$ uses Assumption~\ref{asm:gradient_wasserstein}.

Since $\gamma < \frac{2\mu}{\mu^2 + L^2_{\cW} }$ by assumption, the obtained inequality shows that the mapping $\nu \to \nu P_{\gamma}$ is a contraction in the space $\mathcal{P}_2(\R^d)$ endowed with the $\cW_2$ metric which is complete and separable by~\cite[Theorem 6.18]{villani2009optimal}. Consequently, by Banach's fixed-point theorem, the previous mapping admits a unique fixed point $\pi_{\gamma} \in \mathcal{P}_2(\R^d)$ i.e. such that $\pi_{\gamma} P_{\gamma} = \pi_{\gamma}.$ Moreover, for any initial measure $\xi_0 \in \mathcal{P}_2(\R^d),$ the sequence $(\xi_n)_{n\in \N}$ defined by $\xi_n = \xi_0 P_{\gamma}^n$ converges to $\pi_{\gamma}$ w.r.t. the $\cW_2$ metric.

\hfill \qedsymbol{}

Our Wasserstein convergence proof may be compared to that of~\cite[Proposition 2~(a)]{dieuleveut2020bridging}. Both proofs leverage the unbiased property of the gradient samples and the strong convexity of the objective. However, the combination of Lipschitz smoothness of the objective and Assumption~\ref{asm:gradient_wasserstein} in our setting is replaced by the average co-coercivity condition~\cite[Assumption A7]{dieuleveut2020bridging} (with $p=2$) which, in our notations, corresponds to
\begin{equation*}
    L'\langle \nabla\cL(\theta) - \nabla\cL(\theta'), \theta - \theta'\rangle \geq \E\big[\|G(\theta, \zeta) - G(\theta', \zeta)\|^2\big],
\end{equation*}
for some $L' > 0.$ Note that using the same sample $\zeta$ is important for this inequality to hold. As discussed following Proposition~\ref{prop:wasserstein_convergence}, the step-size condition $\gamma \leq 2/L'$ of~\cite{dieuleveut2020bridging} implies $\gamma \leq 2/L.$ Moreover, for certain cases like the example given in Section~\ref{sec:stepsize_compare}, it turns out to be equivalent to Proposition~\ref{prop:wasserstein_convergence}'s requirement. Finally, our proof leverages the fixed point theorem to establish the existence of a unique limit distribution $\pi_\gamma$ while~\cite{dieuleveut2020bridging} uses a less direct argument via a Cauchy sequence.

\subsection{Proof of Corollary~\ref{cor:concentration}}\label{sec:proof_cor_concentration}

From the proof of Proposition~\ref{prop:invariant_concentration}~\ref{prop:invariant_properties_c}, we have thanks to~\eqref{eq:subgauss_prf_checkpt2} that for $\lambda\leq 1/(2\overline{K}\sqrt{\gamma/\mu}),$
\begin{align}
    \E\exp\big(\lambda^2\|\theta_{T} - \theta^\star\|^2\big) &\leq \exp\Big(\lambda^2\Big((1-\gamma\mu)^T\|\theta_0 - \theta^\star\|^2 + \big(1-(1-\gamma\mu)^T\big)\gamma\overline{K}^2/\mu\Big)\Big)\nonumber\\
    &\leq\exp\Big(\lambda^2\Big((1-\gamma\mu)^T\|\theta_0 - \theta^\star\|^2 + \gamma\overline{K}^2/\mu\Big)\Big).\label{eq:gaussian_cor}
\end{align}
Using Chernoff's method, for $\lambda > 0,$ we find 
\begin{align}
    \Proba\big(\|\theta_T - \theta^\star\|^2 > \epsilon\big) &= \Proba\big(\exp\big(\lambda^2\|\theta_T - \theta^\star\|\big) > \exp(\lambda^2\epsilon)\big)\nonumber\\ 
    &\leq \E\big[\exp\big(\lambda^2\|\theta_T - \theta^\star\|^2 \big)\big]\exp(-\lambda^2 \epsilon)=:\delta.\label{eq:gaussian_cor2}
\end{align}
Moreover, by setting $\gamma = \log(\mu^2 T\|\theta_0 - \theta^\star\|^2/\overline{K}^2)/(\mu T),$ we get $(1-\gamma\mu)^T\|\theta_0 - \theta^\star\|^2 \leq \overline{K}^2/(\mu^2 T).$ We can then plug~\eqref{eq:gaussian_cor} into~\eqref{eq:gaussian_cor2}, take $\lambda= 1/(2\overline{K}\sqrt{\gamma/\mu})$ and solve for $\epsilon$ to find that with probability at least $1-\delta$ we have
\begin{equation*}
    \big\|\theta_T - \theta^\star\big\|^2 \leq \frac{\overline{K}^2}{\mu^2T}\Big(1 + \log\big(\mu^2 T\|\theta_0 - \theta^\star\|^2/\overline{K}^2\big)\big(1+4\log(1/\delta)\big) \Big),
\end{equation*}
as desired. The second part of the corollary is obtained by combining the sub-exponential property of Proposition~\ref{prop:invariant_concentration}~\ref{prop:invariant_properties_d} with Lemma~\ref{lem:subexp} for the same choice of $\gamma.$

\subsection{Proof of Corollary~\ref{cor:special_concentration}}\label{sec:proof_cor_special_concentration}

We consider the function $f(\theta) = \|\theta - \theta^\star\|$ and denote $\Delta_T = f(\theta_T) - \E f(\theta_T).$ Using Proposition~\ref{prop:invariant_special_concentration}~\ref{prop:invariant_properties_e} and Chernoff's method for $t > 0$ and $\lambda > 0,$ we have
\begin{equation}\label{eq:chernoff_cor}
    \Proba(\Delta_T > t ) = \Proba (e^{\lambda \Delta_T} > e^{\lambda t} ) \leq \E \exp\big(\lambda \Delta_T - \lambda t\big) \leq  \exp\big(\lambda^2 K_{\pi}(T)^2 - \lambda t\big).
\end{equation}
After minimizing over $\lambda,$ we get for $\delta > 0,$ with probability at least $1 - \delta,$ the following inequality holds:
\begin{equation}\label{eq:subgauss_confidence}
    \Delta_T \leq 2 K_{\pi}(T) \sqrt{\log(1/\delta)}.
\end{equation}
Additionally, using Proposition~\ref{prop:invariant_properties}~\ref{prop:invariant_properties_bprime} with $\gamma\leq\frac{\mu}{\mu^2 + L_\sigma^2},$ we have $\E f(\theta_T) = \E \|\theta_T - \theta^\star\| \leq \sqrt{\E \|\theta_T - \theta^\star\|^2},$ with
\begin{equation}
    \E\|\theta_T - \theta^\star\|^2 \leq \big(1-\gamma\mu\big)^T\|\theta_{0} - \theta^\star\|^2 + \frac{\gamma\sigma^2}{\mu}.\label{eq:var_bdd}
\end{equation}
Combining with~\eqref{eq:subgauss_confidence} and replacing the value of $\gamma$ yields~\eqref{eq:special_concentration_subgauss}.

To obtain~\eqref{eq:special_concentration_subexp}, we proceed similarly using Proposition~\ref{prop:invariant_special_concentration}~\ref{prop:invariant_properties_f} this time. Applying the constraint $|\lambda| \leq (\gamma K)^{-1}$ (see proof of Proposition~\ref{prop:invariant_special_concentration}) into the optimization of~\eqref{eq:chernoff_cor} yields
\begin{equation}\label{eq:subexp_confidence}
    \Proba(\Delta_{T} > t ) \leq \begin{cases}\exp\Big(\frac{-t^2}{4\gamma K^2/\mu}\Big) & \text{if}\quad t\leq 2K/\mu \\
     \exp\Big(\frac{-t}{2\gamma K}\Big)& \text{otherwise.}
    \end{cases}
\end{equation}
We then express $t$ in terms of the failure probability $\delta$ and combine with~\eqref{eq:var_bdd} as before to find
\begin{align*}
    \Proba\bigg(\big\| \theta - \theta^\star\big\| >  \sqrt{\frac{\gamma\sigma^2}{\mu}} + &2K\bigg(\sqrt{\frac{ \gamma\log(1/\delta)}{ \mu}} \vee \gamma \log(1/\delta) \bigg) \bigg) \leq \delta.
\end{align*}
We then replace the value of $\gamma$ to finish the proof.

\begin{lemma}\label{lem:geometric_covariances}
    Grant Assumption~\ref{asm:smooth_strongconvex},~\ref{asm:gradient},~\ref{asm:gradient_wasserstein} and~\ref{asm:linear_grad}. Let the Markov chain $(\theta_t)_{t \geq 0}$ be initialized with $\theta_0\sim\nu$ and $\gamma$ be chosen as in Proposition~\ref{prop:wasserstein_convergence}. The sequence of SGD iterates $\theta_0, \dots, \theta_n$ satisfies for $0 \leq i,j \leq n:$
    \begin{equation*}
        \E\langle \theta_i - \theta^\star, \theta_j - \theta^\star\rangle \leq 2(1-\gamma\mu)^{|i-j|}\big(\big((1-\gamma\mu) + \gamma^2L^2_{\cW}\big)^{i}\cW_2^2 (\nu, \pi) + \Var_{\pi_{\gamma}}(\theta)\big).
    \end{equation*}
\end{lemma}
\begin{proof}
We assume without loss of generality that $i \leq j.$ Since the gradient is linear it commutes with the expectation. Therefore, by conditioning over $\theta_{j-1}$ and later over $\theta_{j-2}$ we find
\begin{align*}
    \E\langle \theta_{j} - \theta^\star&, \theta_i - \theta^\star \rangle =\E \langle \theta_{j-1} - \gamma G(\theta_{j-1}) -\theta^\star, \theta_i - \theta^\star\rangle \\
    &=\E \langle \theta_{j-1} - \gamma \nabla\cL(\theta_{j-1}) -\theta^\star, \theta_i - \theta^\star\rangle \\
    &=\E \langle \theta_{j-2} -\gamma G(\theta_{j-2}) - \gamma \nabla\cL(\theta_{j-2} - \gamma G(\theta_{j-2})) -\theta^\star, \theta_i - \theta^\star\rangle \\
    &=\E \langle \theta_{j-2} -\gamma \nabla\cL(\theta_{j-2}) - \gamma \nabla\cL(\theta_{j-2} - \gamma \nabla\cL(\theta_{j-2})) -\theta^\star, \theta_i - \theta^\star\rangle .
\end{align*}
It is clear that the previous steps can be repeated for the remaining indices $j-3, j-4,\dots$ down to $i$ at which point the following identity is reached
\begin{equation*}
    \E\langle \theta_{j} - \theta^\star, \theta_i - \theta^\star \rangle = \E \langle\widecheck{\theta}_{j}-\theta^\star, \theta_i - \theta^\star\rangle,
\end{equation*}
where $\widecheck{\theta}_{j}$ is recursively defined by $\widecheck{\theta}_{i} = \theta_i$ and $\widecheck{\theta}_k = \widecheck{\theta}_{k-1} - \gamma \nabla\cL(\widecheck{\theta}_{k-1})$ for $i < k\leq j.$

Using Cauchy-Schwarz and iterating the inequality \begin{equation*}
    \|\widecheck{\theta}_{k}-\theta^\star\| \leq (1-\gamma\mu)\|\widecheck{\theta}_{k-1}-\theta^\star\|,
\end{equation*} yields that
\begin{equation*}
    \E\langle \theta_{j} - \theta^\star, \theta_i - \theta^\star \rangle \leq (1-\gamma\mu)^{j-i}\E \|\theta_i - \theta^\star\|^2.
\end{equation*}
Now, by~\cite[Theorem 4.1]{villani2009optimal}, there exists a random variable $\widetilde{\theta} \sim \pi_\gamma$ such that the coupling $(\theta_i, \widetilde{\theta})$ satisfies
\begin{align*}
    \E\|\theta_i - \widetilde{\theta}\big\|^2 &= \cW_2^2(\mathcal{D}(\theta_i), \pi) = \cW_2^2 (\nu P^i, \pi) \leq \big((1-\gamma\mu) + \gamma^2L^2_{\cW}\big)^{i}\cW_2^2 (\nu, \pi),
\end{align*}
where the inequality comes from Proposition~\ref{prop:wasserstein_convergence}. It then only remains to write
\begin{align*}
    \E \|\theta_i - \theta^\star\|^2 &\leq 2\big(\E \|\theta_i - \widetilde{\theta}\|^2 + \E \|\widetilde{\theta} - \theta^\star\|^2\big) \\
    &\leq 2\big((1-\gamma\mu) + \gamma^2L^2_{\cW}\big)^{i}\cW_2^2 (\nu, \pi) + 2 \Var_{\pi_{\gamma}}(\theta),
\end{align*}
which implies the result. The case $i\geq j$ is handled similarly.
\end{proof}

\begin{lemma}\label{lem:matrix_sum}
    Let $A \in \R^{n\times n}$ be a matrix with positive entries such that there exists $C>0$ and $0< \alpha < 1$ such that 
    \begin{equation*}
        A_{ij} \leq C\alpha^{|j-i|} \quad \text{ for }\quad 1\leq i,j \leq n,
    \end{equation*}
    then we have 
    \begin{equation*}
        \sum_{i,j}A_{ij} \leq C \Big(n+\frac{2\alpha}{1-\alpha}\Big(n - \Big(\frac{1-\alpha^n}{1-\alpha}\Big)\Big)\Big).
    \end{equation*}
\end{lemma}
\begin{proof}
    Straightforward computations yield
    \begin{align*}
        \sum_{i,j}A_{ij} &= \sum_{i=1}^n A_{ii} + 2\sum_{i<j}A_{ij} \leq nC + 2\sum_{i<j}A_{ij},
    \end{align*}
    and we have
    \begin{align*}
        \sum_{i=1}^n\sum_{j=i+1}^n A_{ij} &\leq C \sum_{i=1}^n\sum_{j=i+1}^n \alpha^{j-i} = C \alpha \sum_{i=1}^n\frac{1-\alpha^{n-i}}{1-\alpha}\\
        &= \frac{C\alpha}{1 - \alpha}\Big(n - \sum_{i=1}^n \alpha^{n-i}\Big) = \frac{C\alpha}{1 - \alpha}\Big(n - \frac{1-\alpha^n}{1-\alpha}\Big).
    \end{align*}
\end{proof}

\subsection{Proof of Theorem~\ref{thm:average_concentration}}\label{sec:proof_thm_average_concentration}

We introduce the notations $\theta_{[i]} = (\theta_{{0}}, \theta_2,\dots, \theta_i)$ and $\theta_{[k,l]} = (\theta_k, \theta_{k+1},\dots, \theta_l) $ and define, for ${0}\leq i\leq {n-1},$ the variables
\begin{equation*}
    M^{(i)} := \E\big[f(\vec{\theta}) \vert \theta_{[i]}\big] - \E\big[f(\vec{\theta}) \vert \theta_{[i-1]}\big] \quad \text{so that} \quad f(\vec{\theta}) - \E[f(\vec{\theta})] = \sum_{i={0}}^{{n-1}} M^{(i)}.
\end{equation*}
Notice that, if we condition on $\theta_{[i-1]}$ then $M^{(i)}$ only depends on $\theta_i$. We consider $M^{(i)}$ as a function of $\theta_i$ and compute its Lipschitz constant. We temporarily consider $\theta_i$ and $\theta_i'$ as two fixed deterministic vectors of $\R^d$ and $\theta_{i+1}, \theta_{i+2},\dots$ and $\theta'_{i+1}, \theta'_{i+2},\dots$ are the SGD trajectories obtained from them i.e. for $j> i:$
\begin{equation*}
    \theta_j = \theta_{j-1} - \gamma G(\theta_{j-1}) \quad \text{ and }\quad \theta'_j = \theta'_{j-1} - \gamma G(\theta'_{j-1}).
\end{equation*}    
In the following, we use the Lipschitz property of $f$ and the Kantorovich-Rubinstein dual representation of the $\cW_1$ metric
\begin{equation*}
    \cW_1 (\nu_1,\nu_2) = \sup_{h\in \Lip(\R^d)} \int hd\nu_1 - \int hd\nu_2,
\end{equation*}
in order to find
\begin{align*}
    \big|M^{(i)}&(\theta_i) - M^{(i)}(\theta_i')\big| = \Big|\E\big[f(\vec{\theta})\vert \theta_{[i]}\big] - \E\big[f(\theta_{[i-1]}, \theta_{[i, {n-1}]}')\big\vert \theta_i', \theta_{[i-1]}\big]\Big|\\
    &= \bigg|\sum_{j=i}^{n-1} \E\Big[f\big(\theta_{[i-1]}, \theta_{[i,j-1]}', \theta_{[j,{n-1}]}\big) - f\big(\theta_{[i-1]}, \theta_{[i,j]}', \theta_{[j+1,{n-1}]}\big)\big\vert \theta_i',\theta_{[i-1]}\Big]\bigg|\\
    &\leq \sum_{j=i}^{{n-1}} \cW_{1}\big(\mathcal{D}(\theta_j), \mathcal{D}(\theta_j')\big).
\end{align*}
Using Proposition~\ref{prop:wasserstein_convergence} we have
\begin{align*}
    \cW_{1}\big(\mathcal{D}(\theta_j), \mathcal{D}(\theta_j')\big)&\leq \cW_{2}\big(\mathcal{D}(\theta_j), \mathcal{D}(\theta_j')\big) = \cW_{2}\big(\mathcal{D}(\theta_{j-1}) P, \mathcal{D}(\theta'_{j-1}) P\big)\\
    &\leq \underbrace{\sqrt{(1-\gamma\mu)^2 + \gamma^2L^2_{\cW}}}_{=:\alpha_{\cW}(\gamma, \mu)}\cW_{2}\big(\mathcal{D}(\theta_{j-1}), \mathcal{D}(\theta'_{j-1})\big)\\
    &\leq \dots \\
    &\leq \alpha_{\cW}(\gamma, \mu)^{j-i}\cW_{2}\big(\mathcal{D}(\theta_{i}), \mathcal{D}(\theta'_{i})\big) = \alpha_{\cW}(\gamma, \mu)^{j-i}\big\|\theta_{i}- \theta'_{i}\big\|,
\end{align*}
where the last equality follows from $\theta_{i}$ and $\theta_{i}'$ being deterministic. Provided that $\gamma < \frac{2\mu}{\mu^2 + L^2_{\cW}}$ we have $\alpha_{\cW}(\gamma, \mu) < 1$ so that $\cW_1(\theta_j, \theta_j') \leq \alpha_{\cW}(\gamma, \mu)^{j-i}\|\theta_i - \theta_i'\|$ for $i\leq j\leq {n-1}.$ By summing over $j,$ we find that the $M^{(i)}$s are $(1-\alpha_{\cW}(\gamma, \mu))^{-1}$-Lipschitz
\begin{equation*}
    \big|M^{(i)}(\theta_i) - M^{(i)}(\theta_i')\big| \leq \frac{\|\theta_i - \theta_i'\|}{1 - \alpha_{\cW}(\gamma, \mu)}.
\end{equation*}
In what follows we denote $\E_{k}[\cdot] = \E[\cdot \vert \theta_{[k]}]$ to lighten notation and let $C_{\cW} := (1 - \alpha_{\cW}(\gamma, \mu))^{-1}.$ Let $\lambda \in \R,$ by conditioning on $\theta_{[{n-2}]},$ we have
\begin{align*}
    \E \exp\big( \lambda (f(\vec{\theta}) - \E f(\vec{\theta}))\big) &= \E \exp\Big(\lambda \sum_{i={0}}^{n-1} M^{(i)}\Big) = \E\Big[\E \Big[\exp\Big(\lambda \sum_{i={0}}^{n-1} M^{(i)}\Big) \big\vert \theta_{[{n-2}]}\Big]\Big]\\
    &= \E\Big[ \exp\Big(\lambda \sum_{i={0}}^{{n-2}} M^{(i)}\Big)\E\big[\exp(\lambda M^{({n-1})})\vert \theta_{[{n-2}]}\big]\Big].
\end{align*} 
Recall that conditionally on $\theta_{[{n-2}]},$ we have that $M^{({n-1})}$ is a function of $\theta_{n-1} = \theta_{{n-2}} - \gamma G(\theta_{{n-2}})$ so that $M^{({n-1})}$ is a $\gamma C_{\cW}$-Lipschitz function of $G(\theta_{{n-2}})$ which satisfies Assumption~\ref{asm:grad_special_concentration}~\ref{asm:grad_special_subgauss} and thus
\begin{equation*}
    \E\big[\exp(\lambda M^{({n-1})})\vert \theta_{[{n-2}]}\big] \leq \exp\big(\lambda^2 \gamma^2 C_{\cW}^2 K^2\big).
\end{equation*}
By repeating this argument $n-1$ times, we arrive at
\begin{align*}
    \E \exp\big( \lambda (f(\vec{\theta}) - \E f(\vec{\theta}))\big) &\leq \E\big[ \exp\big(\lambda M^{({0})}\big)\big]\exp\big((n-1)\lambda^2 \gamma^2 C_{\cW}^2 K^2\big)\\
    &\leq \exp\big(\lambda^2 C_{\cW}^2 K^2\gamma/\mu + (n-1)\lambda^2 \gamma^2 C_{\cW}^2 K^2\big),
\end{align*} 
where the last inequality uses that $\theta_{0} \sim \pi_\gamma$ which is $K\sqrt{\gamma/\mu}$-sub-Gaussian by Proposition~\ref{prop:invariant_special_concentration}~\ref{prop:invariant_properties_e}.

The proof in the sub-exponential case is completely analogous using Assumption~\ref{asm:grad_special_concentration}~\ref{asm:grad_special_subexp} and the result of Proposition~\ref{prop:invariant_special_concentration}~\ref{prop:invariant_properties_f} with the main difference that the obtained inequalities only hold for $|\lambda| \leq \big(C_{\cW}K\sqrt{\gamma/\mu}\big)^{-1} \wedge \big(\gamma C_{\cW}K \big)^{-1} = \big(C_{\cW}K\sqrt{\gamma/\mu}\big)^{-1}$ because $\gamma < \mu^{-1}.$

\subsection{Proof of Proposition~\ref{prop:average_concentration}}\label{sec:proof_prop_average_concentration}
For $j\geq 0,$ we introduce the notation
\begin{equation*}
    \Delta_{j} := \Big\|\sum_{t=j+1}^{j+n} \theta_t - n\theta^\star\Big\| - \E \Big\|\sum_{t=j+1}^{j+n} \theta_t - n\theta^\star\Big\|.
\end{equation*}
We are interested in obtaining a high probability bound on the quantity $\Delta_{n_0}.$ We write $\E_{\nu}$ for the expectation when the Markov chain is started with distribution $\nu$
\begin{align*}
    \E_{\nu}\big[ \exp\big(\lambda \Delta_{n_0}\big) \big] &= \E_{\nu P^{n_0}}\big[ \exp\big(\lambda \Delta_{1}\big)\big] = \E_{\pi_{\gamma}}\Big[ \frac{d(\nu P^{n_0})}{d\pi_{\gamma}}\exp\big(\lambda \Delta_{1}\big)\Big] \\
    &\leq  \Big\|\frac{d(\nu P^{n_0})}{d\pi_{\gamma}} \Big\|_{\pi_{\gamma},\infty}\E_{\pi_{\gamma}}\big[\exp\big(\lambda \Delta_{1}\big)\big],
\end{align*}
where the essential supremum $\|f\|_{\pi_{\gamma},\infty}$ of a function $f$ is the smallest value such that $f\leq \|f\|_{\pi_{\gamma},\infty}$ $\pi_{\gamma}$-almost surely. The last expectation in the above inequality can be bounded using Theorem~\ref{thm:average_concentration}. As for the factor coming from the measure change, we write
\begin{equation*}
    \Big\|\frac{d(\nu P^{n_0})}{d\pi_{\gamma}} \Big\|_{\pi_{\gamma},\infty} \leq \Big\|\frac{d(\nu P^{n_0} - \pi_{\gamma})}{d\pi_{\gamma}} \Big\|_{\pi_{\gamma},\infty} + 1.
\end{equation*}
For any function $F:\R^d\to \R,$ we define the norm $\|F\|_V = \sup_{\vartheta \in \R^d} \frac{|F(\vartheta)|}{V(\vartheta)}$ and its induced operator norm $\vertiii{Q}_V = \sup_{F}\frac{\|QF\|_V}{\|F\|_V},$ where $V$ is the function defined in Section~\ref{sec:proof_geometric_ergodicity}. We also denote $\odot$ the pointwise product between functions.
\begin{align*}
    \Big\|\frac{d(\nu P^{n_0} \!-\! \pi_{\gamma})}{d\pi_{\gamma}} \Big\|_{\pi_{\gamma},\infty} \!&= \Big\|\frac{d(\nu (P^{n_0} \!-\! \ind{}\!\otimes\! \pi_{\gamma}))}{d\pi_{\gamma}} \Big\|_{\pi_{\gamma},\infty} = \Big\|(P^{n_0} \!-\! \ind{}\!\otimes\! \pi_{\gamma})^* \frac{d\nu}{d\pi_{\gamma}} \Big\|_{\pi_{\gamma},\infty} \\
    &= \Big\|(P^{n_0} \!-\! \ind{}\!\otimes\! \pi_{\gamma})^* \frac{d\nu}{d\pi_{\gamma}} \!\odot\! V \!\odot\! \frac 1 V \Big\|_{\pi_{\gamma},\infty} \\
    &\leq \Big\|(P^{n_0} \!-\! \ind{}\!\otimes\! \pi_{\gamma})^* \frac{d\nu}{d\pi_{\gamma}} \!\odot\! V \Big\|_{V} \\
    &\leq \vertiii{(P^{n_0} \!-\! \ind{}\!\otimes\! \pi_{\gamma})^*}_V \Big\| \frac{d\nu}{d\pi_{\gamma}} \!\odot\! V \Big\|_{V} \!=\! \vertiii{P^{n_0} \!-\! \ind{}\!\otimes\! \pi_{\gamma}}_V \Big\| \frac{d\nu}{d\pi_{\gamma}} \Big\|_{\infty}.
\end{align*}
The outer product $\ind{}\otimes \pi_\gamma$ denotes the kernel such that $\ind{}\otimes \pi_{\gamma}(\vartheta, A) = \pi_{\gamma}(A)$ for all $\vartheta$ and $A\in \mathcal{B}(\R^d).$ By the proof of Theorem~\ref{thm:ergodicity} and~\cite[Proposition 1.1]{kontoyiannis2012geometric} (see also Equation~(4)) the kernel $P$ has a spectral gap in the Banach space $L_\infty^V$ of functions with finite norm $\|\cdot\|_V$ and, therefore, there exist $\rho < 1$ and $M< \infty$ such that
\begin{equation*}
    \vertiii{P^{n_0} - \ind{}\otimes \pi_{\gamma}}_V \leq M\rho^{n_0},
\end{equation*}
which leads to 
\begin{equation*}
    \Big\|\frac{d(\nu P^{n_0})}{d\pi_{\gamma}} \Big\|_{\pi_{\gamma},\infty} \leq 1 + M\rho^{n_0}\Big\| \frac{d\nu}{d\pi_{\gamma}} \Big\|_{\infty} = \Upsilon(\nu, n_0).
\end{equation*}
Using Theorem~\ref{thm:average_concentration} in the sub-Gaussian case, denoting $\Breve{K} = KC_{\cW}\sqrt{\gamma/\mu + (n-1) \gamma^2},$ we find
\begin{equation*}
    \E_{\nu}\Big[ \exp\big(\lambda \Delta_{n_0}\big)\Big] \leq \Upsilon(\nu, n_0) \exp(\lambda^2 \Breve{K}^2).
\end{equation*}
Using Chernoff's method for a random variable $X\in \Psi_2(\Breve{K})$ and $t > 0$ and $\lambda > 0,$ we have
\begin{equation*}
    \Proba_{\nu}(\Delta_{n_0} > t ) = \Proba_{\nu}(e^{\lambda \Delta_{n_0}} > e^{\lambda t} ) \leq \E_{\nu}\exp\big(\lambda \Delta_{n_0} - \lambda t\big) \leq \Upsilon(\nu, n_0) \exp\big(\lambda^2 \Breve{K}^2 - \lambda t\big).
\end{equation*}
After minimizing over $\lambda,$ we get that for $\delta > 0,$ with probability at least $1 - \Upsilon(\nu, n_0)\delta,$ the following inequality holds
\begin{equation}\label{eq:delta_subgauss}
    \Delta_{n_0} \leq 2\Breve{K}\sqrt{\log(1/\delta)}.
\end{equation}
In the sub-exponential case (under Assumption~\ref{asm:grad_special_concentration}~\ref{asm:grad_special_subexp}), taking the constraint $|\lambda| \leq \big(C_{\cW}K\sqrt{\gamma/\mu}\big)^{-1}$ into account (see the proof of Theorem~\ref{thm:average_concentration}), we get that
\begin{equation*}
    \Proba(\Delta_{n_0} > t ) \leq \begin{cases}\Upsilon(\nu, n_0) \exp\Big(\frac{-t^2}{4 \Breve{K}^2}\Big) & \text{if}\quad t\leq \frac{2\Breve{K}^2}{C_{\cW}K\sqrt{\gamma/\mu}} \\
     \Upsilon(\nu, n_0) \exp\Big(\frac{-t}{2 C_{\cW}K\sqrt{\gamma/\mu}}\Big)& \text{otherwise.}
    \end{cases}
\end{equation*}
So that with probability at least $1 - \Upsilon(\nu, n_0)\delta:$
\begin{equation}\label{eq:delta_subexp}
    \Delta_{n_0} \leq 2\Breve{K}\sqrt{\log(1/\delta)}\:\vee\: 2 C_{\cW}K\sqrt{\gamma/\mu}\log(1/\delta) .
\end{equation}
It then only remains to bound the expectation $\E \Big\|\sum_{t=n_0+1}^{n_0+n}\! \theta_t - n\theta^\star\Big\|,$ which can be done as follows
\begin{align*}
\Big(\E\Big\|\sum_{t=n_0+1}^{n_0+n} \theta_t - n\theta^\star\Big\|\Big)^2 &\leq \E\Big\|\sum_{t=n_0+1}^{n_0+n} (\theta_t - \theta^\star)\Big\|^2 \\
&= \sum_{i=n_0+1}^{n_0+n}\sum_{j=n_0+1}^{n_0+n}\E\langle \theta_i - \theta^\star, \theta_j - \theta^\star \rangle.
\end{align*}
Using Lemmas~\ref{lem:geometric_covariances} and~\ref{lem:matrix_sum}, we find that
\begin{align*}
    \Big(\E\Big\|\sum_{t=n_0+1}^{n_0+n} \theta_t - \theta^\star\Big\|\Big)^2 \leq 2n \frac{1+\alpha}{1-\alpha}\Big( \alpha_{\cW}^{n_0} \cW_2^2(\nu, \pi_{\gamma}) + \Var_{\pi_{\gamma}}(\theta)\Big),
\end{align*}
where $\alpha = 1-\gamma\mu$ and $\alpha_{\cW} = \sqrt{\alpha^2 + \gamma^2L^2_{\cW}}.$ Moreover, since $\gamma < \frac{\mu}{\mu^2 + L^2_{\sigma}},$ by Proposition~\ref{prop:invariant_properties}, we have
\begin{equation*}
    \Var_{\pi_{\gamma}}(\theta) \leq \frac{\gamma \sigma^2}{\mu}.
\end{equation*}
Plugging into Inequalities~\eqref{eq:delta_subgauss} and~\eqref{eq:delta_subexp} and dividing by $n$ finishes the proof.

\subsection{Proof of Lemma~\ref{lem:dimfree_boundedness} }\label{sec:prf_dimfree_boundedness}

Denote $\Xi_{t}^{(N)} = \frac{1}{N}\sum_{i=1}^N \Xi_{tN+i}$ and $\xi_{t}^{(N)} = \frac{1}{N}\sum_{i=1}^N \xi_{tN+i}.$  By Lemma~\ref{lem:intermediate} below, we have the following concentration inequalities for all $0\leq t < T:$
\begin{align*}
    &\Proba\Big(\|\Xi_t^{(N)}\|_2 > 3K_\Xi \Big(\frac{\log(4T/\delta) + 3d}{N} \vee \sqrt{\frac{\log(4T/\delta) + 3d}{N}}\Big)\Big) \leq \delta/(2T) \\
    &\Proba\Big(\|\xi_t^{(N)}\| > 4K_\xi \Big(\frac{\log(4T/\delta) + 2d}{N} \vee \sqrt{\frac{\log(4T/\delta) + 2d}{N}}\Big)\Big) \leq \delta/(2T).
\end{align*}

We will show by induction over $0 \leq t \leq T$ that we have with probability at least $1 - t\delta/T$ that 
\begin{equation}\label{eq:induct}
    \max_{0 \leq s\leq t} \|\theta_s - \theta^\star\| \leq C.
\end{equation}
The case $t=0$ holds by assumption. Further, assuming the property at rank $t$ and conditioning on $\theta_t$ we have with probability at least $1-\delta/T:$
\begin{align*}
    \|\theta_{t+1} &- \theta^\star\|^2 = \|\theta_t - \gamma \nabla\cL(\theta_t) - \gamma(\Xi_t^{(N)}(\theta_t - \theta^\star) + \xi_t^{(N)}) - \theta^\star\|^2\\
    &= \|\theta_t - \gamma \nabla\cL(\theta_t) - \theta^\star\|^2 - 2\gamma\langle \theta_t - \gamma\nabla\cL(\theta_t) - \theta^\star, \Xi_t^{(N)}(\theta_t - \theta^\star) + \xi_t^{(N)} \rangle \\
    &\quad + \gamma^2 \|\Xi_t^{(N)}(\theta_t - \theta^\star) + \xi_t^{(N)}  \|^2 \\
    &\stackrel{\Circled{1}}{\leq} (1-\gamma\mu)^2 \|\theta_t - \theta^\star\|^2 + 2\gamma(1-\gamma\mu)\| \theta_t - \theta^\star\|( \|\Xi_t^{(N)}(\theta_t - \theta^\star)\| + \|\xi_t^{(N)}\| ) \\
    &\quad + 2\gamma^2 \|\Xi_t^{(N)}(\theta_t - \theta^\star)\|^2 + 2\gamma^2\|\xi_t^{(N)}\|^2 \\
    &\leq \big[(1-\gamma\mu)^2 + 2\gamma(1-\gamma\mu) \|\Xi_t^{(N)}\|_2 + 2\gamma^2 \|\Xi_t^{(N)}\|_2^2\big]\|\theta_t - \theta^\star\|^2 \\
    &\quad+ 2\gamma(1-\gamma\mu)\| \theta_t - \theta^\star\|\| \xi_t^{(N)}\| + 2\gamma^2 \|\xi_t^{(N)}\|^2\\
    &\stackrel{\Circled{2}}{\leq} \big[(1-\gamma\mu)^2(1+\epsilon) + 2\gamma(1-\gamma\mu) \|\Xi_t^{(N)}\|_2 + 2\gamma^2 \|\Xi_t^{(N)}\|_2^2\big]\|\theta_t - \theta^\star\|^2 \\
    &\quad + \gamma^2 (2+ 1/\epsilon) \|\xi_t^{(N)}\|^2\\
    &\stackrel{\Circled{3}}{\leq} \big[(1-\gamma\mu) + 2\gamma \|\Xi_t^{(N)}\|_2 + 2\gamma^2 \|\Xi_t^{(N)}\|_2^2\big]\|\theta_t - \theta^\star\|^2 + 3\frac{\gamma}{\mu} \|\xi_t^{(N)}\|^2\\
    &\stackrel{\Circled{4}}{\leq} \big[(1-\gamma\mu) + \gamma \mu/3 + \gamma\mu/3\big]C^2 + \gamma\mu C^2/3 \leq C^2,
\end{align*}
where $\Circled{1}$ uses Lemma~\ref{lem:contraction} and the Cauchy-Schwarz inequality, $\Circled{2}$ uses the inequality $2ab \leq a^2\epsilon + b^2/\epsilon$ valid for all $\epsilon> 0$ and $\Circled{3}$ sets the choice $\epsilon = \gamma \mu$ and uses that $\gamma \leq 1/\mu.$ Finally $\Circled{4}$ uses the high probability bounds stated above and the conditions on $N$ and $\gamma$.

Using a union bound argument, we obtain~\eqref{eq:induct} for $t+1$ with probability at least $1-(t+1)\delta/T.$ The induction argument is completed and implies the result for $t=T.$

\begin{lemma}\label{lem:intermediate}
    Let $\Xi_1, \dots, \Xi_N \in \R^{d\times d}$ be random matrices and $\xi_1, \dots, \xi_N \in \R^d$ random vectors as in Lemma~\ref{lem:dimfree_boundedness}. Then denoting $\overline{\Xi} = \frac{1}{N}\sum_{i=1}^N \Xi_{i}$ and $\overline{\xi} = \frac{1}{N}\sum_{i=1}^N \xi_{i},$ we have the high probability bounds
    \begin{align}
        \Proba\Big(\|\overline{\Xi}\|_2 > 3K_\Xi \phi\Big(\frac{\log(2/\delta) + 3d}{N}\Big)\Big) \leq \delta, \label{eq:XI1}\\
        \Proba\Big(\|\overline{\xi}\| > 4K_\xi \phi\Big(\frac{\log(2/\delta) + 2d}{N}\Big)\Big) \leq \delta,\label{eq:xi2}
    \end{align}
    where $\phi(x) = x\:\vee \sqrt{x}.$
\end{lemma}
\begin{proof}
We first prove~\eqref{eq:XI1}. Denote $S^{d-1} = \{u\in \R^d \: : \: \|u\| = 1\}$ and let $u \in S^{d-1}$ and $|\lambda| \leq N/K_\Xi,$ we have
\begin{align*}
    \E \exp(\lambda \langle u, \overline{\Xi} u \rangle) = \prod_{i=1}^N \E \exp(\lambda \langle u, \Xi_i u \rangle/N) 
    \leq \prod_{i=1}^N \exp(\lambda^2 K_\Xi^2/N^2) = \exp(\lambda^2 K_\Xi^2/N),
\end{align*}
so that for all $u\in S^{d-1}$ we have $\langle u, \overline{\Xi}u \rangle \in \Psi_1(K_\Xi / \sqrt{N}).$ 

Let $\Omega_\epsilon$ be an $\epsilon$-net of $S^{d-1}.$ By~\cite[Lemma 5.2]{vershynin2010introduction}, there exists an $\epsilon$-net such that $|\Omega_\epsilon| \leq (1+2/\epsilon)^d$ and for all $u\in S^{d-1}$ there exists $v\in \Omega_\epsilon$ such that $\|u-v\| \leq \epsilon.$ We write
\begin{equation*}
    \langle u, \overline{\Xi}u \rangle = \langle v, \overline{\Xi}v \rangle + 2\langle u-v, \overline{\Xi}v \rangle + \langle u-v, \overline{\Xi}(u-v) \rangle,
\end{equation*}
which allows us to deduce that
\begin{align*}
    \|\overline{\Xi}\|_2 &= \sup_{u\in S^{d-1}} |\langle u, \overline{\Xi} u \rangle| \leq \sup_{v\in \Omega_\epsilon} |\langle v, \overline{\Xi} v \rangle| + (2\epsilon + \epsilon^2)\|\overline{\Xi}\|_2 \\
    &\implies \|\overline{\Xi}\|_2 \leq \frac{\sup_{v\in \Omega_\epsilon} |\langle v, \overline{\Xi} v \rangle|}{1-2\epsilon - \epsilon^2}.
\end{align*}
Let $v \in \Omega_\epsilon,$ using Chernoff's method and the sub-exponential property of $\overline{\Xi}$ (see also the proof of Corollary~\ref{cor:special_concentration}), we find for $t>0:$
\begin{equation*}
    \Proba(|\langle v , \overline{\Xi} v\rangle| > t) \leq \begin{cases}
        2\exp(-Nt^2/(4K_\Xi^2)) & \text{ if }\quad t\leq 2K_\Xi \\
        2\exp(-Nt/(2K_\Xi)) & \text{ otherwise.}
    \end{cases}
\end{equation*}
Reformulating in terms of a failure probability $\delta,$ we find that
\begin{equation*}
    \Proba\Big(|\langle v , \overline{\Xi} v\rangle| > 2K_\Xi \phi\Big( \frac{\log(2/\delta)}{N} \Big)\Big) \leq \delta.
\end{equation*}
Replacing $\delta$ with $\delta/(1+2/\epsilon)^d$ and using a union bound argument over $\Omega_\epsilon$ we find
\begin{equation*}
    \Proba\Big(\sup_{v\in \Omega_\epsilon}|\langle v , \overline{\Xi} v\rangle| > 2K_\Xi \phi\Big( \frac{\log(2/\delta) + d\log(1+2/\epsilon)}{N} \Big)\Big) \leq \delta.
\end{equation*}
It only remains to set $\epsilon = 1/8$ and plug back into the inequality $\|\overline{\Xi}\|_2 \leq \frac{\sup_{v\in \Omega_\epsilon} |\langle v, \overline{\Xi} v \rangle|}{1-2\epsilon - \epsilon^2}$ in order to obtain~\eqref{eq:XI1}.

To prove~\eqref{eq:xi2}, we proceed similarly and first obtain for all $u \in S^{d-1}$ and $|\lambda| \leq N/K_\xi :$
\begin{equation*}
    \E \exp(\langle u, \overline{\xi} \rangle ) \leq \exp(\lambda^2K_\xi^2/N).
\end{equation*}
For $u\in S^{d-1}$ and $v\in \Omega_\epsilon$ such that $\|u-v\| \leq \epsilon,$ we write $\langle u, \overline{\xi}\rangle = \langle v, \overline{\xi}\rangle + \langle u-v , \overline{\xi}\rangle$ which yields the inequality
\begin{equation*}
    \|\overline{\xi}\| \leq \frac{\sup_{v\in \Omega_\epsilon} |\langle v, \overline{\xi}\rangle|}{1-\epsilon}.
\end{equation*}
As before, the sub-exponential property of $\overline{\xi}$ yields
\begin{equation*}
    \Proba\Big(|\langle v , \overline{\xi}\rangle| > 2K_\xi \phi\Big( \frac{\log(2/\delta)}{N} \Big)\Big) \leq \delta,
\end{equation*}
and using another union bound argument over $\Omega_\epsilon$ we find
\begin{equation*}
    \Proba\Big(\sup_{v\in \Omega_\epsilon}|\langle v , \overline{\xi} \rangle| > 2K_\xi \phi\Big( \frac{\log(2/\delta) + d\log(1+2/\epsilon)}{N} \Big)\Big) \leq \delta.
\end{equation*}
It only remains to set $\epsilon = 1/2$ to finish the proof of~\eqref{eq:xi2}.
\end{proof}

\end{document}